\pdfoutput=1
\documentclass[letterpaper]{article}
\usepackage[preprint]{aaai2027}
\usepackage[hyphens]{url}
\usepackage{graphicx}
\usepackage{fontawesome5}
\usepackage{natbib}
\usepackage{caption}
\usepackage{algorithm}
\usepackage{algorithmic}
\DeclareCaptionStyle{ruled}{labelfont=normalfont,labelsep=colon,strut=off}
\usepackage{booktabs}
\usepackage{amsmath}
\usepackage{amssymb}
\usepackage{float}

\usepackage{xcolor}
\definecolor{stdgray}{gray}{0.50}
\definecolor{forgegreen}{RGB}{0,105,0}

\newcommand{\forge}{FORGE}
\newcommand{\dY}{\Delta\mathbf{Y}}
\newcommand{\pnum}[1]{#1}

\title{FORGE: Fused On-Register Gradient Elimination\\for Memory-Efficient LLM Training}
\author{
    Dikshant Kukreja\textsuperscript{\rm 1,\rm 5},
    Kritarth Prasad\textsuperscript{\rm 2},
    Avinash Anand\textsuperscript{\rm 2},
    Zhengkui Wang\textsuperscript{\rm 2},
    Erik Cambria\textsuperscript{\rm 3},\\
    Timothy Liu\textsuperscript{\rm 4},
    Aik Beng Ng\textsuperscript{\rm 4},
    Simon See\textsuperscript{\rm 4},
    Bapi Chatterjee\textsuperscript{\rm 5}
}
\affiliations{
    \textsuperscript{\rm 1}Puch AI \quad
    \textsuperscript{\rm 2}Singapore Institute of Technology \quad
    \textsuperscript{\rm 3}Nanyang Technological University \quad
    \textsuperscript{\rm 4}NVIDIA \quad
    \textsuperscript{\rm 5}IIIT-Delhi\\
    dikshant@puch.ai\\
    {\faGithub}~Code: \url{https://github.com/dk4248/FORGE}
    \vspace{34pt}
}
\usepackage{longtable}
\usepackage{amsthm}
\theoremstyle{plain}
\newtheorem{proposition}{Proposition}
\theoremstyle{definition}
\newtheorem{definition}{Definition}
\newtheorem{remark}{Remark}
\renewcommand{\S}{\ensuremath{\mathsection}}
\begin{document}

\maketitle

\begin{abstract}
Reverse-mode differentiation computes every weight gradient, writes it to memory, and only then lets the optimizer read it back. Each gradient is used once and discarded, yet at peak the gradients together cost as many bytes as the weights themselves. That pool is an artifact of scheduling rather than something learning requires. We introduce \forge{}, which applies the optimizer to each weight-gradient tile in the fp32 registers that produced it, writing back only the new weight and moments, so the gradient is never stored. The fused step is provably exact whenever only the optimizer's state update reads the gradient, and the same condition settles which parallelism strategies inherit that exactness and which must reduce across ranks first. Nothing in the kernel depends on the surrounding architecture, so one implementation trains transformers, a state-space model, and a MLP mixer. The change is to when the gradient is consumed, not what any method computes, so it composes with the other remedies rather than competing: quantized states, low-rank projections and factored moments each keep their own saving and shed the gradient pool on top, 16–33\% of peak. Because the gradient never rounds to bf16 on its way to the optimizer, \forge{} also removes a truncation bias the standard path carries undamped. On Llama-3.1-8B, peak memory falls from 62.0~GB under vanilla AdamW to 48.4~GB at matched state precision, which also runs $1.5\times$ faster, and to 35.3~GB with int8 moments. \forge{} trains a 32B model with Muon on a H200 where standard Muon does not fit, and on an 8-GPU node it reaches the lowest per-rank memory of any method that trains it.
\end{abstract}

\section{Introduction}

Four pools share a GPU's memory during training: the weights, the optimizer states, the saved activations, and the gradients. At the moment a step peaks,
\begin{equation}
\mathcal{M}_{\mathrm{peak}} \;=\; \underbrace{2P}_{\text{weights}} \;+\; \underbrace{kP}_{\text{opt.\ states}} \;+\; \underbrace{2P}_{\text{gradients}} \;+\; \underbrace{\textstyle\sum_{\ell}|a_\ell|}_{\text{activations}},
\label{eq:peak}
\end{equation}
where $P$ is the parameter count, $k$ is the byte cost of optimizer state per parameter (8 for fp32 AdamW moments, 4 for bf16, 2 for int8), and $|a_\ell|$ is the byte size of the activations saved for layer $\ell$. Three of the four terms have well-developed remedies. States can be quantized \citep{bnb8bit,coat} or projected to low rank \citep{galore,apollo}, activations can be recomputed \citep{chen2016}, and weights can be stored in FP8 or four-bit formats \citep{coat,quartet}. None of the three reaches the gradient term, by construction: shrinking $k$, the weight bytes, or the activation sum leaves the $2P$ of gradients exactly where it was. The remedies that do reach it coarsen the granularity of the update rather than removing the pool. For Llama-3.1-8B \citep{llama31} that is about 15~GB, alive at the backward-to-optimizer boundary, the moment that decides whether a run fits.

\begin{figure}[t]
\centering
\includegraphics[height=3.35cm]{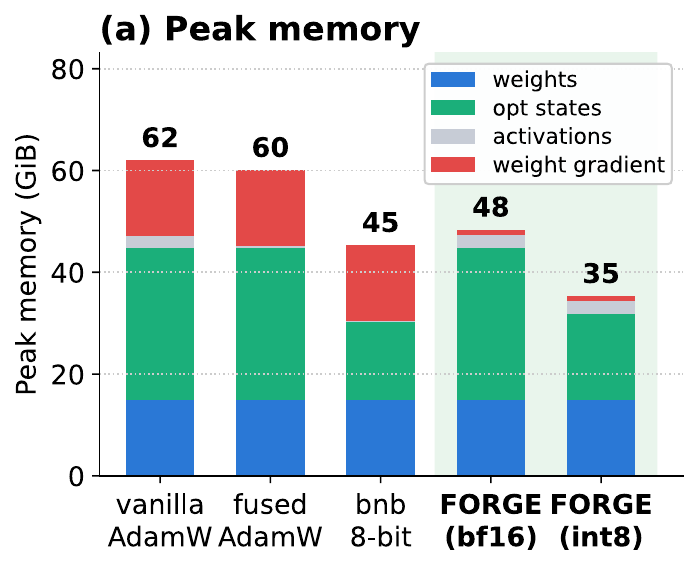}\hfill
\includegraphics[height=3.35cm]{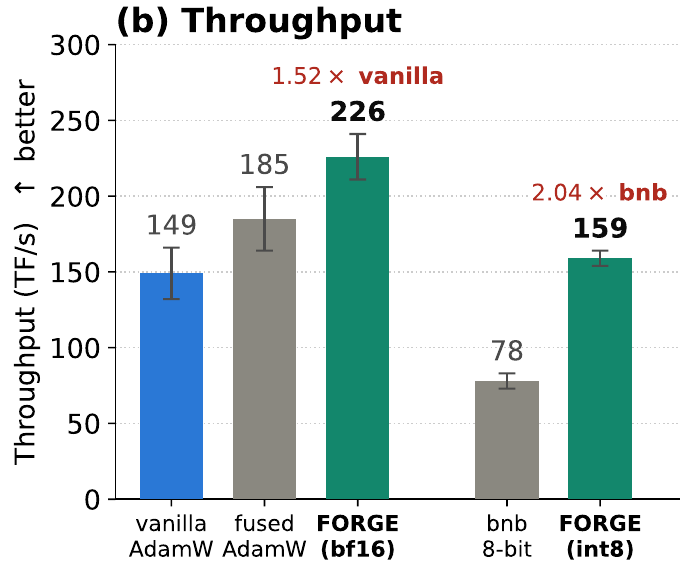}
\caption{\forge{} on Llama-3.1-8B (H200, BT=512). (a)~Peak memory: the weight gradient (red) collapses under \forge{}; its two arms differ only in moment precision. (b)~Achieved throughput, grouped by optimizer-state class; bars are three-run means, whiskers one sd (Table~\ref{tab:headline}).}
\label{fig:teaser}
\end{figure}

That pool is a consequence of how the backward pass is scheduled. Backpropagation proceeds in two phases: it computes every gradient, and only then allows the optimizer to read them back. Between the two phases every layer's gradient is resident at once, which is the plateau measured in Figure~\ref{fig:trace}. The gradient itself, however, is produced inside a matrix multiply, in registers, already in fp32, and the optimizer applies roughly a dozen arithmetic operations to it before discarding it (Algorithm~\ref{alg:forge}). Standard training nonetheless writes that gradient to HBM and immediately reads it back. The round trip is imposed by the schedule rather than required by the update rule, and it dominates the cost of the step: the optimizer moves far more bytes than it performs arithmetic on, so the step is limited by memory bandwidth rather than by compute \citep{micikevicius2018}.

A line of work targets this floor by updating parameters \emph{during} the backward pass, always at a granularity coarser than the gradient's own. \citet{pudipeddi2020} introduced the schedule; LOMO \citep{lomo} brought it to LLM training for plain SGD, and AdaLomo \citep{adalomo} restored adaptivity through a factored second moment \citep{adafactor}, trading exactness for memory. PyTorch exposes it as a gradient hook \citep{paszke2019}, which optimi \citep{optimi} and FlashOptim \citep{flashoptim} build into per-parameter AdamW; both wait for a whole parameter's gradient to materialize, 1~GB for Llama-3.1-8B's embedding, and pay one launch per parameter. Adam Accumulation \citep{adama} is the closest prior point, accumulating partial gradients in HBM a layer at a time; GradLite \citep{gradlite2025} shrinks the gradient with a low-rank Jacobian approximation. Each approximates the optimizer or stops above where the gradient is computed; none reaches the tile with an exact update. Appendix~J places them on a common axis of granularity, exactness, and state cost.

Fusing a consumer into the kernel that produced its operand is routine at operator scope. Attention kernels never materialize the score matrix \citep{flashattn}; \citet{jiang2021} named the framework-level reordering optimizer fusion; Megatron-LM folds gradient accumulation into the weight-gradient epilogue \citep{megatron}; Liger and Cut Cross-Entropy fuse the weight-gradient computation for the LM head alone \citep{liger,cce}. Concurrent work pushes fusion deeper into the block \citep{coda2026,chronicals2026}, yet none of it touches the optimizer. The one fusion none of this literature takes is the optimizer's own update, folded into the weight-gradient epilogue of every linear layer. \forge{} (Fused On-Register Gradient Elimination) is that fusion.

\begin{figure}[t]
\centering
\includegraphics[width=\linewidth]{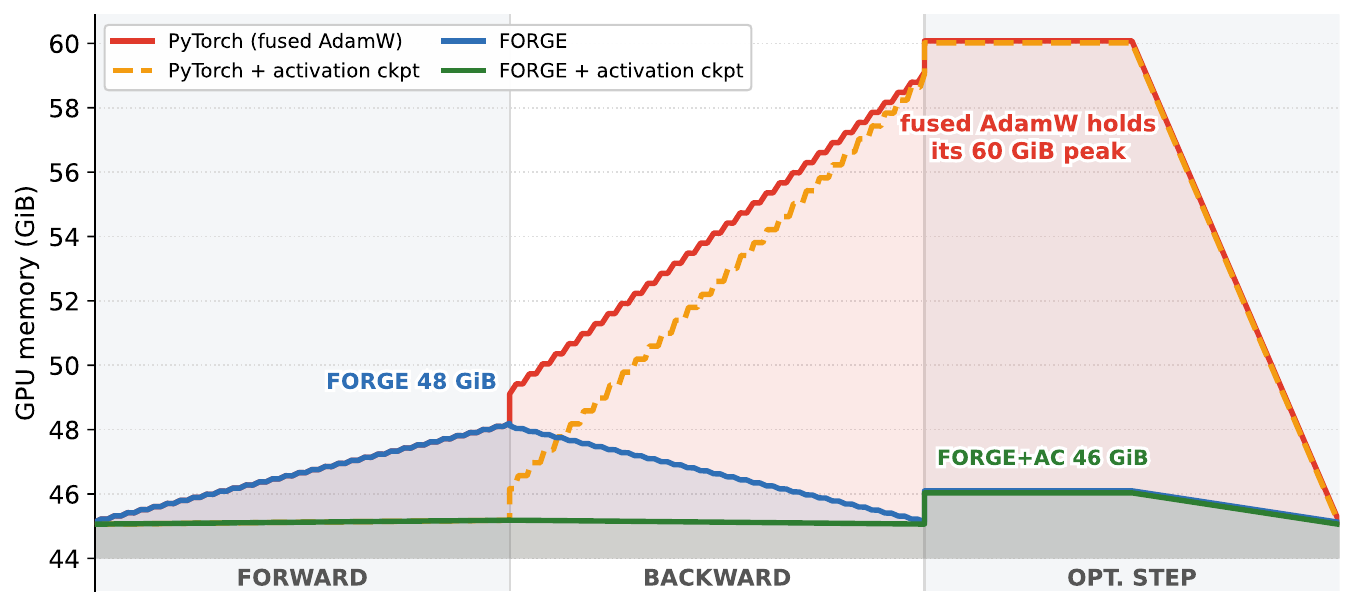}
\caption{Measured GPU memory across one training step (Llama-3.1-8B, sequence~512; per-layer trace). Fused AdamW accumulates the full gradient through the backward pass and holds that 60~GB peak through the optimizer step; \forge{} consumes each gradient tile as it is produced, so its curve falls back during backward and the peak never forms.}
\label{fig:trace}
\end{figure}

The kernel accumulates one tile of the weight gradient in fp32 registers,
runs the optimizer arithmetic on that tile, writes the updated weight and
moment tiles back, and discards the gradient. Section~\ref{sec:method}
states the condition that makes this legal. A step splits in two: a
\emph{state update} that reads the gradient, and a \emph{weight update} that
reads only the state it produced; the schedule constrains the first and
leaves the second free. AdamW \citep{adamw}, SGD with momentum
\citep{robbins1951,polyak1964}, Lion \citep{lion}, RMSprop \citep{rmsprop}, and Adagrad \citep{adagrad} enter state one coordinate at a time and
fuse entirely in the register epilogue. The condition admits more than
that. Muon \citep{muon} and Scion \citep{scion} enter state affinely, so
the weight-gradient multiply accumulates in place into the momentum buffer and the orthogonalization that follows reads state that is already resident, never a gradient; LAMB \citep{lamb} takes a layer-global norm,
but of the update its moments produce rather than of the gradient, so it
fuses on the same terms.

What this buys can be read off Equation~\ref{eq:peak}. The gradient term collapses from $2P$ to a single tile's working set, independent of model size, and the saving is the same under every precision recipe, because no recipe touches that term. Everything else stays the baseline's --- same weights, moments, precisions. Speed is bounded by what fusion removes, the separate optimizer phase, so the gain can be no larger than the share of the step it occupied. Measured across models and batch sizes it tracks exactly that share (Section~\ref{sec:experiments}): up to $1.5\times$ when the optimizer dominates the step, parity when activations do, with the memory saving in both cases.

The same argument decides what survives distribution. Where a rank's local gradient is already complete the fusion carries over unchanged, which covers tensor, sequence, and expert parallelism \citep{megatron,seqpar,switchtransformer}. Data and context parallelism replicate the weight, and an adaptive update cannot precede the cross-rank sum \citep{ddp,ulysses}, so \forge{} reduces bucket by bucket and steps each reduced bucket, replacing the $O(P)$ pool with a one-bucket transient (Section~\ref{sec:parallelism}). The removed bytes convert directly to capability: with fp8 moments a 32B Qwen3 model \citep{qwen3} trains on a single H200 where fused AdamW runs out of memory, and on an 8-GPU node the fully-sharded int8 arm trains that model at 70.6~GB per rank where ZeRO-3 does not fit.

\section{Method}
\label{sec:method}

\subsection{Standard Backward and the Gradient Floor}
\label{sec:floor}

Consider a linear layer with weight $\mathbf{W} \in \mathbb{R}^{D_{\mathrm{out}} \times D_{\mathrm{in}}}$ and activation batch $\mathbf{X} \in \mathbb{R}^{BT \times D_{\mathrm{in}}}$, $BT$ the rows of $\mathbf{X}$ (batch flattened over sequence or spatial positions); the forward pass computes
\begin{equation}
\mathbf{Y} \;=\; \mathbf{X}\mathbf{W}^{\!\top} \;\in\; \mathbb{R}^{BT \times D_{\mathrm{out}}},
\label{eq:fwd}
\end{equation}
and the weight gradient of the training loss $\mathcal{L}$ is a single dense matrix multiply,
\begin{equation}
\nabla_{\mathbf{W}}\mathcal{L} \;=\; \Big(\tfrac{\partial \mathcal{L}}{\partial \mathbf{Y}}\Big)^{\!\top} \mathbf{X} \;\in\; \mathbb{R}^{D_{\mathrm{out}} \times D_{\mathrm{in}}},
\label{eq:wgrad}
\end{equation}
with upstream gradient $\dY := \partial\mathcal{L}/\partial\mathbf{Y} \in \mathbb{R}^{BT \times D_{\mathrm{out}}}$. Standard training stores $\nabla_{\mathbf{W}}\mathcal{L}$ in HBM; the optimizer later reads the gradient and both moments and writes back the moments and updated weights. Peak memory therefore follows Equation~\ref{eq:peak}, whose $2P$ gradient term no other remedy touches. \forge{} removes it by folding the optimizer's state update inside the kernel that produces it.

\begin{figure*}[t]
\centering
\includegraphics[width=0.82\textwidth]{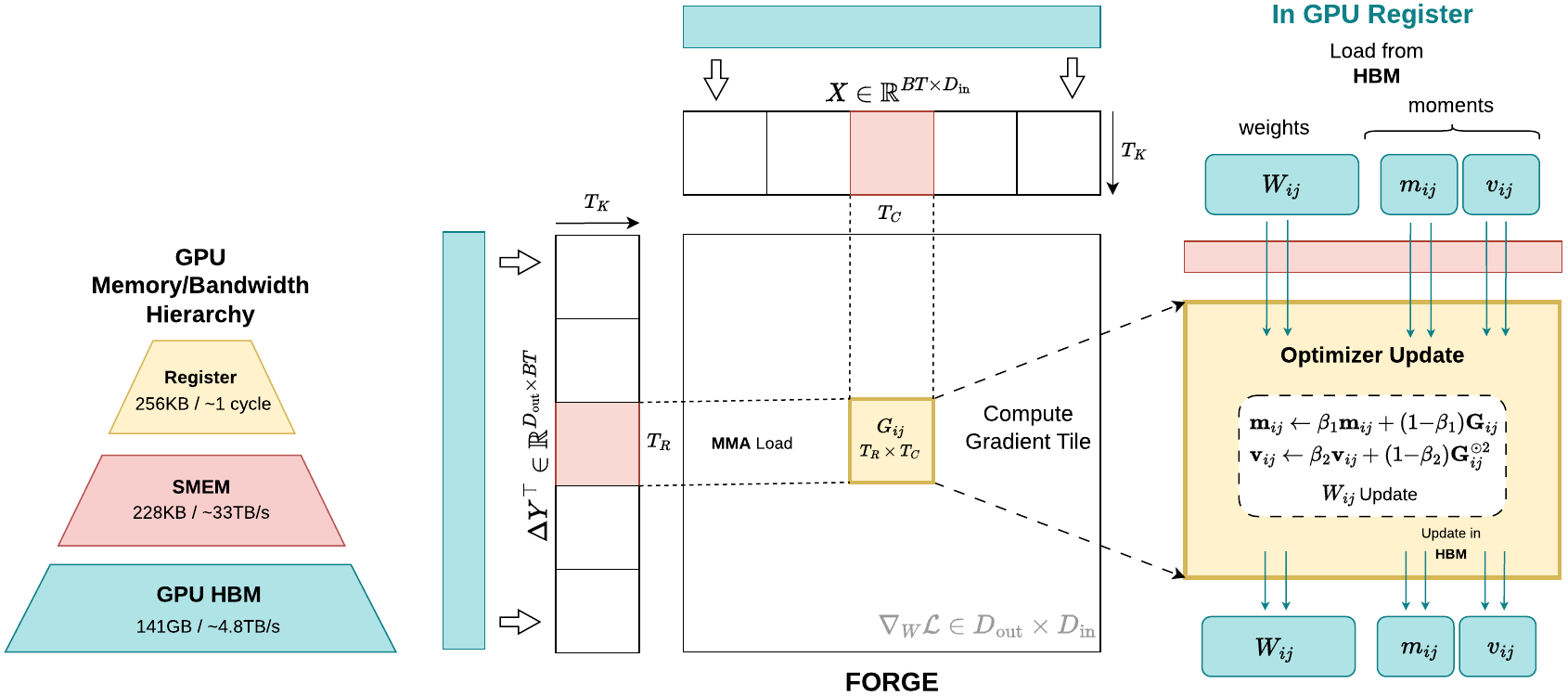}
\caption{The fused step. Each thread block forms one gradient tile in registers, applies the optimizer there, and writes back only the updated weight and moment tiles; the gradient never descends the memory hierarchy.}
\label{fig:mechanism}
\end{figure*}

\subsection{Per-Tile Fused Backward and Optimizer}
\label{sec:fused}

We partition $\mathbf{W}$ into $T_R{\times}T_C$ tiles; the tile geometry, reduction chunk $T_K$, and pipeline depth are chosen per layer shape by an autotuner (Appendix~H). Each thread block owns one tile $(i,j)$ and loops over the $BT$ axis in chunks, accumulating
\begin{equation}
\mathbf{G}_{ij} \;=\; \sum_{\kappa=0}^{\lceil BT/T_K\rceil - 1} \big(\dY_{\kappa,i}\big)^{\!\top} \mathbf{X}_{\kappa,j}
\label{eq:tile}
\end{equation}
in fp32 registers, with $\kappa$ indexing $T_K$-chunks of the $BT$ axis and $i,j$ the tile blocks (Algorithm~\ref{alg:forge}). When the loop finishes, $\mathbf{G}_{ij}$ is the exact $(i,j)$ block of $\nabla_{\mathbf{W}}\mathcal{L}$. The same registers then feed the optimizer: the moment tiles are loaded, updated, and stored, the weight tile is updated in place, and $\mathbf{G}_{ij}$ is discarded. No gradient byte reaches global memory, no separate optimizer kernel launches, and cross-layer execution order is unchanged. One ordering constraint follows from updating the weight in place. The layer's input gradient $\Delta\mathbf{X} = \dY\,\mathbf{W}$ must be read from the pre-update weight, so the kernel issues that GEMM first and applies the optimizer afterwards; $\Delta\mathbf{X}$ is otherwise computed as usual, and \forge{} changes only the weight-gradient path.

\begin{algorithm}[t]
\caption{\forge{}: fused backward and optimizer for weight tile $(i,j)$. AdamW shown: $\mathbf{m},\mathbf{v}$ the moments, $\eta$ learning rate, $\beta_1,\beta_2$ their decays, $\epsilon$ stabilizer, $\lambda$ weight decay.}
\label{alg:forge}
\begin{algorithmic}[1]
\REQUIRE $\dY,\ \mathbf{X}$; tiles $\mathbf{W}_{ij}, \mathbf{m}_{ij}, \mathbf{v}_{ij}$; step $t$; $(\eta,\beta_1,\beta_2,\epsilon,\lambda)$
\STATE $\Delta\mathbf{X} \leftarrow \dY\,\mathbf{W}$ \hfill $\triangleright$ layer-level; before the tiled loop
\STATE $\mathbf{G}_{ij} \leftarrow \mathbf{0}$ \hfill $\triangleright$ fp32 registers
\FOR{$\kappa = 0, \dots, \lceil BT/T_K \rceil - 1$}
    \STATE $\mathbf{G}_{ij} \leftarrow \mathbf{G}_{ij} + \dY_{\kappa,i}^{\!\top}\,\mathbf{X}_{\kappa,j}$
\ENDFOR
\STATE load $\mathbf{W}_{ij}, \mathbf{m}_{ij}, \mathbf{v}_{ij}$ from HBM \hfill $\triangleright$ $\mathbf{G}_{ij}$ stays in registers
\STATE $\mathbf{m}_{ij} \leftarrow \beta_1\mathbf{m}_{ij} + (1{-}\beta_1)\,\mathbf{G}_{ij}$ \hfill $\triangleright$ lines 7--10 are
\STATE $\mathbf{v}_{ij} \leftarrow \beta_2\mathbf{v}_{ij} + (1{-}\beta_2)\,\mathbf{G}_{ij}^{\odot 2}$ \hfill $\triangleright$ only optimizer-
\STATE $\hat{\mathbf{m}} \leftarrow \mathbf{m}_{ij}/(1{-}\beta_1^{t})$; \quad $\hat{\mathbf{v}} \leftarrow \mathbf{v}_{ij}/(1{-}\beta_2^{t})$ \hfill $\triangleright$ specific part
\STATE $\mathbf{W}_{ij} \leftarrow (1{-}\eta\lambda)\mathbf{W}_{ij} - \eta\,\hat{\mathbf{m}}/(\sqrt{\hat{\mathbf{v}}}+\epsilon)$
\STATE store $\mathbf{W}_{ij}, \mathbf{m}_{ij}, \mathbf{v}_{ij}$ \hfill $\triangleright$ $\mathbf{G}_{ij}$ discarded, never stored
\end{algorithmic}
\end{algorithm}

\paragraph{Memory and bandwidth.}
The peak of Equation~\ref{eq:peak} loses its gradient term:
\begin{equation}
\mathcal{M}_{\forge} = \underbrace{2P}_{\text{weights}} + \underbrace{kP}_{\text{states}} + \underbrace{0}_{\text{grads}} + \sum_{\ell}|a_\ell| + O(T_R T_C),
\label{eq:mforge}
\end{equation}
where the last term is the register working set of the in-flight tiles, independent of $P$. Nothing else changes: weights, moments, and precisions are exactly the baseline's, so every comparison in Section~\ref{sec:experiments} stays within a single recipe. The deletion is also a bandwidth cut. A standard bf16 step streams sixteen bytes per parameter through HBM, the gradient, both moments, and the weight each written once and read once, of which the gradient's four bytes are a pure round trip: written only to be read again. \forge{} forms the gradient in registers, so the step moves twelve bytes instead of sixteen. For a memory-bound step that bounds the gain at $16/12 = 1.33\times$ when both arms stream at the same efficiency, which is what Section~\ref{sec:experiments} measures against fused AdamW ($1.22\times$). Against vanilla AdamW the gain is larger, $1.52\times$, because that path reaches only $24\%$ of the HBM ceiling where \forge{} reaches $74\%$ : \forge{} moves fewer bytes and moves them closer to peak bandwidth. With the gradient floor gone, quantizing states or weights converts into peak savings byte for byte (int8 moments and FP8 weights in Section~\ref{sec:experiments}).

What bounds $T_R$ and $T_C$ is the register file rather than HBM. At the $128{\times}128$ tile the autotuner selects for most shapes, the fp32 accumulator holds $16{,}384$ values across eight warps, or 64 registers per thread against the 255 a thread may address, leaving room for the moment and weight tiles the epilogue streams through; wider tiles spill (Appendix~H).

\begin{figure}[H]
\centering
\includegraphics[width=0.88\linewidth]{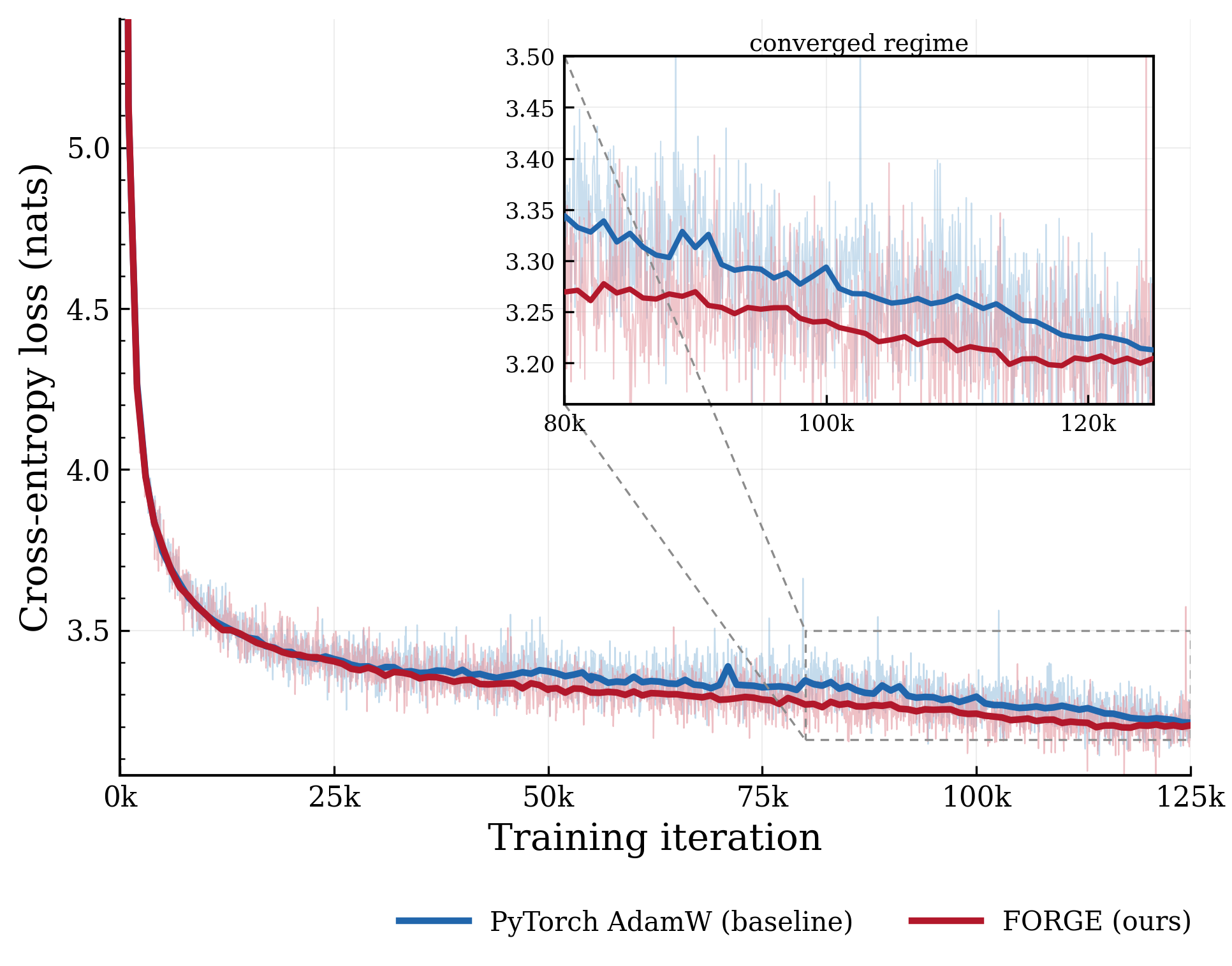}
\caption{GPT-2 124M pretrained from scratch on FineWeb-Edu \citep{fineweb}, held-out loss: \forge{} tracks fused AdamW throughout (3.20 vs 3.22 nats).}
\label{fig:gpt2}
\end{figure}

\paragraph{Exactness.}
Write a step as a state update $S_t = A(S_{t-1}, \nabla_{\mathbf{W}}\mathcal{L})$
then a weight update $W_t = U(W_{t-1}, S_t)$. Only $A$ is constrained: $U$ reads
persistent state, never the gradient, so it may read a row, a matrix, or a whole
layer. When $A$ is \emph{tile-local} --- a coordinate's new state depending only
on its own gradient, state and weight --- Equation~\ref{eq:tile} adds every $BT$
chunk exactly once, so stepping tile by tile equals forming the full gradient and
stepping once (Appendix~A). AdamW and seven other families have $A$ and $U$ both
tile-local and fuse entirely in the epilogue. Muon \citep{muon} does not:
$\mathbf{B} \leftarrow \mu\mathbf{B} + \mathbf{G}$ is tile-local, but
Newton--Schulz reads the whole matrix --- of $\mathbf{B}$, not of the gradient,
so $A$ fuses and $U$ runs afterwards over resident state. LAMB \citep{lamb} is
the same shape: elementwise moments, then a layer-global norm of the update they
produce. What does not fuse is a state update needing a statistic of the
gradient --- Adafactor and SM3 \citep{sm3}, row and column reductions of
$\mathbf{G}^{\odot 2}$, and LARS \citep{lars}, which reads $\|\mathbf{G}\|$.
These run deferred, each layer's gradient formed, consumed and freed inside its
own backward, and only
cross-layer statistics fall outside the schedule entirely
(Section~\ref{sec:limitations}). Appendix~I sorts all thirteen.

\section{Experiments}
\label{sec:experiments}
\forge{} is implemented as a single Triton kernel replacing the
weight-gradient GEMM of every linear layer, with no architecture-specific code
across GPT-2 \citep{gpt2}, Llama-3.1-8B, five Qwen3 sizes, vision transformers
to 25B, Mamba-2 to 20B, and MLP-Mixers to 4.9B. Single-GPU measurements use one
NVIDIA H200 (141~GB), BF16 everywhere, batch~1 and sequence~512 unless a
caption says otherwise; Appendix~G repeats the grids on H100 and B200. Every
comparison is drawn within one precision recipe and state class, every arm
passes a 20-step fp32-parity gate before any cell is recorded, and the full
protocol, library versions, seeds and timing method are in Appendix~C.
Standalone the update reaches 74\% of the measured 4{,}252~GB/s HBM ceiling,
against 61\% for fused AdamW and 24\% for vanilla.

\begin{table}[H]
\centering
\footnotesize
\setlength{\tabcolsep}{6pt}
\begin{tabular}{@{}l r r r@{}}
\toprule
Method & Peak (GB) & Step (ms) & TF/s \\
\midrule
vanilla AdamW           & 62.04 & $167.1 \pm 18.0$ & $149 \pm 17$ \\
fused AdamW             & 60.08 & $134.3 \pm 14.1$ & $185 \pm 21$ \\
\midrule
FlashOptim              & 42.27 & $148.0 \pm 4.2$  & $167 \pm 5$ \\
GaLore $r{=}128$        & 37.24 & $149.2 \pm 1.9$  & $165 \pm 2$ \\
APOLLO $r{=}256$        & 38.39 & $184.3 \pm 4.1$  & $134 \pm 3$ \\
optimi grad-release     & 64.55 & $188.7 \pm 7.6$  & $131 \pm 5$ \\
GaLore $r{=}1024$       & 44.77 & $189.9 \pm 2.3$  & $130 \pm 2$ \\
APOLLO-Mini $r{=}1$     & 36.15 & $195.4 \pm 14.9$ & $127 \pm 10$ \\
bitsandbytes 8-bit      & 45.36 & $316.3 \pm 20.6$ & $78 \pm 5$ \\
AdaLomo                 & 23.60 & $632.9 \pm 13.0$ & $39 \pm 1$ \\
\midrule
\textbf{\forge{}}        & \textbf{48.36} & $\mathbf{110.2 \pm 8.7}$ & $\mathbf{226 \pm 15}$ \\
\textbf{\forge{} (int8)} & \textbf{35.32} & $155.0 \pm 4.4$ & $159 \pm 5$ \\
\bottomrule
\end{tabular}
\caption{Single-GPU comparison on Llama-3.1-8B : peak memory, step time and achieved throughput (mean $\pm$ sd over three runs, each the median of 20 steps after warmup).}
\label{tab:headline}
\end{table}

\subsection{Precision and Convergence}
\label{sec:precisionconv}
The bf16 pipeline rounds every weight gradient to bf16 when the backward GEMM
writes it to HBM, a relative error of about $2^{-8}$, and the optimizer reads
that rounded value into both moments. The error does not average away: each
gradient enters the moving average scaled by $(1{-}\beta_1)$ but persists for
about $1/(1{-}\beta_1)$ steps, so the two factors cancel and the truncation
passes through undamped. \forge{} reads the fp32 register accumulator instead,
at no cost in bytes or time; prior comparisons never observed the difference
because both arms rounded the gradient the same way. Over twenty steps against
an fp32 reference \forge{}'s rms weight error sits below the standard path's
(5.83e-4 vs 6.09e-4), while the worst-case element runs the other way, the
tile-order reduction having reassociated the token-axis sum (Appendix~E).
Every implemented family passes the parity gate and trains end to end
(Appendix~I). In full runs GPT-2 124M pretrained from random initialization
tracks fused AdamW throughout (Figure~\ref{fig:gpt2}; 3.20 vs 3.22 nats), and
continued pretraining on OpenMathInstruct-2 \citep{openmathinstruct2} across
Llama-3.1-8B and five Qwen3 sizes stays within 0.001 nats on average
(Appendix~F).

\subsection{Headline Comparison}
\label{sec:headline}

Table~\ref{tab:headline} places \forge{} against standard memory-efficient training methods at the default setup. \forge{} improves all three axes at once: 110.2~ms per step against 134.3 for fused AdamW and 167.1 for vanilla AdamW, 48.4~GB peak against 60.1 at matched bf16 states and 35.3~GB with int8 moments, and 226 against 185~TF/s. The pattern behind the numbers is the granularity ladder of Section~1. Per-parameter release (optimi) frees each gradient as it is produced, but its per-parameter accumulator pool leaves the peak above the monolithic allocation it replaces (64.6~GB), and one optimizer launch per parameter costs it 188.7~ms. bitsandbytes pays 8-bit unpacking on every step (316.3~ms), and the low-rank methods (GaLore, APOLLO) change the update itself, buying memory at a different point on the fidelity axis. \forge{} pays neither the launches nor the materialization, so the memory saving and the speedup arrive together.

\subsection{Composability}
\label{sec:compose}

\begin{figure}[H]
\centering
\includegraphics[width=\linewidth]{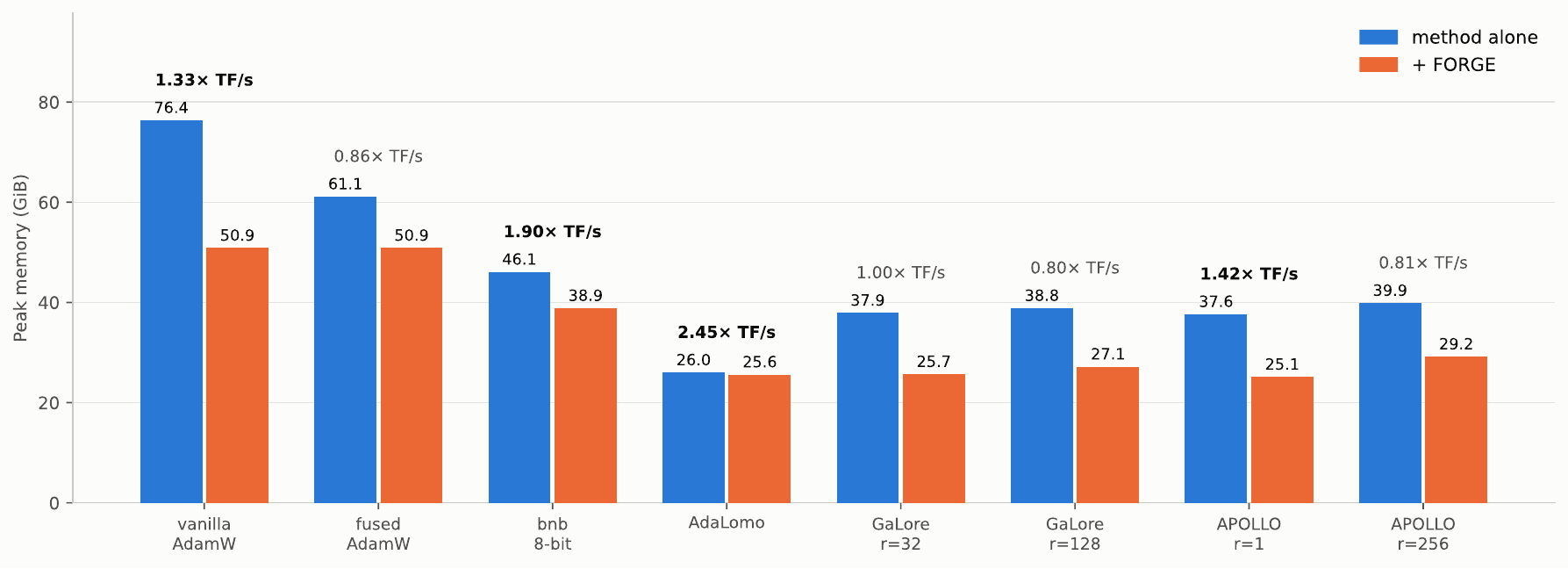}
\caption{Peak memory alone and composed with \forge{} (Qwen3-8B, one H200,
batch~1, sequence~512, bf16 states). Gradient elimination is orthogonal
to state quantization, low-rank projection and factored moments: each keeps its
own saving and sheds the pool on top. }
\label{fig:compose}
\end{figure}

\begin{figure}[H]
\centering
\includegraphics[width=\linewidth]{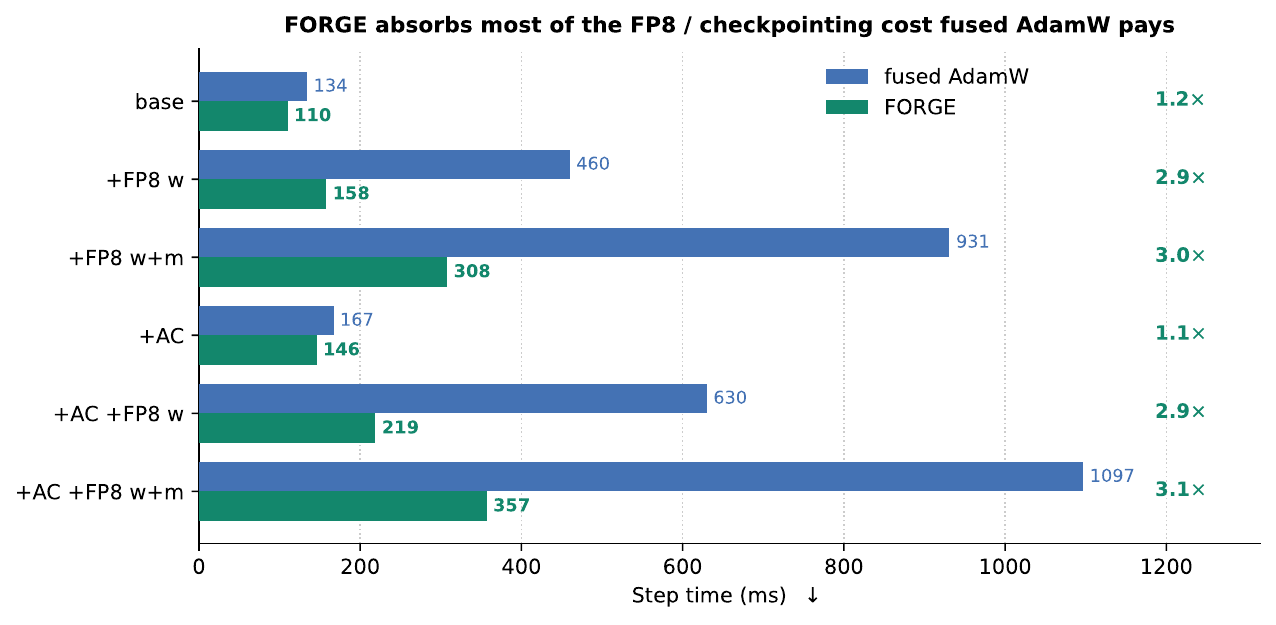}
\caption{Step time as FP8 weight (w) and weight+moment (w+m) quantization and
activation checkpointing (AC) are layered on (Llama-3.1-8B). \forge{} folds
dequantization into the per-tile epilogue; fused AdamW runs it as separate
kernels. FP8-moment cells use \texttt{AdamWFp8} from a reference run, agreeing
with the measured cells to within 1\%.}
\label{fig:variants}
\end{figure}

\forge{}'s gradient elimination is orthogonal to the other remedies: the
memory result is unconditional, the speed result is the bound
(Section~\ref{sec:method}). Each method in Table~\ref{tab:headline} shrinks a
different term of Equation~\ref{eq:peak}, and composing sheds the gradient
pool on top of each in turn (Figure~\ref{fig:compose}), $16$--$33\%$ of peak
against those that shrink state and leave the gradient alone, and
$2.45\times$ in step time against AdaLomo, already near the memory floor. Two
methods do not compose by construction: optimi and FlashOptim occupy the same
scheduling slot, and a weight cannot be stepped by both.

Quantizing what remains composes with removing the gradient
(Figure~\ref{fig:variants}). As FP8 weight and weight+moment quantization are
layered onto fused AdamW, its step climbs from 134 to 460 to 931~ms, because
dequantization runs as separate kernels on every forward, backward, and
update. \forge{} folds dequantization into the same per-tile epilogue, so its
step rises far less, from 110 to 158 to 308~ms, a $3.0\times$ gap at FP8 w+m.
Activation checkpointing composes the same way (146 vs 167~ms at base, and
$3.1\times$ at the full stack). The comparison isolates the dequantization tax
of the standard PyTorch path, which \forge{} largely absorbs; it claims no
advantage over FP8-native pipelines such as TransformerEngine.

\subsection{Operating Regime}
\label{sec:regime}

\begin{table}[H]
\centering
\small
\setlength{\tabcolsep}{4pt}
\begin{tabular}{@{}l ccc@{}}
\toprule
SEQ & BS=1 & BS=2 & BS=4 \\
\midrule
512  & 48.4 ($1.52\times$) & 51.8 ($1.60\times$) & 58.7 ($1.69\times$) \\
1024 & 51.8 ($1.58\times$) & 58.7 ($1.66\times$) & 72.4 ($1.52\times$) \\
2048 & 58.7 ($1.61\times$) & 72.4 ($1.50\times$) & 99.8 ($1.41\times$) \\
4096 & 72.4 ($1.47\times$) & 99.8 ($1.37\times$) & OOM \\
\bottomrule
\end{tabular}
\caption{\forge{} (bf16 states) across the batch$\times$sequence grid (Llama-3.1-8B, H200): peak memory in GB, with the step-time speedup over vanilla AdamW in the same cell.}
\label{tab:regime}
\end{table}

What governs the trade-off is the token count $BT$, not batch or sequence separately. Peak memory is constant along the constant-$BT$ anti-diagonals of Table~\ref{tab:regime}, exactly so in every group we measured, because the removed term scales with parameters and the surviving one with tokens. That asymmetry also bounds the benefit: \forge{} deletes a fixed ${\approx}15$~GB, so as activations inflate the peak the same absolute saving becomes a smaller share of it. At matched bf16 states the reduction falls from 22\% at $BT{=}512$ to nothing at $BT{\geq}4096$, where \forge{}'s peak meets the baseline's because activations, not the gradient, set the watermark. \forge{} is therefore a small-$BT$ method, which is the regime that dominates fine-tuning and continued pretraining; beyond it, activation checkpointing rather than gradient elimination is what moves the peak (full grids in Appendix~G).

\subsection{Model Scale and Capability}
\label{sec:capability}

\begin{figure}[H]
\centering
\includegraphics[width=\linewidth]{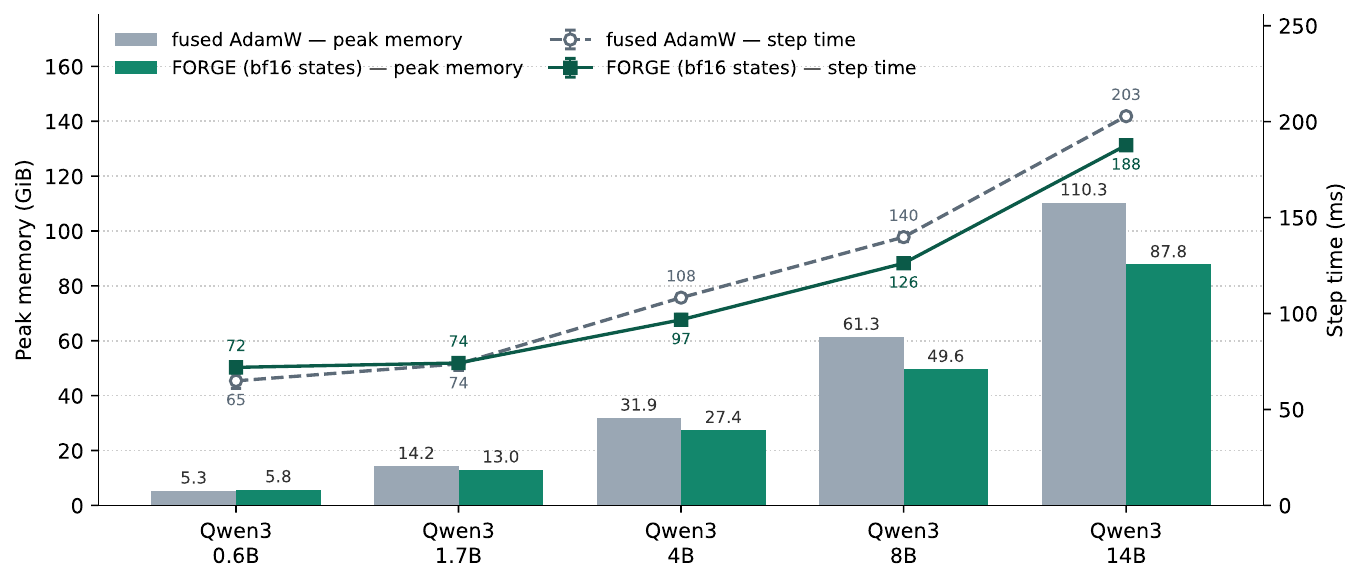}
\caption{Peak memory (bars, left) and step time (lines, right; 3-run mean $\pm$ sd) across the Qwen3 series (sequence~512, quiet H200).}
\label{fig:scaling}
\end{figure}

Figure~\ref{fig:scaling} scales the model axis: the memory ratio improves with parameter count, because optimizer state and the gradient transient grow with $P$ while activation memory does not at fixed tokens, from $1.08\times$ at Qwen3-0.6B to $0.80\times$ at Qwen3-14B. Step time follows the same amortization: $0.90\times$ at 0.6B, $1.08$--$1.12\times$ from 4B up. At Qwen3-14B fused AdamW peaks at a sequence-invariant states-plus-gradients watermark of 110.3~GB while \forge{} holds 87.8~GB at matched states and 62.1~GB with fp8 moments, all fitting the H200. The same structure decides what fits at all: at Qwen3-32B the standard path needs roughly $8$~bytes per parameter and \forge{} at matched bf16 states roughly $6$, both beyond the card; even int8 moments miss, reaching 137.0 of 139.8~GB before failing. With fp8 moments \forge{} trains it at 134.4~GB, where every other arm is out of memory. That is the point at which eliminating the gradient stops being an optimization and becomes the difference between training and not (six-platform capability in Appendix~O).

\subsection{Beyond the Transformer}
\label{sec:archgen}

\begin{figure}[H]
\centering
\includegraphics[width=0.96\linewidth]{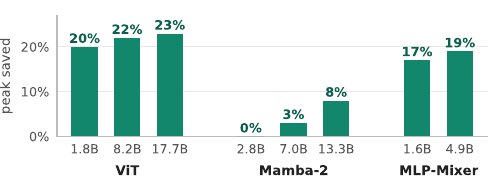}
\caption{Peak-memory saving vs.\ fused AdamW at matched bf16 states (H200), three architecture families, one kernel. Mamba-2's flat bars are interface coverage ($65\%$ of its weights reach the kernel), not architecture. Absolute footprints and int8 states: Appendix~P.}
\label{fig:archgen}
\end{figure}

Section~\ref{sec:method} makes architecture-independence a prediction rather than a hope: the fused step reads only the identity $\nabla_{\mathbf{W}}\mathcal{L}=\dY^{\!\top}\mathbf{X}$, so any weight read by a single backward path is the same case whatever surrounds it. We test the prediction at its extremes: a pure vision transformer \citep{vit}, a Mamba-2 selective-scan state-space model \citep{mamba2} with no attention, and an all-MLP Mixer \citep{mixer} that is almost nothing but linear layers. The identical kernel that trains the language models trains all three, with no architecture-specific code (Figure~\ref{fig:archgen}). The saving follows Equation~\ref{eq:mforge}: it climbs with scale toward the $25\%$ bf16 ceiling as the activation floor washes out, and with int8 states it reaches $46\%$ and trains a 25B vision model and a 20B state-space model on one H200, sizes fused AdamW cannot hold at any batch size. What varies across architectures is interface coverage rather than mathematics: $99.9\%$ of the Mixer's parameters reach the kernel, $99\%$ of the vision transformer's, and $65\%$ of Mamba-2's, whose optimized scan bypasses the module interface. That is an integration seam, not a research obstacle (Appendix~P).

\section{Distributed FORGE}
\label{sec:parallelism}

The fused step is exact when the rank that steps a weight block holds that block's complete gradient (Appendix~A), and each parallelism strategy fixes the finest unit in which that is true. \forge{} materializes the gradient in that unit and nothing larger: a register tile under tensor, sequence, and expert sharding, where a shard's owner already sees every token; one reduction bucket under data and context parallelism, where the local gradient is a partial sum over tokens and an update nonlinear in $\mathbf{G}$ cannot be applied to the pieces; one reduce-scattered shard under full sharding; one accumulator under a pipeline, where the incompleteness is in time rather than across ranks. No exact scheme materializes less, so the $O(P)$ pool the baseline holds above that floor is what \forge{} deletes. Exactness is gated rather than assumed (Appendix~L), every comparison is drawn within one recipe $\times$ sharding class (Appendix~K), and two implementation costs are charged to \forge{} throughout: the coordinator issues one optimizer launch per parameter per bucket rather than a fused multi-tensor call, and \forge{}-FSDP all-gathers synchronously where FSDP2 prefetches, so distributed step times are an upper bound on the mechanism's cost.

\begin{figure}[t]
\centering
\includegraphics[width=0.85\linewidth]{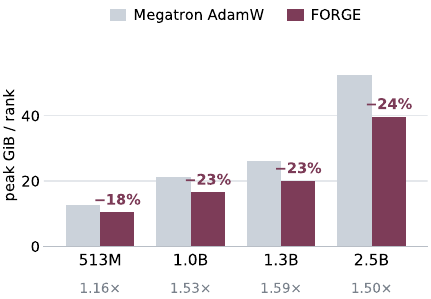}
\caption{Tensor parallelism (megatron.core, fp32-master both arms, quiet node): activation-set peaks move $0.0\%$, parameter-set peaks drop $18$--$24\%$ at $1.16$--$1.59\times$.}
\label{fig:tp}
\end{figure}

\paragraph{Tensor and sequence parallelism (Figure~\ref{fig:tp}).}
A shard's gradient is complete, so the epilogue fuses with no distributed machinery and the production kernel runs on every rank. Across TP~2--8 on 8$\times$RTX PRO 6000 the low-variance quiet-node A/B cells run at $1.00$--$1.03\times$ at matched bf16 states: fusion costs nothing. What it buys follows the single-GPU law of Section~3.4. On Megatron's own fp32-master recipe, cells where the forward crest sets the peak move $0.0\%$ and cells where parameter state sets it drop $18$--$24\%$ at $1.16$--$1.59\times$.

\begin{figure}[H]
\centering
\includegraphics[width=\linewidth]{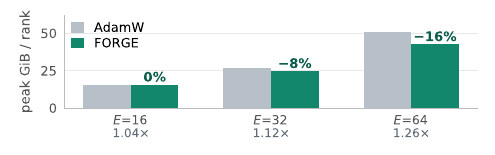}
\caption{Pure expert parallelism (EP4): parity while activations set the peak, $-16\%$ at $1.26\times$ once experts pack the rank at $E{=}64$.}
\label{fig:ep}
\end{figure}

\paragraph{Expert parallelism (Figure~\ref{fig:ep}).}
Each expert owns its complete gradient, so the fused update is exact with no reduction in either arm, the cleanest timed comparison in the paper. On a mixture-of-experts model of 12 layers and width 2048, EP8 at 16 experts is at memory parity (117 against 118~ms); as experts pack the rank the deleted gradient emerges, $-2.2$~GB at $E{=}32$ and $-8.2$~GB or $-16\%$ at $E{=}64$, at $1.12\times$ and $1.26\times$ the baseline's speed. Adding data parallelism moves the granule down one rung, at $+0.37$~GB (Appendix~N).

\paragraph{Data parallelism (Figure~\ref{fig:dp}).}
Weight-gradient tiles stream into a rotating bucket; once all-reduced it holds a complete gradient for its parameters, feeds one fused pass overlapped with the ongoing backward, and is recycled. On the 96~GB PCIe node that pool is exactly what stops the baseline: at Qwen3-8B, micro-batch 2, sequence 2048, fused AdamW under DDP is out of memory while \forge{} trains at 84.0~GB and, with its int8 sibling, gives the two fastest arms on the node (1537 and 1520~ms against 1886 for the best sharded baseline). Deleting the pool keeps the cheap replicated schedule viable where sharding pays PCIe collectives. The allocator decomposition pins the cause: the only segment that differs is the pool, 15.28~GB and growing under DDP against a flat 0.39~GB bucket (Appendix~N). The transient is $O(\text{bucket})$ and the replica axis shows it: 83.97~GB at $N{=}1,2,4,8$ alike, while DDP is out of memory at every $N$ including $N{=}1$. Against production bitsandbytes \cite{bnb8bit} at matched int8 moments the arm is smaller and faster at every size, $3\%$ to $20\%$ from 0.6B to 8B, and trains the 14B it cannot; the stochastic-rounding store \citep{gupta2015} tracks the fp32-master trajectory twice as closely as plain bf16 at identical memory (Appendix~N).

\begin{figure}[tb]
\centering
\includegraphics[width=0.85\linewidth]{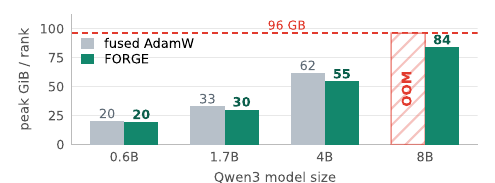}\\[3pt]
\includegraphics[width=0.85\linewidth]{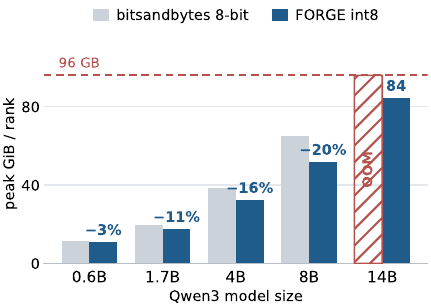}
\caption{Data parallelism on 8$\times$RTX PRO 6000 (96~GB). Top: matched bf16 states (DP8), fused AdamW out of memory at 8B while \forge{} trains at 84~GB. Bottom: matched int8 moments ($N{=}2$, micro-batch 1, sequence 2048), where the boundary moves in that class too.}
\label{fig:dp}
\end{figure}

\paragraph{Context parallelism.}
The same schedule at a different shape marks the boundary. At CP8 on a 2.5B model the live gradient collapses from 4.68 to 0.28~GB at the end of backward, but from 16k tokens the peak sits at the forward crest, where no gradient exists in either arm, so memory is parity plus the $0.37$~GB pool, flat to 131k tokens with and without activation checkpointing, at $6$--$12\%$ faster steps. One mechanism: 14.9~GB saved where parameters set the peak, 0.37~GB paid where activations do.

\paragraph{Fully sharded.}
The reduce-scatter hands each rank its complete shard and the fused pass runs on arrival. At the 8$\times$H200 anchor on a matched fp32-master recipe \forge{} runs at 53.69~GB per rank against 67.51 for DeepSpeed's ZeRO-3 \cite{zero}, both arms holding the master (Appendices~M--N); at 32B every replicated arm and ZeRO-3 are out of memory while \forge{} trains three recipes, down to 70.55~GB with int8 states. Against FSDP2 \cite{fsdp} at matched bf16 states the arms sit at 46.28 and 46.64~GB and \forge{} costs step time. That is predicted, not excepted: FSDP2 already frees each gradient shard as it reduces it, so its granule is the shard and nothing is left above the floor. The residual gap is \forge{}'s replicated embeddings, which makes our fully sharded savings conservative (Appendix~N).

\paragraph{Pipelining and composition (Appendix~N).}
Under a stashing schedule the earlier micro-batches' graphs still reference the pre-update weight, so an in-place update before the last is invalid; \forge{}'s exact-accumulation mode makes the accumulator the granule, at $+2.6\%$ with matched bf16 accumulation. Composing tensor, context, and data parallelism, the TP shards step locally while the coordinator completes the sum over exactly the CP$\times$DP slice that holds partials: across twenty-two cells the arms sit at memory parity within $+0.10$--$0.42$~GB, with losses in family at every configuration including the three-way TP2$\times$CP2$\times$DP2. On NVLink the composed arms run $1.12$--$1.56\times$ faster from 4B up and at parity below; on PCIe the current implementation is $2$--$2.6\times$ slower, for causes we name in Appendix~N, while the memory result holds on both.

\section{Limitations and Conclusion}

\label{sec:limitations}

All our numbers come from one kernel. We did not specialize it per architecture, optimizer or parallelism strategy; only the autotuner varies. \forge{} helps most when parameters rather than activations dominate the peak, where fine-tuning and continued pretraining sit. Outside that regime we still remove the same bytes at the backward-to-optimizer boundary; what changes is whether the peak is still there to lower. A long forward crest moves it, and checkpointing cuts that crest, so the two work together: an 8B model saves 14.9~GB with it on (Appendices~M--N).

Two things constrain us. Micro-batch accumulation needs a gradient-shaped buffer to sum into, so the pool comes back. And an optimizer whose state update reads a gradient statistic, Adafactor and SM3 among them, cannot close it inside a tile; it takes a deferred path that frees the gradient inside one layer's backward, so residency falls to the largest layer, not to zero. Registers and shared memory are fixed per SM and cap the tile sizes the autotuner can pick (Appendix~H), leaving launch overhead visible below about 1B parameters.

Within those limits the claim holds: the materialized weight gradient is an artifact of how the backward pass is scheduled, and removing it buys capacity and speed.

\bibliography{refs}

\clearpage
\onecolumn
\raggedbottom
\appendix
\renewcommand{\thesection}{\Alph{section}}

\begin{center}
{\LARGE\bfseries Technical Appendix}
\end{center}
\vspace{1.2em}

\noindent Every section of this appendix supports a specific claim in the main paper and is referenced from the sentence it backs. Sections: \textbf{A} exactness, \textbf{B} precision theory, \textbf{C} single-GPU measurement protocol, \textbf{D} update-path microbenchmarks, \textbf{E} numerical-fidelity tables, \textbf{F} convergence protocols, \textbf{G} extended single-GPU grids (H200, H100, and B200), \textbf{H} autotuning, \textbf{I} optimizer families end to end, \textbf{J} method taxonomy, \textbf{K} distributed protocol and recipes, \textbf{L} distributed exactness gates, \textbf{M} distributed results on 8$\times$H200, \textbf{N} distributed results on 8$\times$RTX PRO 6000, \textbf{O} cross-platform capability (A100-40/80GB, H200, RTX PRO 6000, B200, B300), \textbf{P} cross-architecture generalization (ViT, Mamba-2, and MLP-Mixer).

\section{Exactness}\label{app:exact}

\paragraph{Setup.} A linear layer maps $\mathbf{X}\in\mathbb{R}^{BT\times H}$ to $\mathbf{Y}=\mathbf{X}\mathbf{W}^{\!\top}\in\mathbb{R}^{BT\times V}$ with $\mathbf{W}\in\mathbb{R}^{V\times H}$. Writing $\dY=\partial\mathcal{L}/\partial\mathbf{Y}$, the weight gradient is $\mathbf{G}=\dY^{\!\top}\mathbf{X}$, i.e. $G_{pq}=\sum_{t=1}^{BT}\Delta Y_{tp}X_{tq}$. Nothing below uses any property of $\mathbf{X}$ beyond this identity, so the results hold for any linear layer regardless of the surrounding architecture.

\subsection{The state--weight decomposition}

Every step considered here factors into a \emph{state update} followed by a \emph{weight update},
\begin{equation}
S_t \;=\; A\big(S_{t-1},\,\mathbf{G}\big),
\qquad
\mathbf{W}_t \;=\; U\big(\mathbf{W}_{t-1},\,S_t\big),
\label{eq:au}
\end{equation}
The factorization is forced, not chosen: a weight update that read $\mathbf{G}$ directly is indistinguishable from one that first wrote it into state. The schedule therefore constrains $A$ alone --- $U$ reads state, which is resident by construction --- and that condition is strictly weaker than requiring the whole step to be coordinate-wise.

\begin{definition}[Tile-local state update]\label{def:tl}
The state update $A$ is \emph{tile-local} if it carries per-coordinate state $S_{pq}$ and computes every new entry through one shared map
\[
S'_{pq}\;=\;\alpha\big(S_{pq},\,G_{pq},\,W_{pq};\,\theta\big)
\]
that reads only the matching coordinate of $(S,\mathbf{G},\mathbf{W})$. A state update that couples coordinates through a row, column, block, or global statistic \emph{of the gradient} is not tile-local. AdamW (with $S=(m,v)$), SGD, SGD with momentum, RMSprop, Adagrad, Lion, NAdam, RAdam, LAMB, Muon and Scion all have tile-local $A$; Adafactor, AdaLomo, SM3, Adam-mini and LARS do not.
\end{definition}

\begin{definition}[Tile-local and state-global weight updates]\label{def:tlu}
The weight update $U$ is \emph{tile-local} if $W'_{pq}=\upsilon(W_{pq},S'_{pq};\theta)$ for a shared map $\upsilon$, and \emph{state-global} if it reads a row, matrix, or layer of $S'$. A state-global $U$ still reads no gradient: its argument is the state $A$ has just produced.
\end{definition}

\subsection{Per-tile exactness}

\begin{proposition}[Per-tile exactness]\label{prop:tile}
Let $A$ be tile-local and let $\{\mathcal{T}_k\}$ be any partition of the weight index set $[V]\times[H]$ into tiles. Consider the tile-fused schedule that, for each tile $\mathcal{T}_k$, forms the gradient block by accumulating over a partition $\{\kappa\}$ of the token axis,
\[
\mathbf{G}|_{\mathcal{T}_k}=\sum_{\kappa}(\dY_\kappa)^{\!\top}\mathbf{X}_\kappa,
\]
applies $\alpha$ coordinate-wise on $\mathcal{T}_k$, and discards $\mathbf{G}|_{\mathcal{T}_k}$. In exact arithmetic the state $S'$ it produces is identical, coordinate for coordinate, to forming the full gradient $\mathbf{G}$ and applying $A$ in a single shot; consequently $U(\mathbf{W},S')$ equals the one-shot weight update for \emph{every} $U$, tile-local or state-global.
\end{proposition}

\begin{proof}
Three steps. \emph{(Gradient correctness.)} Fix $(p,q)\in\mathcal{T}_k$. Since $\{\kappa\}$ partitions the token index set, the token-chunk accumulation adds each contribution exactly once: $\sum_\kappa\sum_{t\in\kappa}\Delta Y_{tp}X_{tq}=\sum_{t=1}^{BT}\Delta Y_{tp}X_{tq}=G_{pq}$. The tile therefore holds exactly the sub-block $\mathbf{G}|_{\mathcal{T}_k}$ of the full gradient. \emph{(Separability of $A$.)} By Definition~\ref{def:tl} the new state at $(p,q)$ is $\alpha(S_{pq},G_{pq},W_{pq};\theta)$, a function of that coordinate alone; it depends neither on which tile contains $(p,q)$ nor on the order in which tiles are visited, so the union of the per-tile maps over $[V]\times[H]$ is the one-shot $A$. Both $\alpha$ and $\upsilon$ touch only coordinates of the tile in flight, so every $\alpha$ reads a pre-update weight exactly as the one-shot schedule does. \emph{(Invariance of $U$.)} $U$ is a deterministic function of $(\mathbf{W}_{t-1},S_t)$; the previous step makes $S_t$ identical to the one-shot state and leaves $\mathbf{W}_{t-1}$ as the one-shot schedule found it, so $U$ receives identical arguments and returns an identical result. Nothing in this step uses the structure of $U$, which is why it is unconstrained.
\end{proof}

\paragraph{Four regimes.} Definitions~\ref{def:tl} and~\ref{def:tlu} sort an optimizer by two independent questions --- what $A$ reads, and what $U$ reads --- and the sort determines how many gradient bytes are resident at the backward-to-optimizer boundary.

\begin{center}\small
\setlength{\tabcolsep}{5pt}
\begin{tabular}{@{}c l l l@{}}
\toprule
regime & condition & the update runs & resident $\mathbf{G}$ \\
\midrule
1 & $A$ tile-local, $U$ tile-local & entirely in the tile epilogue & $O(T_RT_C)$ \\
2 & $A$ tile-local, $U$ state-global & $A$ in the epilogue, $U$ after, over state & $O(T_RT_C)$ \\
3 & $A$ reads a gradient statistic of one layer & deferred, inside that layer's backward & $O(\max_\ell P_\ell)$ \\
4 & $A$ reads a gradient statistic across layers & after the entire backward & $O(P)$ \\
\bottomrule
\end{tabular}
\end{center}

\noindent Regimes~1 and~2 both eliminate the gradient; regime~2 gives up only that the weight write leaves the epilogue. Muon is the instance: $\mathbf{B}\leftarrow\mu\mathbf{B}+\mathbf{G}$ is tile-local and affine, so the weight-gradient multiply accumulates in place into the momentum buffer through the GEMM's own $\beta$ and the Newton--Schulz iteration then reads $\mathbf{B}$, never a gradient. Regime~3 moves the granule rather than losing it: a row or column reduction of $\mathbf{G}^{\odot2}$ cannot close inside a tile, but it is confined to one matrix, so the gradient is formed, consumed and freed inside that layer's backward and residency falls from $O(P)$ to the largest single layer. Appendix~\ref{app:zoo} sorts and measures every family, and the regime-3 rules turn out to gain the most.

\begin{remark}[Tiles that receive no gradient]\label{rem:untouched}
Proposition~\ref{prop:tile} applies $\alpha$ to every coordinate once per step. Folding the state's decay into the accumulate's $\beta$ --- what makes the fusion free --- applies $\alpha$ only where a contribution arrives, and the two sets coincide only when every tile is touched. Conditionally routed layers break that: an expert nothing routes to on a step must have the gradient-free part of $\alpha$ applied separately, or its state stops decaying while every other expert's decays. It is a property of $A$, not of the kernel.
\end{remark}

\paragraph{Floating-point characterization.} Proposition~\ref{prop:tile} is an exact-arithmetic statement; in floating point two cases arise. With fp32 tile accumulation and the same token-axis summation order as a reference that forms $\mathbf{G}$ in fp32 and then applies $A$, the operands are bit-identical and so are the results. Against any other implementation (a cuBLAS weight-gradient GEMM followed by a separate optimizer kernel), the sole source of discrepancy is the summation order of the token-axis reduction, a floating-point non-associativity intrinsic to the GEMM and independent of the fusion, of the magnitude two GEMM libraries already exhibit against each other. A state-global $U$ adds nothing to either case: it runs on a state the argument above has already made bit-identical, over the same layout, in the same order.

\begin{remark}[Ordering and multi-consumer weights]\label{rem:order}
Because \forge{} overwrites $\mathbf{W}$ in place during the backward pass, the chain rule requires the layer's input gradient $\Delta\mathbf{X}=\dY\,\mathbf{W}$ to be read \emph{before} the update; the kernel schedules $\Delta\mathbf{X}$ first and applies $\alpha$, then $U$, afterward. This single-consumer precondition fails when a weight is read by more than one backward path within a step, e.g. a tied embedding/LM head; such weights are excluded from fusion and stepped by the standard optimizer.
\end{remark}

\begin{remark}[Activation checkpointing]\label{rem:ac}
With checkpointing, each segment's forward recomputation runs at the start of that segment's backward, before any weight inside the segment is updated; updates flow deepest-segment-first. Every recomputation therefore observes exactly the weights its original forward observed, and the composition is exact.
\end{remark}

\subsection{Distribution}

\begin{proposition}[Distributed completeness]\label{prop:dist}
The fused update is exact precisely when the argument $A$ receives is the complete $\mathbf{G}$.
(i) \emph{Tensor, sequence, and expert parallelism} shard the weight itself (a row or column block of $\mathbf{W}$, or one expert); the token-axis contraction defining a shard's gradient is entirely rank-local, so Proposition~\ref{prop:tile} applies to the shard's coordinate subset with no coordination of any kind.
(ii) \emph{Data and context parallelism} replicate the weight, and rank $r$ computes a partial $\mathbf{G}^{(r)}$ with $\mathbf{G}=\sum_r\mathbf{G}^{(r)}$. If $A$ is nonlinear in its gradient argument --- AdamW's second moment is quadratic, $v\mapsto\beta_2 v+(1-\beta_2)\mathbf{G}^{\odot 2}$ --- then $\sum_r A(\cdot,\mathbf{G}^{(r)})\neq A(\cdot,\sum_r\mathbf{G}^{(r)})$ in general, and the state update must wait for the cross-rank sum.
(iii) If $A$ is \emph{affine} in its gradient argument, so that $A(S,\mathbf{G})=\mathcal{A}(S)+c\,\mathbf{G}$ for a gradient-free $\mathcal{A}$, the obstruction disappears. Seeding $\mathcal{A}(S)$ on exactly one rank of the group and $\mathbf{0}$ on the others gives $\sum_r A_r=\mathcal{A}(S)+c\sum_r\mathbf{G}^{(r)}=A(S,\mathbf{G})$, so reducing the \emph{state} is exact and no rank ever forms a gradient.
\end{proposition}

\paragraph{Bucket-transient schedule (the nonlinear case).} Under data and context parallelism with a nonlinear $A$, \forge{} reduces gradients bucket by bucket (bucket capacity matched to the framework's) and applies the fused tile-local step to each \emph{reduced} bucket, which holds a complete gradient; exactness then follows from Proposition~\ref{prop:tile} applied bucket-wise. The materialized transient is one rotating bucket, $O(\text{bucket})$ independent of $P$, in place of the full $O(P)$ gradient. Any exact scheme compatible with data parallelism must materialize at least the reduction granule, so this schedule meets that lower bound. Empirical gates for both cases are in Appendix~\ref{app:gates}.

\paragraph{Reduce-into-state (the affine case).} Heavy-ball momentum and Muon's $\mathbf{B}\leftarrow\mu\mathbf{B}+\mathbf{G}$ are Proposition~\ref{prop:dist}(iii), and the construction needs no change to the collective library: the accumulate already takes the decay as its $\beta$ on a step's first contribution and $1$ thereafter, so passing $\mu$ on one rank and $0$ on the others \emph{is} the zero seed, and an in-place \texttt{all\_reduce} of $\mathbf{B}$ reconstructs $\mu\mathbf{B}+\sum_r\mathbf{G}^{(r)}$ with no division and no staging buffer. Wire volume is unchanged; the object reduced is the momentum, so the bucket transient does not arise. Exact by Proposition~\ref{prop:dist}, not bit-exact in practice --- a collective's reduction tree does not reproduce a sequential summation order.

\begin{remark}[State-global $U$ under weight sharding]\label{rem:ushard}
A tile-local $U$ never crosses a shard boundary, because every coordinate it writes is one it read. A state-global $U$ reads a whole matrix of $S'$, which sharding has split, so it must gather it --- and the gathered volume is the layer's full parameter count \emph{independent of sharding degree}, since sharding divides what a rank owns and not what its $U$ reads. LAMB's $U$ reduces two scalars per layer; Muon's reads every entry. The cost belongs to the optimizer rather than to the fusion --- any implementation of Newton--Schulz over a sharded weight pays it, fused or not --- and it is the axis on which to choose between a regime-1 and a regime-2 rule under heavy sharding.
\end{remark}

\section{Precision Theory}\label{app:precision}

\paragraph{Undamped truncation bias.} In a BF16-everywhere pipeline the backward GEMM accumulates in fp32 but rounds $\nabla_{\mathbf{W}}\mathcal{L}$ to bf16 when writing it to HBM, a relative error of order $2^{-8}$; the optimizer reads that rounded value into both moments. Write the stored first moment as $\tilde m_t=\beta_1\tilde m_{t-1}+(1-\beta_1)(g_t+\varepsilon_t)$ with truncation $\varepsilon_t$. Unrolling, the induced deviation is $\Delta_t=\sum_{k\le t}\beta_1^{t-k}(1-\beta_1)\varepsilon_k$: each error enters scaled by $(1-\beta_1)$ but persists with horizon $\sum_j\beta_1^j=1/(1-\beta_1)$, so the two factors cancel and a systematic $\varepsilon$ passes through with magnitude $\sim\varepsilon$, undamped. \forge{}'s update reads the fp32 register accumulator, so $\varepsilon_t\equiv 0$ on this path.

\paragraph{INT8 states.} Moments are stored with a per-block absmax scale over 64-wide blocks: for block absmax $a$, scale $s=a/127$ and $q(x)=s\,\mathrm{round}(x/s)$, re-quantized every step. The stored-state error is a damped accumulation of per-step round-off: with $|\varepsilon_k|\le s/2$, $|\tilde m_t-m_t|\le\frac{s}{2}\sum_j\beta_1^j=\frac{s}{2(1-\beta_1)}$, and round-to-nearest makes the errors mean-zero with RMS $\approx s/\sqrt{12(1-\beta_1^2)}$; the $v$ bound follows with $\beta_2$. The $1/(1-\beta_2)$ horizon is why the second moment is the delicate state: at $\beta_2=0.999$ a quantized-to-zero entry would persist $\sim$1000 steps and blow up $1/\sqrt{v}$, so $v$ is floored away from zero. The 64-wide blocks confine an outlier's influence to its own block; the update arithmetic and the gradient remain fp32 throughout.

\section{Single-GPU Measurement Protocol}\label{app:protocol1}

\begin{center}\small
\setlength{\tabcolsep}{5pt}
\begin{tabular}{@{}l l@{}}
\toprule
memory metric & \texttt{torch.cuda.max\_memory\_allocated()}$/2^{30}$, reported as GB \\
process model & one process per cell (the autotuner pins its workspace for process life) \\
warmup & $\ge$30 iterations, 10 at the largest shapes; absorbs JIT and autotuning \\
measured & three runs per cell, each the median of 20 steps; headline cells \\
 & report mean $\pm$ sd across runs, grid cells the median \\
ratios & drawn only between cells from one session on one GPU \\
per-cell record & GPU index, SM and memory clocks, temperature, power, library \\
 & versions, source-tree hash, chosen autotuner configuration \\
gate & 20-step parity vs.\ an fp32 reference (App.~\ref{app:fidelity}) before any cell is \\
 & recorded; hyperparameters and seeded inputs identical across arms \\
\bottomrule
\end{tabular}
\end{center}

\noindent Percentages are against ceilings measured on the same device, never datasheet numbers: cuBLAS bf16 square GEMM 678.8~TF/s at $4096^3$ falling to 602.4 at $16384^3$; per-shape TN weight-gradient GEMMs 577--631~TF/s at $BT{=}32$k; HBM 4{,}252~GB/s triad and 4{,}087 device-to-device.

\section{Update-Path Microbenchmarks}\label{app:micro}

Achieved bandwidth of the optimizer step alone, and the marginal cost of the in-backward update over an identical kernel that stores the gradient instead (q/o\_proj shape, matched configurations).

\begin{figure}[tbp]
\centering
\includegraphics[width=\linewidth]{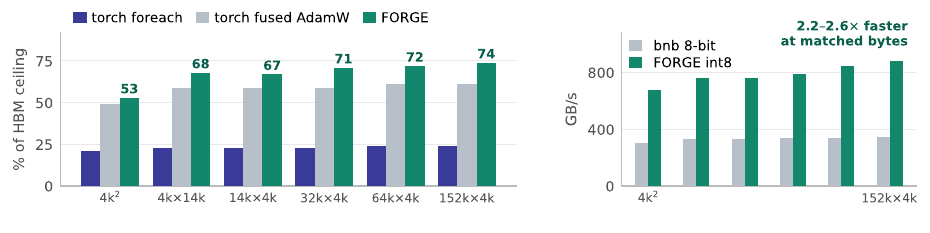}
\caption{Standalone update path. Left: achieved bandwidth as \% of the measured 4{,}252~GB/s ceiling across parameter-tensor sizes. Right: quantized-state arms at matched bytes per element (GB/s); \forge{}-int8 runs $2.2$--$2.6\times$ faster than bitsandbytes 8-bit at every size.}
\label{fig:suppbw}
\end{figure}

\begin{center}\small
\begin{tabular}{@{}l rr@{}}
\toprule
family (state tensors) & $BT{=}512$ & $BT{=}4096$ \\
\midrule
store-$\mathbf{G}$ control & 53~$\mu$s & 286~$\mu$s \\
SGD (0) & $+11$ & $+30$ \\
momentum / Adagrad / RMSprop / Lion (1) & $+49$--$59$ & $+82$--$92$ \\
Adamax / Adam / AdamW (2) & $+81$--$124$ & $+106$--$156$ \\
\bottomrule
\end{tabular}
\end{center}

\noindent Monotone in state count and nearly flat in $BT$, as a state-I/O-only cost model predicts.

\section{Numerical-Fidelity Tables}\label{app:fidelity}

Weight error after 20 steps against an fp32 reference (identical seeded inputs; $V{=}1024$, $H{=}4096$, $BT{=}2048$):

\begin{center}\small
\begin{tabular}{@{}l rr@{}}
\toprule
arm & rms-rel($W$) & max $|\Delta W|$ \\
\midrule
standard (cuBLAS $+$ torch fused AdamW) & 6.091e-4 & 7.324e-4 \\
cuBLAS $+$ \forge{} optimizer kernel & 6.096e-4 & 7.324e-4 \\
\forge{} fused (non-persistent) & \textbf{5.825e-4} & 8.545e-4 \\
\forge{} fused (persistent) & 5.825e-4 & 8.545e-4 \\
\forge{} fused (TMA) & 5.825e-4 & 8.545e-4 \\
\forge{} int8 states (qblock 64) & 2.588e-3 & 3.662e-3 \\
\forge{} fp8 states (per-tensor scale) & 2.560e-2 & 1.601e-2 \\
CUTLASS two-kernel (Hopper / Stream-K) & 6.096e-4 & 7.324e-4 \\
\bottomrule
\end{tabular}
\end{center}

\noindent The fused path's rms error sits below the decomposed path's because the gradient reaches the optimizer without the bf16 rounding applied when $\nabla_{\mathbf{W}}\mathcal{L}$ is materialized; the int8 mode's stored-state error is bounded in Appendix~\ref{app:precision}; the fp8 mode is per-tensor scaled with delayed scaling, so its error is reported here rather than bounded. Quantized modes are always labeled as such.

\section{Convergence Protocols and Curves}\label{app:convergence}

\paragraph{Continued pretraining.} OpenMathInstruct-2, batch 4, sequence 512, 20{,}000 steps, at least three seeds per method (\texttt{torch.manual\_seed}), across Llama-3.1-8B and Qwen3 \{0.6B, 1.7B, 4B, 8B, 14B\}; loss deltas are reported as the maximum across seeds and curves as the across-seed mean, reproducible under deterministic-algorithms mode. End-of-run smoothed deltas of \forge{} against fused AdamW average 0.001 nats (0.003 worst case) across the family; both bf16 arms sit the known BF16-vs-FP32 margin above the fp32-master reference ($\le 0.009$ nats), a precision-recipe floor independent of \forge{}. Figures~\ref{fig:cpt-main} and \ref{fig:cpt-family} reproduce the curves.

\begin{figure}[tbp]
\centering
\includegraphics[width=0.80\linewidth]{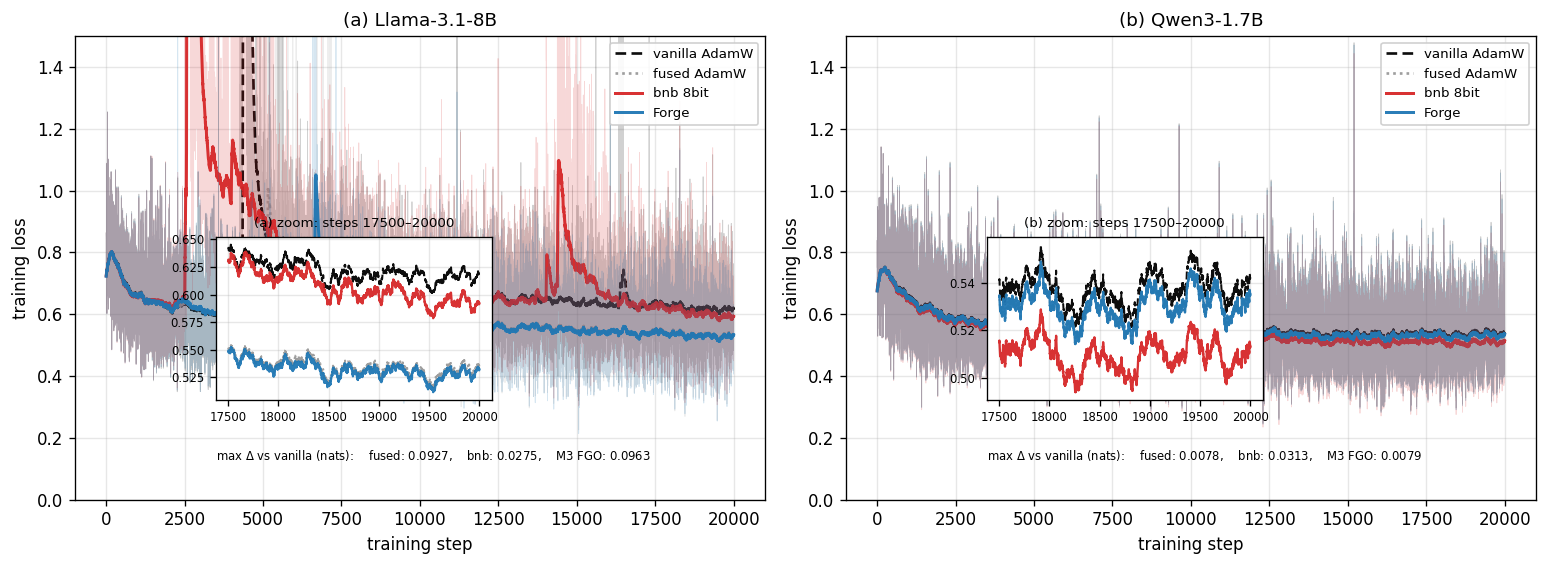}
\caption{Continued-pretraining parity: Llama-3.1-8B (left) and Qwen3-1.7B (right), with insets over the final 2{,}500 steps.}
\label{fig:cpt-main}
\end{figure}

\begin{figure}[tbp]
\centering
\includegraphics[width=0.44\linewidth]{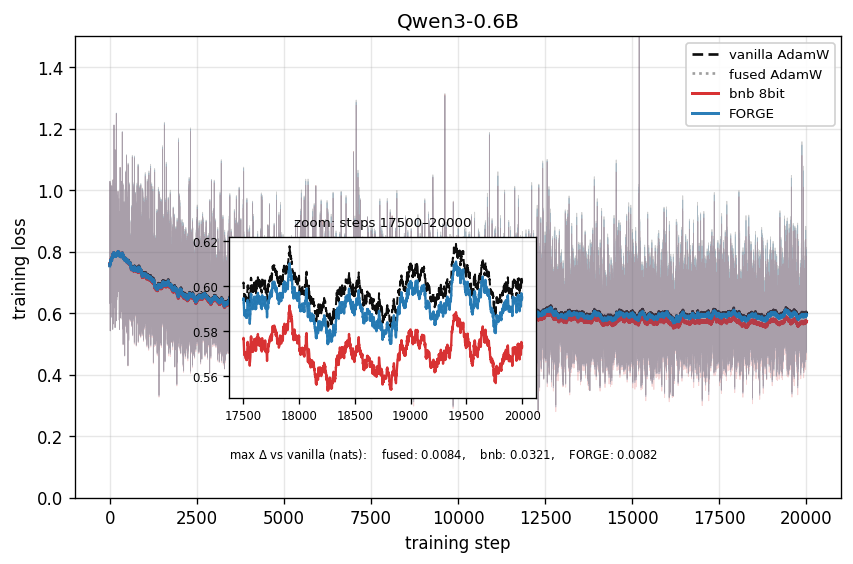}\hfill
\includegraphics[width=0.44\linewidth]{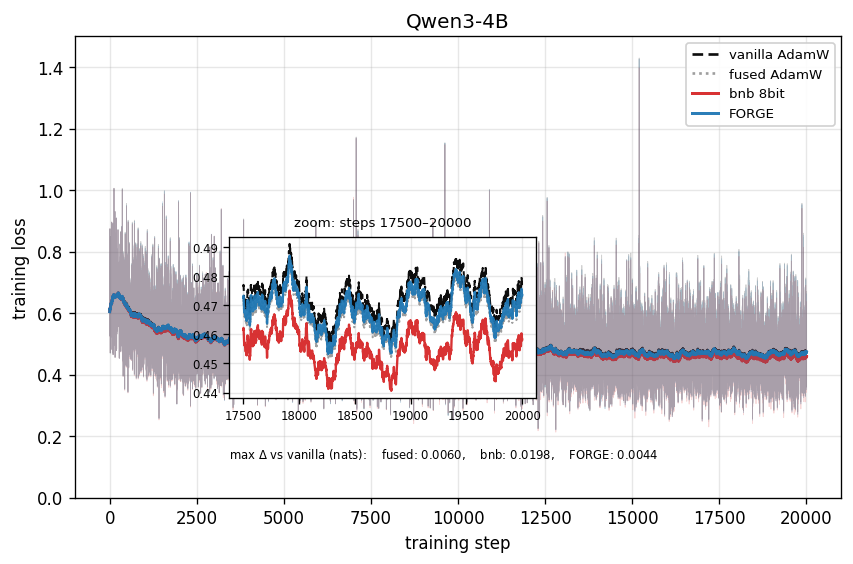}\\[3pt]
\includegraphics[width=0.44\linewidth]{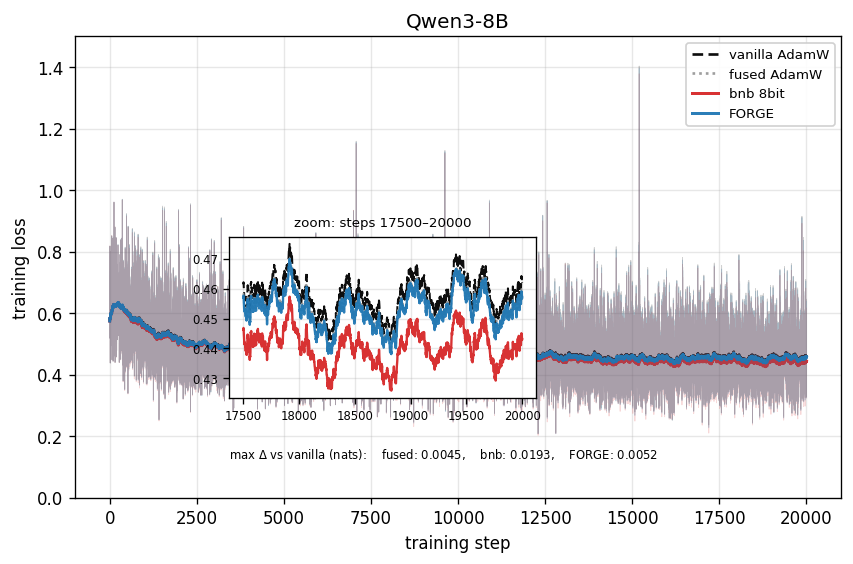}\hfill
\includegraphics[width=0.44\linewidth]{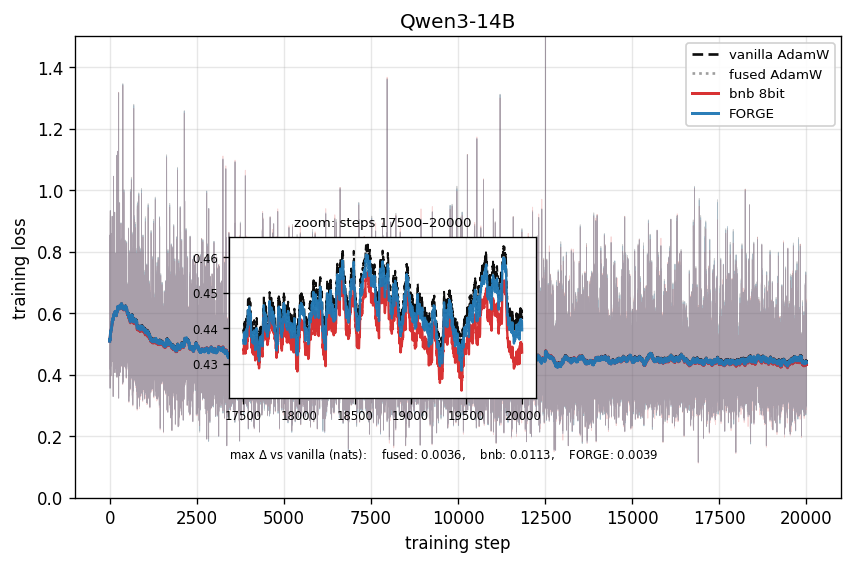}
\caption{Continued-pretraining parity across the remaining Qwen3 sizes (0.6B, 4B, 8B, 14B).}
\label{fig:cpt-family}
\end{figure}

\paragraph{From-scratch pretraining.} GPT-2 124M on FineWeb-Edu sample-10BT, 125k iterations (Figure~\ref{fig:gpt2full}); held-out loss 3.20 (\forge{}) vs 3.22 (baseline).

\begin{center}\small
\setlength{\tabcolsep}{5pt}
\begin{tabular}{@{}l l@{}}
\toprule
architecture & 12 layers, 12 heads, embedding 768, head dim 64, context 1024, \\
 & vocabulary 50{,}304 (GPT-2 BPE, padded), no biases, no dropout \\
weight tying & $\texttt{lm\_head}\leftrightarrow\texttt{wte}$; the tied head is stepped by the standard \\
 & optimizer (Remark~\ref{rem:order}) \\
\forge{} coverage & the 48 block linears, 84.93M of 124{,}373{,}760 parameters (68.3\%) \\
optimizer & AdamW, $(\beta_1,\beta_2)=(0.9,0.95)$, $\epsilon=10^{-8}$, wd 0.1 on ndim $\ge2$ \\
schedule & peak LR $6{\times}10^{-4}$, min $6{\times}10^{-5}$, 2{,}000-iteration linear warmup, cosine \\
batch & $64\times$ sequence 1024 (65{,}536 tokens/iteration), no accumulation \\
other & global-norm clipping off on every arm, seed 1337, eval every 1{,}000 \\
\bottomrule
\end{tabular}
\end{center}

\begin{figure}[tbp]
\centering
\includegraphics[width=0.68\linewidth]{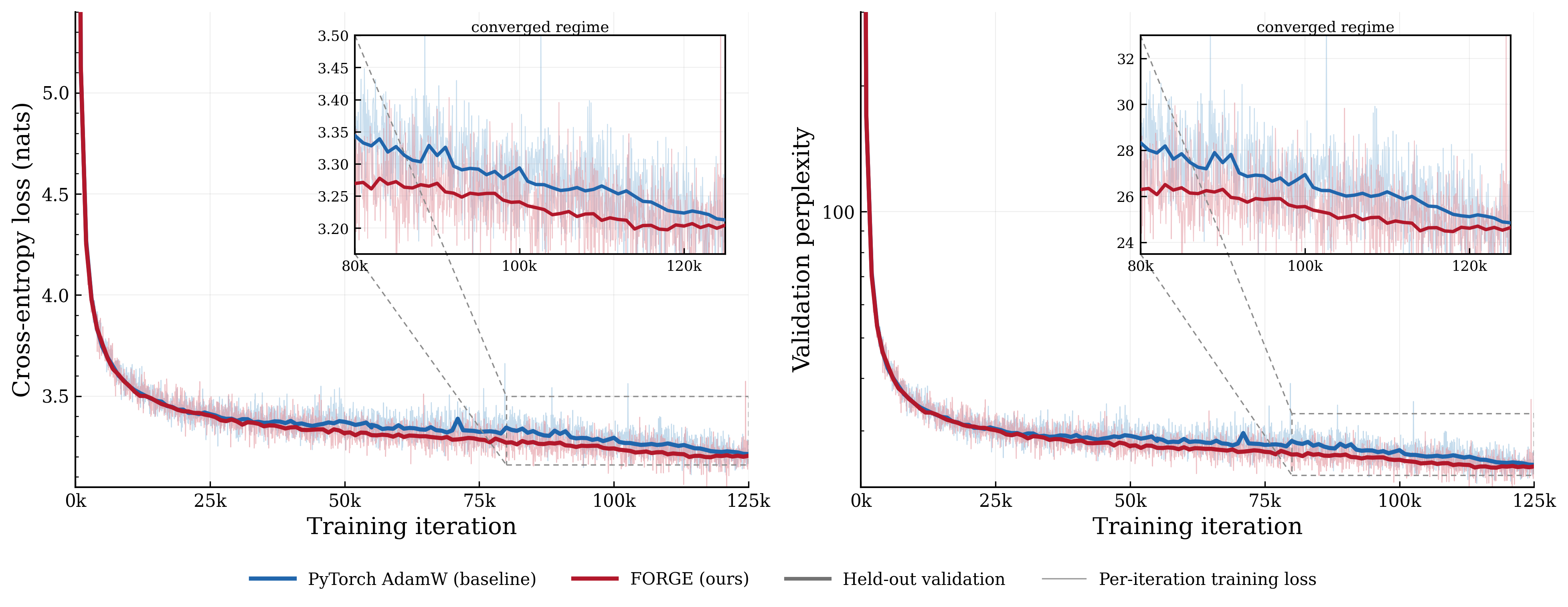}
\caption{GPT-2 124M from random initialization: held-out cross-entropy (left) and validation perplexity (right) with converged-regime insets.}
\label{fig:gpt2full}
\end{figure}

\section{Extended Single-GPU Grids}\label{app:grids1}

Peak memory is shape-deterministic; latencies are $N{=}20$ medians under the protocol of Appendix~\ref{app:protocol1}. $^{\P}$ marks arms measured together in one session, between which latency ratios may be drawn.

\subsection{Batch $\times$ sequence, per baseline (Llama-3.1-8B, H200; peak GB)}

\begin{center}\small
\begin{tabular}{@{}l l rrr@{}}
\toprule
SEQ & method & BS=1 & BS=2 & BS=4 \\
\midrule
1024 & fused AdamW & 60.27 & 60.65 & 72.36 \\
 & bnb 8-bit & 45.6 & 45.9 & 46.3 \\
 & \forge{} (bf16 states) & 51.80 & 58.65 & 72.36 \\
 & \forge{} (int8 states) & 38.79 & 45.57 & 59.31 \\
\addlinespace[2pt]
2048 & fused AdamW & 60.65 & 72.36 & 99.79 \\
 & bnb 8-bit & 45.9 & 57.6 & 85.1 \\
 & \forge{} (bf16 states) & 58.65 & 72.36 & 99.79 \\
 & \forge{} (int8 states) & 45.57 & 59.31 & 86.78 \\
\addlinespace[2pt]
4096 & fused AdamW & 72.37 & 99.79 & OOM \\
 & bnb 8-bit & 57.6 & 85.1 & OOM \\
 & \forge{} (bf16 states) & 72.37 & 99.79 & OOM \\
 & \forge{} (int8 states) & 59.31 & 86.78 & OOM \\
\bottomrule
\end{tabular}
\end{center}

\noindent At BS=16, SEQ=512 ($BT{=}8192$) the bf16 arm is at parity (both 99.79~GB) while int8 holds 86.78: past the crossover only the persistent state saving survives. BS=4, SEQ=4096 is recovered by \forge{}$+$AC.

\subsection{Multi-model grid (BS=1, SEQ=512, H200; peak GB)}

\begin{center}\small
\setlength{\tabcolsep}{4.5pt}
\begin{tabular}{@{}l rrrrrr@{}}
\toprule
method & Q-0.6B & Q-1.7B & Q-4B & Q-8B & L-8B & Q-14B \\
\midrule
vanilla AdamW & 5.34 & 14.24 & 31.85 & 63.63 & 62.04 & 113.22 \\
fused AdamW & 5.33 & 14.24 & 31.85 & 61.31 & 60.08 & 110.33 \\
optimi.grad-release & 6.3 & 15.0 & 33.5 & 66.4 & 64.5 & 117.3 \\
bnb 8-bit & 4.2 & 11.1 & 24.4 & 46.3 & 45.4 & 83.7 \\
GaLore r=128 & 4.0 & 9.5 & 19.1 & 39.0 & 37.2 & 65.9 \\
APOLLO r=256 & 4.2 & 9.9 & 20.1 & 40.1 & 38.4 & 67.6 \\
AdaLomo & 3.1 & 7.0 & 12.1 & 25.2 & 23.6 & 40.7 \\
FlashOptim & 4.7 & 10.4 & 22.6 & 43.5 & 42.3 & 76.6 \\
\forge{} (bf16 states) & 5.77 & 12.98 & 27.39 & 49.64 & 48.36 & 87.80 \\
\forge{} (int8 states) & 4.66 & 9.98 & 20.31 & 36.53 & 35.32 & 63.78 \\
\bottomrule
\end{tabular}
\end{center}

\noindent Peak is shape-deterministic, so these columns carry unchanged to the H100 and B200 sweeps below; only latencies are per-platform.

\subsection{Capability walk on one H200 (batch 1, no activation checkpointing)}

\begin{center}\small
\begin{tabular}{@{}l ccc c r@{}}
\toprule
& \multicolumn{3}{c}{14B trains at seq} & 32B & peak @ 14B s512 \\
method & 512 & 2048 & 8192 & 512 & (GB) \\
\midrule
standard (torch fused) & \checkmark & \checkmark & --- & --- & 110.33 \\
bitsandbytes 8-bit & \checkmark & \checkmark & --- & --- & 83.6 \\
\forge{} bf16 & \checkmark & \checkmark & --- & --- & 87.80 \\
\forge{} int8 & \checkmark & \checkmark & --- & --- & 63.78 \\
\forge{} fp8 & \checkmark & \checkmark & --- & \checkmark & 62.1 \\
\bottomrule
\end{tabular}
\end{center}

\noindent The 32B cell runs at 134.4 of 139.8~GB, 1165~ms/step, and is out of memory at sequence 2048. The standard path's 14B peak is sequence-invariant --- a states-plus-gradients watermark --- while \forge{}'s tracks activations.

\subsection{Optimizer-phase and peak at 8B (batch 1, sequence 4096, H200)}

Columns restricted to the quantities cited in the main text:

\begin{center}\small
\begin{tabular}{@{}l rr@{}}
\toprule
method & optimizer phase (ms) & peak (GB) \\
\midrule
standard fused & 46.4 & 72.4 \\
standard fused $+$ CUDA-graph step & 50.8 & 91.4 \\
standard foreach & 111.0 & 72.4 \\
bnb 8-bit & 245.1 & 61.7 \\
\forge{} bf16 & 4.5 & 72.4 \\
\forge{} int8 & 4.5 & 63.4 \\
\forge{} fp8 & 4.4 & 62.5 \\
\bottomrule
\end{tabular}
\end{center}

\noindent The peak sits at the logits moment here, so bf16 states match the baseline while int8/fp8 cut 12\%; the graph-captured step is no faster and holds $+15$~GB of capture pool.

\subsection{Full baseline sweep across sequence length (Llama-3.1-8B, H200, BS=1)}\label{app:fullh200}

Each cell: per-step latency (ms) / peak memory (GB), measured BF16 everywhere with FlashAttention-3 and Liger kernels. $^\dagger$approximated AdamW dynamics (not numerically identical to vanilla AdamW). $^\ddagger$COAT rows use the COAT paper's fp32-master $+$ autocast(bf16) $+$ SDPA recipe, not BF16-everywhere $+$ FA3; within those rows the correct reference is COAT-anchor ($\sim$120~GB), not the BF16-everywhere fused-AdamW row --- the two regimes are never compared against each other.

\begin{center}\small
\setlength{\tabcolsep}{3.5pt}
\begin{tabular}{@{}l rr rr rr rr@{}}
\toprule
 & \multicolumn{2}{c}{$S{=}512$} & \multicolumn{2}{c}{$S{=}1024$} & \multicolumn{2}{c}{$S{=}2048$} & \multicolumn{2}{c}{$S{=}4096$} \\
\cmidrule(lr){2-3}\cmidrule(lr){4-5}\cmidrule(lr){6-7}\cmidrule(lr){8-9}
method & ms & GB & ms & GB & ms & GB & ms & GB \\
\midrule
vanilla AdamW$^{\P}$ & 167.1 & 62.04 & 178.0 & 62.22 & 255.3 & 62.59 & 431.6 & 72.37 \\
fused AdamW$^{\P}$ & 134.3 & 60.08 & 149.9 & 60.27 & 225.0 & 60.65 & 402.0 & 72.37 \\
\forge{} (bf16 states)$^{\P}$ & \pnum{110.2} & 48.36 & \pnum{140.8} & 51.80 & \pnum{221.4} & 58.65 & \pnum{405.2} & 72.37 \\
\forge{} (int8 states)$^{\P}$ & \pnum{155.0} & 35.32 & \pnum{192.2} & 38.79 & \pnum{306.6} & 45.57 & \pnum{516.7} & 59.31 \\
optimi.grad-release & 188.7 & 64.5 & 220.2 & 127.1 & \multicolumn{2}{c}{OOM} & \multicolumn{2}{c}{OOM} \\
bnb 8-bit & 316.3 & 45.4 & 339.0 & 45.6 & 408.4 & 45.9 & 574.3 & 57.6 \\
GaLore r$=$128$^\dagger$ & 149.2 & 37.2 & 160.8 & 37.4 & 235.5 & 37.8 & 411.9 & 47.5 \\
APOLLO r$=$256$^\dagger$ & 184.3 & 38.4 & 207.7 & 38.6 & 277.5 & 38.9 & 435.3 & 48.6 \\
AdaLomo$^\dagger$ & 632.9 & 23.6 & 643.8 & 26.3 & 681.6 & 31.7 & 802.6 & 42.5 \\
FlashOptim & 148.0 & 42.3 & 172.9 & 60.3 & 240.7 & 82.1 & 413.4 & 110.7 \\
\midrule
COAT-anchor$^\ddagger$ & 173.2 & 120.1 & 208.3 & 120.4 & 301.8 & 121.6 & \multicolumn{2}{c}{OOM} \\
COAT-opt$^\ddagger$ & 445.8 & 76.7 & 483.2 & 77.0 & 577.8 & 78.2 & 781.9 & 96.1 \\
COAT-act$^\ddagger$ & 364.1 & 120.1 & 367.9 & 120.4 & 366.5 & 121.1 & 448.4 & 122.6 \\
COAT-both$^\ddagger$ & 636.5 & 76.7 & 639.4 & 77.0 & 639.3 & 77.7 & 721.9 & 79.1 \\
\bottomrule
\end{tabular}
\end{center}

\noindent $^{\P}$Measured together on an idle H200 under the protocol of Appendix~\ref{app:protocol1}; the $S{=}512$ column is the main paper's headline cell, whose dispersion is reported there. Two results the table makes directly: at $S{=}4096$ vanilla AdamW, fused AdamW and \forge{} (bf16) all meet at $72.37$~GB, which is the $BT$ crossover; and optimi's peak \emph{rises} from 64.5 to 127.1~GB between $S{=}512$ and $1024$, its per-parameter accumulator pool exceeding the monolithic allocation it replaces.

\subsection{Qwen3 family $\times$ sequence, H200}\label{app:h200qwen}

Each Qwen3 size at batch~1 across the same four sequence lengths and protocol. Both PyTorch references fit Qwen3-14B on one H200 at every sequence length. $^{\S}$bf16-state rows give the peak only; their step times are the main paper's.

{\footnotesize
\setlength{\tabcolsep}{4pt}
\begin{longtable}{@{}l rrrr@{}}
\toprule
method & $S{=}512$ & $S{=}1024$ & $S{=}2048$ & $S{=}4096$ \\
 & ms / GB & ms / GB & ms / GB & ms / GB \\
\midrule
\endfirsthead
\toprule
method & $S{=}512$ & $S{=}1024$ & $S{=}2048$ & $S{=}4096$ \\
 & ms / GB & ms / GB & ms / GB & ms / GB \\
\midrule
\endhead
\multicolumn{5}{@{}l}{\emph{Qwen3-0.6B}} \\
vanilla AdamW & 111.9 / 5.34 & 111.4 / 6.80 & 113.7 / 10.21 & 118.1 / 17.01 \\
fused AdamW & 107.3 / 5.33 & 106.4 / 6.80 & 109.5 / 10.21 & 113.6 / 17.01 \\
\forge{} (int8 states) & \pnum{100.2} / 4.66 & \pnum{100.6} / 6.38 & \pnum{104.1} / 9.83 & \pnum{108.0} / 16.60 \\
\forge{} (bf16 states)$^{\S}$ & --- / 5.77 & --- / 7.38 & --- / 10.79 & --- / 17.59 \\
optimi.grad-release & 177.4 / 6.3 & 174.0 / 12.4 & 176.6 / 20.3 & 181.4 / 31.5 \\
bnb 8-bit & 134.6 / 4.2 & 135.7 / 5.7 & 138.8 / 9.1 & 142.8 / 15.9 \\
GaLore r$=$128$^\dagger$ & 136.6 / 4.0 & 137.6 / 5.4 & 140.2 / 8.8 & 144.3 / 15.6 \\
APOLLO$^\dagger$ & 159.1 / 4.2 & 160.6 / 5.7 & 163.2 / 9.1 & 167.4 / 15.9 \\
AdaLomo$^\dagger$ & 527.1 / 3.1 & 517.5 / 4.6 & 525.0 / 8.0 & 534.1 / 14.8 \\
FlashOptim & 135.2 / 4.7 & 136.1 / 7.5 & 138.8 / 12.0 & 142.9 / 19.9 \\
\midrule
\multicolumn{5}{@{}l}{\emph{Qwen3-1.7B}} \\
vanilla AdamW & 127.3 / 14.24 & 127.2 / 14.44 & 129.9 / 18.25 & 151.7 / 26.81 \\
fused AdamW & 114.5 / 14.24 & 114.3 / 14.44 & 117.1 / 18.25 & 139.4 / 26.81 \\
\forge{} (int8 states) & \pnum{101.7} / 9.98 & \pnum{102.7} / 12.12 & \pnum{106.4} / 16.40 & \pnum{112.6} / 24.97 \\
\forge{} (bf16 states)$^{\S}$ & --- / 12.98 & --- / 15.12 & --- / 19.41 & --- / 27.97 \\
optimi.grad-release & 177.2 / 15.0 & 177.8 / 30.0 & 182.1 / 47.1 & 186.2 / 68.5 \\
bnb 8-bit & 164.8 / 11.1 & 166.7 / 11.3 & 169.7 / 15.1 & 191.5 / 23.7 \\
GaLore r$=$128$^\dagger$ & 137.4 / 9.5 & 138.7 / 9.7 & 141.8 / 13.5 & 156.7 / 22.0 \\
APOLLO$^\dagger$ & 160.3 / 9.9 & 162.3 / 10.1 & 164.8 / 13.9 & 184.7 / 22.5 \\
AdaLomo$^\dagger$ & 527.5 / 7.0 & 520.3 / 7.6 & 531.1 / 11.8 & 537.3 / 20.4 \\
FlashOptim & 138.1 / 10.4 & 139.2 / 15.8 & 141.9 / 23.3 & 146.2 / 35.0 \\
\midrule
\multicolumn{5}{@{}l}{\emph{Qwen3-4B}} \\
vanilla AdamW & 187.6 / 31.85 & 188.6 / 32.10 & 199.4 / 35.61 & 298.0 / 48.51 \\
fused AdamW & 158.4 / 31.85 & 159.4 / 32.10 & 171.0 / 35.61 & 271.9 / 48.51 \\
\forge{} (int8 states) & \pnum{129.8} / 20.31 & \pnum{132.6} / 23.71 & \pnum{135.4} / 30.00 & \pnum{200.0} / 42.94 \\
\forge{} (bf16 states)$^{\S}$ & --- / 27.39 & --- / 30.72 & --- / 37.05 & --- / 49.96 \\
optimi.grad-release & 225.6 / 33.5 & 230.9 / 66.9 & 232.8 / 103.4 & OOM \\
bnb 8-bit & 260.6 / 24.4 & 262.7 / 24.7 & 275.3 / 28.2 & 370.7 / 41.1 \\
GaLore r$=$128$^\dagger$ & 176.9 / 19.1 & 180.0 / 19.4 & 182.8 / 22.8 & 280.3 / 35.8 \\
APOLLO$^\dagger$ & 212.4 / 20.1 & 215.1 / 20.4 & 224.6 / 23.8 & 317.1 / 36.8 \\
AdaLomo$^\dagger$ & 809.4 / 12.1 & 801.1 / 14.2 & 813.3 / 20.5 & 841.4 / 33.5 \\
FlashOptim & 179.0 / 22.6 & 181.4 / 33.5 & 184.5 / 47.3 & 276.1 / 67.8 \\
\midrule
\multicolumn{5}{@{}l}{\emph{Qwen3-8B}} \\
vanilla AdamW & 243.5 / 63.63 & 246.1 / 63.85 & 303.8 / 64.28 & 479.0 / 76.27 \\
fused AdamW & 184.8 / 61.31 & 187.1 / 61.53 & 249.3 / 61.97 & 429.0 / 76.27 \\
\forge{} (int8 states) & \pnum{131.6} / 36.53 & \pnum{135.1} / 40.27 & \pnum{155.0} / 47.91 & \pnum{288.7} / 63.07 \\
\forge{} (bf16 states)$^{\S}$ & --- / 49.64 & --- / 53.45 & --- / 61.06 & --- / 76.27 \\
optimi.grad-release & 230.0 / 66.4 & 248.5 / 130.3 & OOM & OOM \\
bnb 8-bit & 370.8 / 46.3 & 375.3 / 46.5 & 433.4 / 47.0 & 605.8 / 61.3 \\
GaLore r$=$128$^\dagger$ & 196.1 / 39.0 & 199.7 / 39.2 & 251.6 / 39.7 & 429.2 / 51.5 \\
APOLLO$^\dagger$ & 240.0 / 40.1 & 243.7 / 40.3 & 297.7 / 40.7 & 460.3 / 52.6 \\
AdaLomo$^\dagger$ & 909.7 / 25.2 & 919.4 / 28.2 & 933.3 / 34.0 & 1032.1 / 45.8 \\
FlashOptim & 178.8 / 43.5 & 196.9 / 62.1 & 259.1 / 84.9 & 438.5 / 115.4 \\
\midrule
\multicolumn{5}{@{}l}{\emph{Qwen3-14B}} \\
vanilla AdamW & 265.3 / 113.22 & 317.5$^{\P}$ / 113.44 & 474.2$^{\P}$ / 113.89 & 791.5$^{\P}$ / 123.96 \\
fused AdamW & 211.9 / 110.33 & 266.0$^{\P}$ / 110.55 & 417.0$^{\P}$ / 111.01 & 738.6$^{\P}$ / 123.96 \\
\forge{} (int8 states) & \pnum{150.1} / 63.78 & \pnum{152.1} / 68.74 & \pnum{248.7} / 79.10 & \pnum{480.7} / 99.59 \\
\forge{} (bf16 states)$^{\S}$ & --- / 87.80 & --- / 92.93 & --- / 103.28 & --- / 123.96 \\
optimi.grad-release & 343.6 / 117.3 & OOM & OOM & OOM \\
bnb 8-bit & 564.1 / 83.7 & 600.8 / 83.7 & 744.7 / 84.2 & 1050.3 / 97.1 \\
GaLore r$=$128$^\dagger$ & 247.2 / 65.9 & 271.7 / 66.2 & 419.2 / 66.5 & 737.0 / 76.5 \\
APOLLO$^\dagger$ & 315.0 / 67.6 & 346.3 / 67.8 & 479.5 / 68.3 & 782.3 / 78.2 \\
AdaLomo$^\dagger$ & 1216.6 / 40.7 & 1230.7 / 44.9 & 1307.6 / 53.5 & 1540.8 / 70.7 \\
FlashOptim & 257.8 / 76.6 & 300.3 / 108.6 & OOM & OOM \\
\bottomrule
\end{longtable}}

\subsection{Qwen3 family $\times$ sequence, H100 SXM 80~GB}\label{app:h100qwen}

The 80~GB budget exposes boundaries the H200 sweep does not. Among arms keeping the exact AdamW update, \forge{} (int8) is the only one that trains Qwen3-14B at BS=1, $S\le2048$ here; the approximate-dynamics arms ($^\dagger$) also fit, bitsandbytes does not. Latencies run uniformly higher than H200, consistent with the $\sim$30\% HBM3-vs-HBM3e gap. Memory columns are the platform-independent peaks; only latencies are this card's. ``n/m'': not timed on this platform. $^{\S}$peak only.

\begin{figure}[tbp]
\centering
\includegraphics[width=\linewidth]{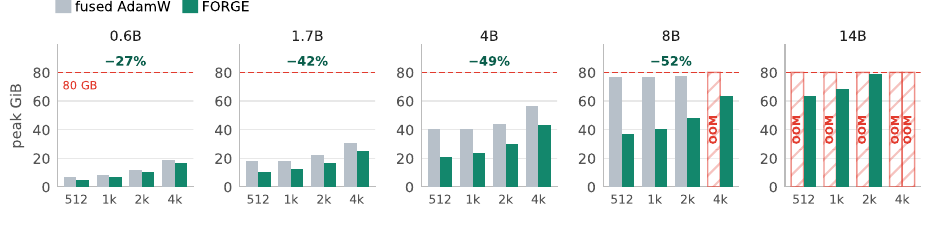}
\caption{Qwen3 family $\times$ sequence on H100 (80~GB), batch~1: peak GB, fused AdamW vs.\ \forge{}; hatched red = out of memory; the label over each panel is the saving at $S{=}512$. Vanilla AdamW is out of memory at every 8B and 14B cell. The full grid is tabulated below.}
\label{fig:supph100}
\end{figure}

{\footnotesize
\setlength{\tabcolsep}{4pt}
\begin{longtable}{@{}l rrrr@{}}
\toprule
method & $S{=}512$ & $S{=}1024$ & $S{=}2048$ & $S{=}4096$ \\
 & ms / GB & ms / GB & ms / GB & ms / GB \\
\midrule
\endfirsthead
\toprule
method & $S{=}512$ & $S{=}1024$ & $S{=}2048$ & $S{=}4096$ \\
 & ms / GB & ms / GB & ms / GB & ms / GB \\
\midrule
\endhead
\multicolumn{5}{@{}l}{\emph{Qwen3-0.6B}} \\
vanilla AdamW & 125.3 / 5.34 & 124.8 / 6.80 & 127.3 / 10.21 & 132.3 / 17.01 \\
fused AdamW & 118.0 / 5.33 & 117.0 / 6.80 & 120.5 / 10.21 & 125.0 / 17.01 \\
\forge{} (int8 states) & \pnum{107.2} / 4.66 & \pnum{107.6} / 6.38 & \pnum{111.4} / 9.83 & \pnum{115.6} / 16.60 \\
\forge{} (bf16 states)$^{\S}$ & --- / 5.77 & --- / 7.38 & --- / 10.79 & --- / 17.59 \\
optimi.grad-release & 195.1 / 6.3 & 191.4 / 12.4 & 194.3 / 20.3 & 199.5 / 31.5 \\
bnb 8-bit & 168.3 / 4.2 & 169.6 / 5.7 & 173.5 / 9.1 & 178.5 / 15.9 \\
GaLore r$=$128$^\dagger$ & 153.0 / 4.0 & 154.1 / 5.4 & 157.0 / 8.8 & 161.6 / 15.6 \\
APOLLO$^\dagger$ & 179.8 / 4.2 & 181.5 / 5.7 & 184.4 / 9.1 & 189.2 / 15.9 \\
AdaLomo$^\dagger$ & 579.8 / 3.1 & 569.3 / 4.6 & 577.5 / 8.0 & 587.5 / 14.8 \\
FlashOptim & 162.2 / 4.7 & 163.3 / 7.5 & 166.6 / 12.0 & 171.5 / 19.9 \\
\midrule
\multicolumn{5}{@{}l}{\emph{Qwen3-1.7B}} \\
vanilla AdamW & 142.6 / 14.24 & 142.5 / 14.44 & 145.5 / 18.25 & 169.9 / 26.81 \\
fused AdamW & 126.0 / 14.24 & 125.7 / 14.44 & 128.8 / 18.25 & 153.3 / 26.81 \\
\forge{} (int8 states) & \pnum{108.8} / 9.98 & \pnum{109.9} / 12.12 & \pnum{113.8} / 16.40 & \pnum{120.5} / 24.97 \\
\forge{} (bf16 states)$^{\S}$ & --- / 12.98 & --- / 15.12 & --- / 19.41 & --- / 27.97 \\
optimi.grad-release & 194.9 / 15.0 & 195.6 / 30.0 & 200.3 / 47.1 & 204.8 / 68.5 \\
bnb 8-bit & 206.0 / 11.1 & 208.4 / 11.3 & 212.1 / 15.1 & 239.4 / 23.7 \\
GaLore r$=$128$^\dagger$ & 153.9 / 9.5 & 155.3 / 9.7 & 158.8 / 13.5 & 175.5 / 22.0 \\
APOLLO$^\dagger$ & 181.1 / 9.9 & 183.4 / 10.1 & 186.2 / 13.9 & 208.7 / 22.5 \\
AdaLomo$^\dagger$ & 580.3 / 7.0 & 572.3 / 7.6 & 584.2 / 11.8 & 591.0 / 20.4 \\
FlashOptim & 165.7 / 10.4 & 167.0 / 15.8 & 170.3 / 23.3 & 175.4 / 35.0 \\
\midrule
\multicolumn{5}{@{}l}{\emph{Qwen3-4B}} \\
vanilla AdamW & 210.1 / 31.85 & 211.2 / 32.10 & 223.3 / 35.61 & 333.8 / 48.51 \\
fused AdamW & 174.2 / 31.85 & 175.3 / 32.10 & 188.1 / 35.61 & 299.1 / 48.51 \\
\forge{} (int8 states) & \pnum{138.9} / 20.31 & \pnum{141.9} / 23.71 & \pnum{144.9} / 30.00 & \pnum{214.0} / 42.94 \\
\forge{} (bf16 states)$^{\S}$ & --- / 27.39 & --- / 30.72 & --- / 37.05 & --- / 49.96 \\
optimi.grad-release & 248.2 / 33.5 & 254.0 / 66.9 & OOM & OOM \\
bnb 8-bit & 325.8 / 24.4 & 328.4 / 24.7 & 344.1 / 28.2 & 463.4 / 41.1 \\
GaLore r$=$128$^\dagger$ & 198.1 / 19.1 & 201.6 / 19.4 & 204.7 / 22.8 & 313.9 / 35.8 \\
APOLLO$^\dagger$ & 240.0 / 20.1 & 243.1 / 20.4 & 253.8 / 23.8 & 358.3 / 36.8 \\
AdaLomo$^\dagger$ & 890.3 / 12.1 & 881.2 / 14.2 & 894.6 / 20.5 & 925.5 / 33.5 \\
FlashOptim & 214.8 / 22.6 & 217.7 / 33.5 & 221.4 / 47.3 & 331.3 / 67.8 \\
\midrule
\multicolumn{5}{@{}l}{\emph{Qwen3-8B}} \\
vanilla AdamW & n/m / 63.63 & n/m / 63.85 & n/m / 64.28 & n/m / 76.27 \\
fused AdamW & 203.3 / 61.31 & 205.8 / 61.53 & 274.2 / 61.97 & n/m / 76.27 \\
\forge{} (int8 states) & \pnum{140.8} / 36.53 & \pnum{144.6} / 40.27 & \pnum{165.9} / 47.91 & \pnum{308.9} / 63.07 \\
\forge{} (bf16 states)$^{\S}$ & --- / 49.64 & --- / 53.45 & --- / 61.06 & --- / 76.27 \\
optimi.grad-release & 253.0 / 66.4 & OOM & OOM & OOM \\
bnb 8-bit & 463.5 / 46.3 & 469.1 / 46.5 & 541.8 / 47.0 & 757.3 / 61.3 \\
GaLore r$=$128$^\dagger$ & 219.6 / 39.0 & 223.7 / 39.2 & 281.8 / 39.7 & 480.7 / 51.5 \\
APOLLO$^\dagger$ & 271.2 / 40.1 & 275.4 / 40.3 & 336.4 / 40.7 & 520.1 / 52.6 \\
AdaLomo$^\dagger$ & 1000.7 / 25.2 & 1011.3 / 28.2 & 1026.6 / 34.0 & 1135.3 / 45.8 \\
FlashOptim & 214.6 / 43.5 & 236.3 / 62.1 & OOM & OOM \\
\midrule
\multicolumn{5}{@{}l}{\emph{Qwen3-14B}} \\
vanilla AdamW & OOM / 113.22 & OOM / 113.44 & OOM / 113.89 & OOM / 123.96 \\
fused AdamW & OOM / 110.33 & OOM / 110.55 & OOM / 111.01 & OOM / 123.96 \\
\forge{} (int8 states) & \pnum{160.6} / 63.78 & \pnum{162.7} / 68.74 & \pnum{266.1} / 79.10 & OOM / 99.59 \\
\forge{} (bf16 states)$^{\S}$ & OOM / 87.80 & OOM / 92.93 & OOM / 103.28 & OOM / 123.96 \\
optimi.grad-release & OOM & OOM & OOM & OOM \\
bnb 8-bit & OOM & OOM & OOM & OOM \\
GaLore r$=$128$^\dagger$ & 276.9 / 65.9 & 304.3 / 66.2 & 469.5 / 66.5 & 825.4 / 76.5 \\
APOLLO$^\dagger$ & 356.0 / 67.6 & 391.3 / 67.8 & 541.8 / 68.3 & 884.0 / 78.2 \\
AdaLomo$^\dagger$ & 1338.3 / 40.7 & 1353.8 / 44.9 & 1438.4 / 53.5 & 1694.9 / 70.7 \\
FlashOptim & 309.4 / 76.6 & OOM & OOM & OOM \\
\bottomrule
\end{longtable}}

\subsection{B200 (Blackwell) sweeps}\label{app:b200}

\paragraph{Provenance.} $^\star$mean of 5 rather than the $N{=}20$ median, so latencies carry wider per-step variance and ratios are drawn only within the marked rows. Two cells falling outside their within-model trend (Qwen3-1.7B at $S{=}2048$, Qwen3-14B at $S{=}512$) are withheld rather than reported, since neither is separable from sample variance; their memory is unaffected. Memory columns are the platform-independent peaks, latencies B200's own; ``n/m'' as on H100, $^{\S}$peak only. Stack: FlashAttention-4 (FA3 ships sm\_90 cubins only, and the attention backend does not touch the weight-gradient path), CUDA~12.8, Triton~3.6.

\paragraph{Llama-3.1-8B, BS=1, across sequence.}

\begin{center}\small
\setlength{\tabcolsep}{3.5pt}
\begin{tabular}{@{}l rr rr rr rr@{}}
\toprule
 & \multicolumn{2}{c}{$S{=}512$} & \multicolumn{2}{c}{$S{=}1024$} & \multicolumn{2}{c}{$S{=}2048$} & \multicolumn{2}{c}{$S{=}4096$} \\
\cmidrule(lr){2-3}\cmidrule(lr){4-5}\cmidrule(lr){6-7}\cmidrule(lr){8-9}
method & ms & GB & ms & GB & ms & GB & ms & GB \\
\midrule
vanilla AdamW$^\star$ & 240.1 & 62.04 & 242.6 & 62.22 & 251.3 & 62.59 & 324.7 & 72.37 \\
fused AdamW$^\star$ & 205.2 & 60.08 & 212.9 & 60.27 & 216.9 & 60.65 & 285.0 & 72.37 \\
\forge{} (int8 states)$^\star$ & \pnum{172.6} & 35.32 & \pnum{178.1} & 38.79 & \pnum{178.7} & 45.57 & \pnum{202.0} & 59.31 \\
\forge{} (bf16 states)$^{\S}$ & --- & 48.36 & --- & 51.80 & --- & 58.65 & --- & 72.37 \\
optimi.grad-release & 240.6 & 64.5 & 239.1 & 127.0 & \multicolumn{2}{c}{OOM} & \multicolumn{2}{c}{OOM} \\
bnb 8-bit & 364.6 & 45.3 & 368.6 & 45.5 & 368.4 & 45.9 & 406.9 & 57.6 \\
GaLore r$=$128$^\dagger$ & 242.1 & 37.2 & 249.6 & 37.4 & 244.4 & 37.7 & 271.1 & 47.4 \\
APOLLO$^\dagger$ & 279.5 & 38.3 & 280.2 & 38.5 & 282.0 & 38.9 & 313.5 & 48.5 \\
AdaLomo$^\dagger$ & 708.0 & 23.5 & 715.2 & 26.2 & 715.5 & 31.6 & 763.3 & 42.4 \\
FlashOptim & 218.7 & 42.2 & 218.1 & 60.3 & 223.1 & 82.0 & 263.2 & 110.7 \\
\bottomrule
\end{tabular}
\end{center}

\paragraph{Qwen3 family $\times$ sequence.} The Protocol-A vanilla and fused arms use the fp32-master reference recipe, which is why they are out of memory at Qwen3-14B within the 180~GB budget.

\begin{figure}[tbp]
\centering
\includegraphics[width=\linewidth]{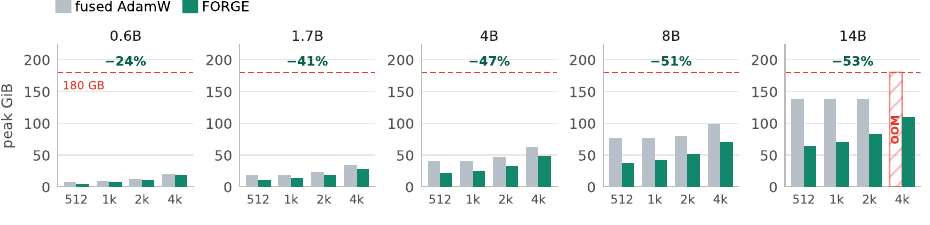}
\caption{Qwen3 family $\times$ sequence on B200 (180~GB), batch~1, Protocol-A ($^\star$) cells: peak GB, fused AdamW vs.\ \forge{}; hatched red = out of memory; panel labels give the saving at $S{=}512$. Vanilla AdamW is out of memory at every 14B cell; peak memory is protocol-independent. The full grid is tabulated below.}
\label{fig:suppb200}
\end{figure}

{\footnotesize
\setlength{\tabcolsep}{4pt}
\begin{longtable}{@{}l rrrr@{}}
\toprule
method & $S{=}512$ & $S{=}1024$ & $S{=}2048$ & $S{=}4096$ \\
 & ms / GB & ms / GB & ms / GB & ms / GB \\
\midrule
\endfirsthead
\toprule
method & $S{=}512$ & $S{=}1024$ & $S{=}2048$ & $S{=}4096$ \\
 & ms / GB & ms / GB & ms / GB & ms / GB \\
\midrule
\endhead
\multicolumn{5}{@{}l}{\emph{Qwen3-0.6B}} \\
vanilla AdamW$^\star$ & 188.4 / 5.34 & 198.9 / 6.80 & 199.2 / 10.21 & 186.4 / 17.01 \\
fused AdamW$^\star$ & 186.9 / 5.33 & 197.3 / 6.80 & 198.1 / 10.21 & 192.6 / 17.01 \\
\forge{} (int8 states)$^\star$ & \pnum{181.5} / 4.66 & \pnum{189.6} / 6.38 & \pnum{191.3} / 9.83 & \pnum{171.4} / 16.60 \\
\forge{} (bf16 states)$^{\S}$ & --- / 5.77 & --- / 7.38 & --- / 10.79 & --- / 17.59 \\
optimi.grad-release & 248.6 / 6.2 & 246.3 / 12.4 & 248.9 / 20.2 & 258.3 / 31.5 \\
bnb 8-bit & 232.7 / 4.2 & 235.3 / 5.7 & 232.3 / 9.1 & 233.2 / 15.9 \\
GaLore r$=$128$^\dagger$ & 229.7 / 3.9 & 229.1 / 5.4 & 231.3 / 8.8 & 234.7 / 15.6 \\
APOLLO$^\dagger$ & 249.4 / 4.2 & 251.6 / 5.7 & 251.0 / 9.1 & 254.2 / 15.9 \\
AdaLomo$^\dagger$ & 584.9 / 3.1 & 582.2 / 4.5 & 584.1 / 7.9 & 583.5 / 14.8 \\
FlashOptim & 221.5 / 4.6 & 217.0 / 7.4 & 219.1 / 12.0 & 222.5 / 19.9 \\
\midrule
\multicolumn{5}{@{}l}{\emph{Qwen3-1.7B}} \\
vanilla AdamW$^\star$ & 150.5 / 14.24 & 123.1 / 14.44 & 123.8 / 18.25 & 127.7 / 26.81 \\
fused AdamW$^\star$ & 192.1 / 14.24 & 113.2 / 14.44 & 114.5 / 18.25 & 118.3 / 26.81 \\
\forge{} (int8 states)$^\star$ & \pnum{180.6} / 9.98 & \pnum{101.6} / 12.12 & n / 16.40 & \pnum{100.4} / 24.97 \\
\forge{} (bf16 states)$^{\S}$ & --- / 12.98 & --- / 15.12 & --- / 19.41 & --- / 27.97 \\
optimi.grad-release & 255.9 / 15.0 & 250.6 / 29.9 & 258.9 / 47.0 & 276.2 / 68.4 \\
bnb 8-bit & 242.4 / 11.0 & 243.3 / 11.2 & 245.0 / 15.0 & 248.9 / 23.6 \\
GaLore r$=$128$^\dagger$ & 267.3 / 9.4 & 255.1 / 9.6 & 261.9 / 13.4 & 266.1 / 22.0 \\
APOLLO$^\dagger$ & 246.0 / 9.8 & 253.6 / 10.1 & 252.4 / 13.8 & 255.1 / 22.4 \\
AdaLomo$^\dagger$ & 614.3 / 6.9 & 598.3 / 7.5 & 589.8 / 11.8 & 494.0 / 20.4 \\
FlashOptim & 220.6 / 10.4 & 219.6 / 15.7 & 221.4 / 23.2 & 225.6 / 35.0 \\
\midrule
\multicolumn{5}{@{}l}{\emph{Qwen3-4B}} \\
vanilla AdamW$^\star$ & 251.0 / 31.85 & 262.0 / 32.10 & 264.7 / 35.61 & 275.4 / 48.51 \\
fused AdamW$^\star$ & 237.4 / 31.85 & 250.1 / 32.10 & 248.9 / 35.61 & 258.1 / 48.51 \\
\forge{} (int8 states)$^\star$ & \pnum{222.1} / 20.31 & \pnum{229.3} / 23.71 & \pnum{231.0} / 30.00 & \pnum{234.4} / 42.94 \\
\forge{} (bf16 states)$^{\S}$ & --- / 27.39 & --- / 30.72 & --- / 37.05 & --- / 49.96 \\
optimi.grad-release & 293.3 / 33.5 & 304.1 / 66.9 & OOM & OOM \\
bnb 8-bit & 299.7 / 24.4 & 298.6 / 24.6 & 301.1 / 28.1 & 307.2 / 41.1 \\
GaLore r$=$128$^\dagger$ & 274.2 / 19.1 & 272.7 / 19.3 & 273.6 / 22.8 & 275.5 / 35.8 \\
APOLLO$^\dagger$ & 347.2 / 20.1 & 337.0 / 20.3 & 346.6 / 23.8 & 353.7 / 36.7 \\
AdaLomo$^\dagger$ & 905.6 / 12.1 & 914.9 / 14.2 & 904.3 / 20.5 & 915.0 / 33.4 \\
FlashOptim & 311.3 / 22.6 & 318.9 / 33.5 & 309.8 / 47.3 & 316.0 / 67.7 \\
\midrule
\multicolumn{5}{@{}l}{\emph{Qwen3-8B}} \\
vanilla AdamW$^\star$ & 283.9 / 63.63 & 294.1 / 63.85 & 294.5 / 64.28 & 374.1 / 76.27 \\
fused AdamW$^\star$ & 252.6 / 61.31 & 263.0 / 61.53 & 263.6 / 61.97 & 330.9 / 76.27 \\
\forge{} (int8 states)$^\star$ & \pnum{222.0} / 36.53 & \pnum{226.0} / 40.27 & \pnum{230.5} / 47.91 & \pnum{256.3} / 63.07 \\
\forge{} (bf16 states)$^{\S}$ & --- / 49.64 & --- / 53.45 & --- / 61.06 & --- / 76.27 \\
optimi.grad-release & OOM & OOM & OOM & OOM \\
bnb 8-bit & 397.2 / 46.2 & 400.6 / 46.5 & 402.3 / 46.9 & 429.5 / 61.2 \\
GaLore r$=$128$^\dagger$ & 269.2 / 38.9 & 272.7 / 39.1 & 273.8 / 39.6 & 291.5 / 51.5 \\
APOLLO$^\dagger$ & 332.8 / 40.0 & 352.0 / 40.2 & 353.8 / 40.7 & 366.9 / 52.6 \\
AdaLomo$^\dagger$ & 946.7 / 25.2 & 940.0 / 28.1 & 937.3 / 34.0 & 976.4 / 45.7 \\
FlashOptim & 268.7 / 43.4 & 267.4 / 62.1 & 269.9 / 84.9 & 290.0 / 115.3 \\
\midrule
\multicolumn{5}{@{}l}{\emph{Qwen3-14B}} \\
vanilla AdamW$^\star$ & n/m / 113.22 & n/m / 113.44 & n/m / 113.89 & n/m / 123.96 \\
fused AdamW$^\star$ & 226.3 / 110.33 & 312.0 / 110.55 & 336.0 / 111.01 & n/m / 123.96 \\
\forge{} (int8 states)$^\star$ & n/a / 63.78 & \pnum{246.1} / 68.74 & \pnum{252.9} / 79.10 & \pnum{362.6} / 99.59 \\
\forge{} (bf16 states)$^{\S}$ & --- / 87.80 & --- / 92.93 & --- / 103.28 & --- / 123.96 \\
optimi.grad-release & OOM & OOM & OOM & OOM \\
bnb 8-bit & 555.3 / 83.6 & 559.0 / 83.7 & 577.1 / 84.2 & 704.4 / 97.0 \\
GaLore r$=$128$^\dagger$ & 300.1 / 65.9 & 256.7 / 66.1 & 301.9 / 66.5 & 429.9 / 76.4 \\
FlashOptim & 344.4 / 76.6 & 340.9 / 108.5 & 353.1 / 146.4 & OOM \\
\bottomrule
\end{longtable}}

\section{Autotuning and Tile Selection}\label{app:autotune}

\paragraph{Search space.} Six knobs per output tile: tile height and width $T_R,T_C\in\{64,128,256\}$, reduction chunk $T_K\in\{32,64,128\}$, an L2 swizzle group size $\in\{8,16,32\}$ that co-schedules row-adjacent tiles so they share input slabs in cache, warps $\in\{4,8\}$, and pipeline stages $\in\{2,\dots,6\}$: 810 Cartesian combinations, pruned to a curated 34-point base set shipped with the kernel, which an architecture-specific pre-prune reduces further per device.

\paragraph{Why $128{\times}128$.} Pre-pruning identified two failure modes: $256$-wide tiles spill registers at 8 warps and lose occupancy across the 132 SMs (consistent with the shared-memory budget of two $32$~KiB operand slabs per pipeline stage against the per-CTA limit at three or more stages), while $64$-wide tiles produce so many blocks that launch and per-tile epilogue overhead dominates the matrix-multiply mainloop. $T_R{=}T_C{=}128$ sits at the empirical Pareto point on every device measured, and the per-shape autotune varies mainly the pipeline depth with the reduction extent.

\paragraph{Certification.} An exhaustive H200 pass timed all 800 resource-valid configurations on the reference cell (best: $128{\times}128{\times}64$, group 16, 8 warps, 4 stages) and refined the top 20 across four shapes and three $BT$; promoting two deep-pipeline variants that HBM3e's latency rewards into the shipped list recovered $+2.5$--$8.4\%$ at $BT{=}4096$, left $BT\in\{512,32768\}$ neutral within $\pm3\%$, and passed the parity gate unchanged (rms 5.825e-4).

\section{Optimizer Families End to End}\label{app:zoo}

What the state update $A$ reads fixes the regime; what the weight update $U$ reads fixes only whether the weight write stays inside the epilogue (Appendix~\ref{app:exact}). Cells are Llama-3.1-8B, batch~1, sequence~512, 10{,}000 steps after a 500-step linear warmup, RTX PRO 6000, peak GB and mean ms per step (Figure~\ref{fig:zoo1}). \emph{Both arms hold fp32 optimizer state throughout this appendix}, which is why the AdamW baseline reads 75.0~GB here against 62.04 for the BF16-everywhere recipe of Appendix~\ref{app:fullh200}: the two are different recipes on different hardware and are never compared across.

\begin{center}\footnotesize
\setlength{\tabcolsep}{4pt}
\begin{tabular}{@{}l l l c rr rr@{}}
\toprule
& & & & \multicolumn{2}{c}{peak GB} & \multicolumn{2}{c}{ms/step} \\
\cmidrule(lr){5-6}\cmidrule(lr){7-8}
family & $A$ reads & $U$ reads beyond its coordinate & reg. & base & \forge{} & base & \forge{} \\
\midrule
SGD & $G_{pq}$ & --- & 1 & 45.1 & \textbf{20.9} & 217.3 & \textbf{148.2} \\
SGD w/ mom. & $G_{pq}$ & --- & 1 & 60.0 & \textbf{34.9} & 290.4 & \textbf{166.1} \\
Adagrad & $G_{pq}$ & --- & 1 & 75.0 & \textbf{34.9} & 361.6 & \textbf{167.0} \\
RMSprop & $G_{pq}$ & --- & 1 & 75.0 & \textbf{34.9} & 361.2 & \textbf{164.6} \\
AdamW & $G_{pq}$ & --- & 1 & 75.0 & \textbf{48.8} & 379.0 & \textbf{236.0} \\
NAdam & $G_{pq}$ & --- & 1 & --- & 48.8 & --- & 177.0 \\
RAdam & $G_{pq}$ & --- & 1 & 75.0 & \textbf{48.8} & 443.7 & \textbf{181.6} \\
Lion & $G_{pq}$ & --- & 1 & 48.0 & \textbf{34.9} & 380.6 & \textbf{186.8} \\
\addlinespace[2pt]
LAMB & $G_{pq}$ & $\|\mathbf{W}\|$, $\|$update$\|$, per layer & 2 & --- & 48.8 & --- & 218.6 \\
Muon & $G_{pq}$ & matrix of $\mathbf{B}$ (Newton--Schulz) & 2 & \multicolumn{4}{c}{\S\ref{app:zoomuon}} \\
Scion & $G_{pq}$ & matrix of $S'$ & 2 & \multicolumn{4}{c}{classified only} \\
\addlinespace[2pt]
Adam-mini & block mean $\mathbf{G}^{\odot2}$ & its block's scalar & 3 & 65.9 & \textbf{34.9} & 663.1 & \textbf{266.2} \\
Adafactor & row/col sum $\mathbf{G}^{\odot2}$ & row and column of $S'$ & 3 & 41.9 & \textbf{22.2} & 951.4 & \textbf{280.9} \\
AdaLomo$^{\ddagger}$ & row/col sum $\mathbf{G}^{\odot2}$ & row and column of $S'$ & 3 & 39.9 & \textbf{20.9} & 615.7 & \textbf{189.3} \\
SM3 & row/col max $\mathbf{G}^{\odot2}$ & row and column of $S'$ & 3 & 39.9 & \textbf{20.9} & 583.8 & \textbf{168.4} \\
LARS & $\|\mathbf{G}\|$, per layer & the layer's scalars & 3 & \multicolumn{4}{c}{classified only} \\
\bottomrule
\end{tabular}
\end{center}

\noindent Peak falls $27$--$54\%$ and step time $32$--$71\%$ in every family with a baseline. The largest gains are the regime-3 families --- Adafactor $3.4\times$, SM3 $3.5\times$, AdaLomo $3.3\times$, Adam-mini $2.5\times$ --- so the rules that fuse least completely gain most, their reference implementations paying a per-tensor dispatch that the deferred path removes. Two int8-state arms were also run: AdamW at 35.8~GB / 182.6~ms, Lion at 28.3 / 167.0. NAdam and LAMB are reported single-arm, having no counterpart in \texttt{torch.optim} at matched state. $^{\ddagger}$AdaLomo runs 4{,}000 steps. Every update is verified to floating-point round-off against a reference PyTorch implementation on small tensors ($D{=}512$, $BT{=}256$) before any end-to-end run, and all thirteen train end to end.

\noindent LAMB classifies as regime~2: its norms are of $\mathbf{W}$ and of the update its moments produce, never of the gradient. The cell reported is its deferred path, so it understates the residency the classification permits; the update computed is identical either way.

One fixed protocol across families: the sweep tests the \emph{mechanism} under a shared configuration, not per-family tuned convergence.

\begin{center}\small
\begin{tabular}{@{}l llll@{}}
\toprule
family & lr & wd & $(\beta_1,\beta_2)$ & $\epsilon$ \\
\midrule
adafactor & $10^{-3}$ & 0 & --- & --- \\
adagrad & $2{\times}10^{-5}$ & $10^{-2}$ & --- & $10^{-10}$ \\
adalomo & $10^{-3}$ & 0 & --- & $10^{-8}$ \\
adam\_mini & $2{\times}10^{-5}$ & $10^{-2}$ & (0.9, 0.999) & $10^{-8}$ \\
adamw / nadam / radam & $2{\times}10^{-5}$ & $10^{-2}$ & (0.9, 0.999) & $10^{-8}$ \\
lamb & $2{\times}10^{-5}$ & $10^{-2}$ & (0.9, 0.999) & $10^{-6}$ \\
lion & $3{\times}10^{-6}$ & 1.0 & (0.9, 0.99) & --- \\
rmsprop & $2{\times}10^{-5}$ & $10^{-2}$ & --- & $10^{-8}$ \\
sgd / sgd\_momentum & $2{\times}10^{-5}$ & $10^{-2}$ & ($\beta{=}0.9$) & --- \\
sm3 & $2{\times}10^{-5}$ & --- & --- & $10^{-8}$ \\
\bottomrule
\end{tabular}
\end{center}

\begin{figure}[p]
\centering
\includegraphics[width=0.495\linewidth]{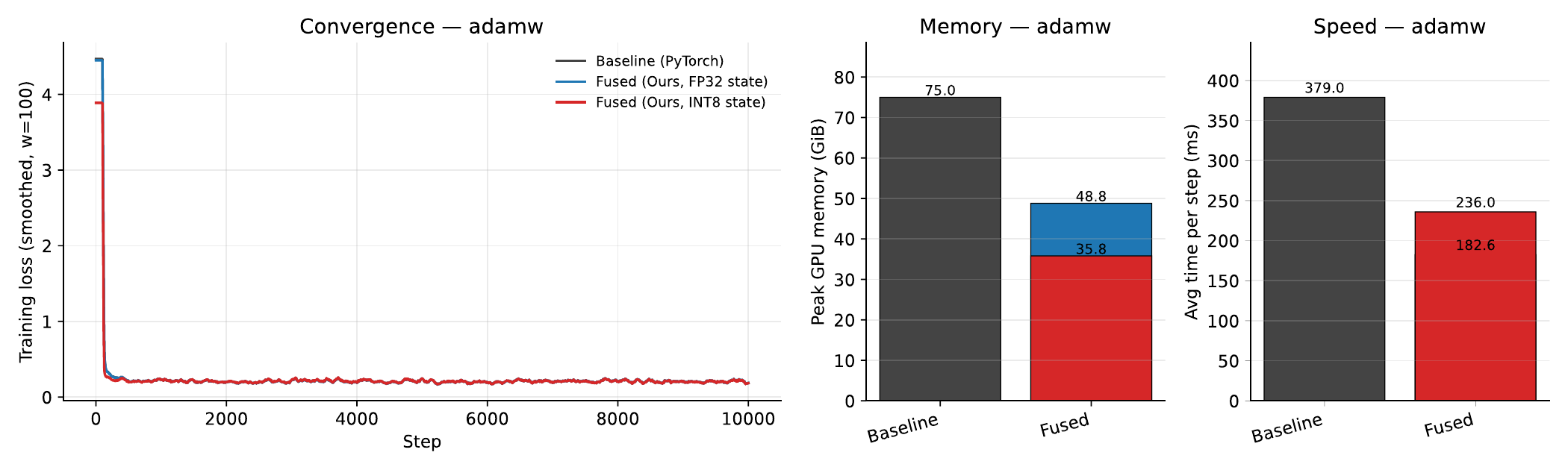}\hfill\includegraphics[width=0.495\linewidth]{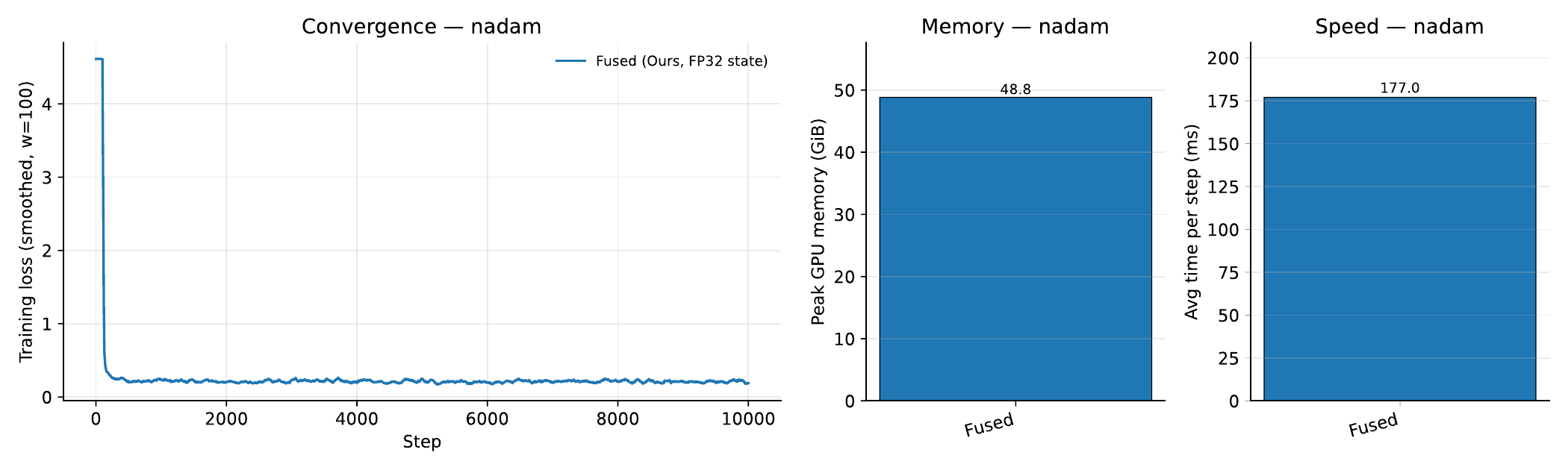}\\[1.5pt]
\includegraphics[width=0.495\linewidth]{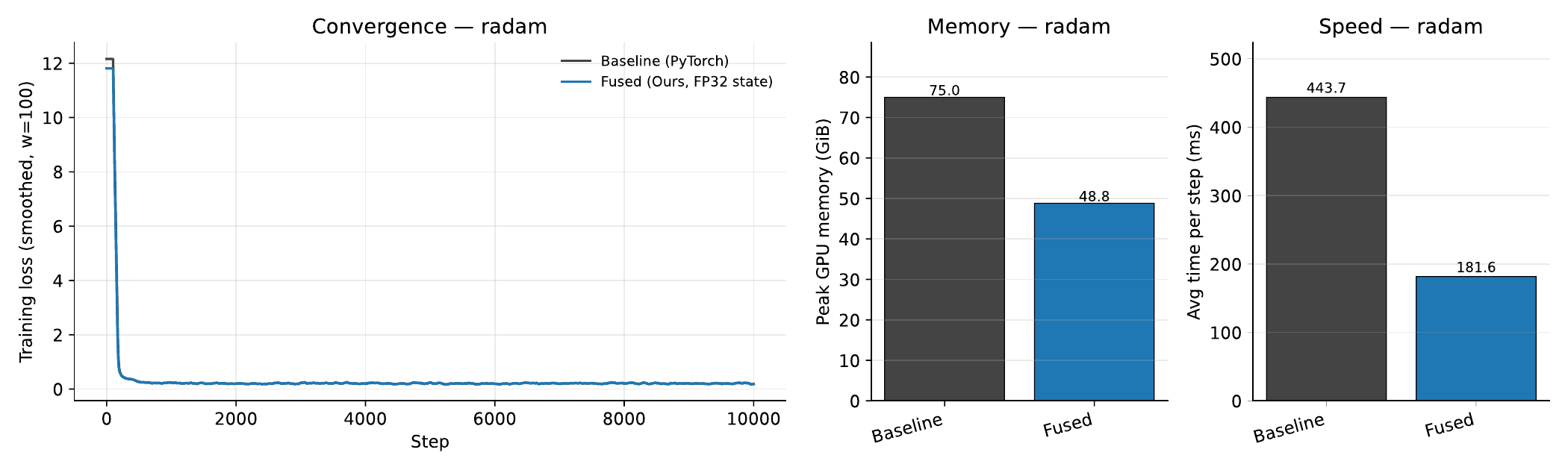}\hfill\includegraphics[width=0.495\linewidth]{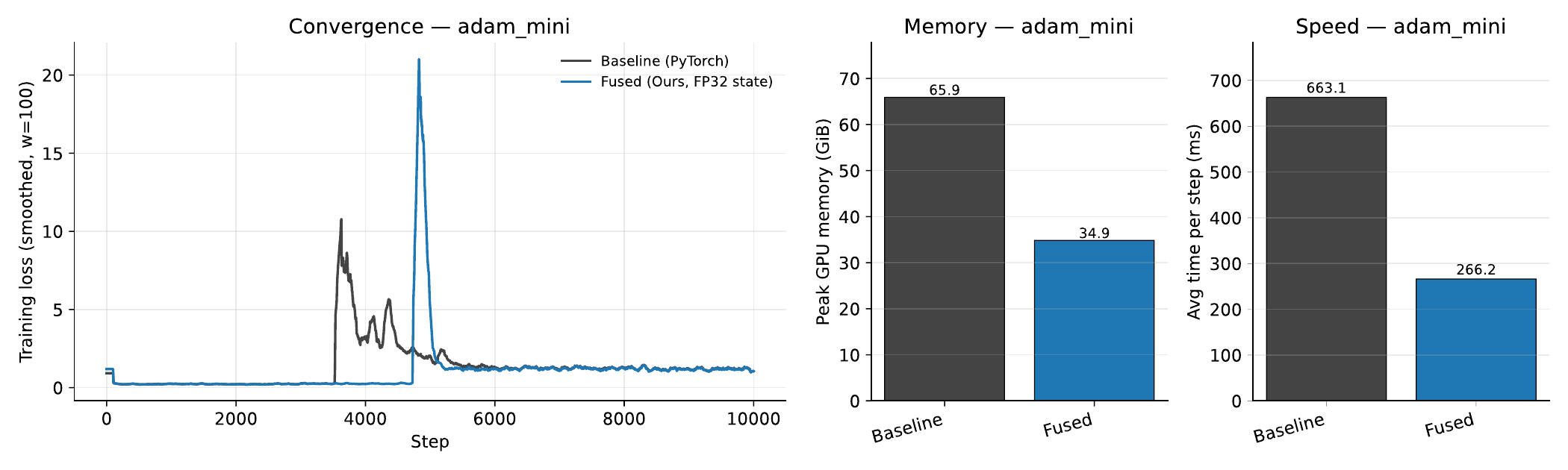}\\[1.5pt]
\includegraphics[width=0.495\linewidth]{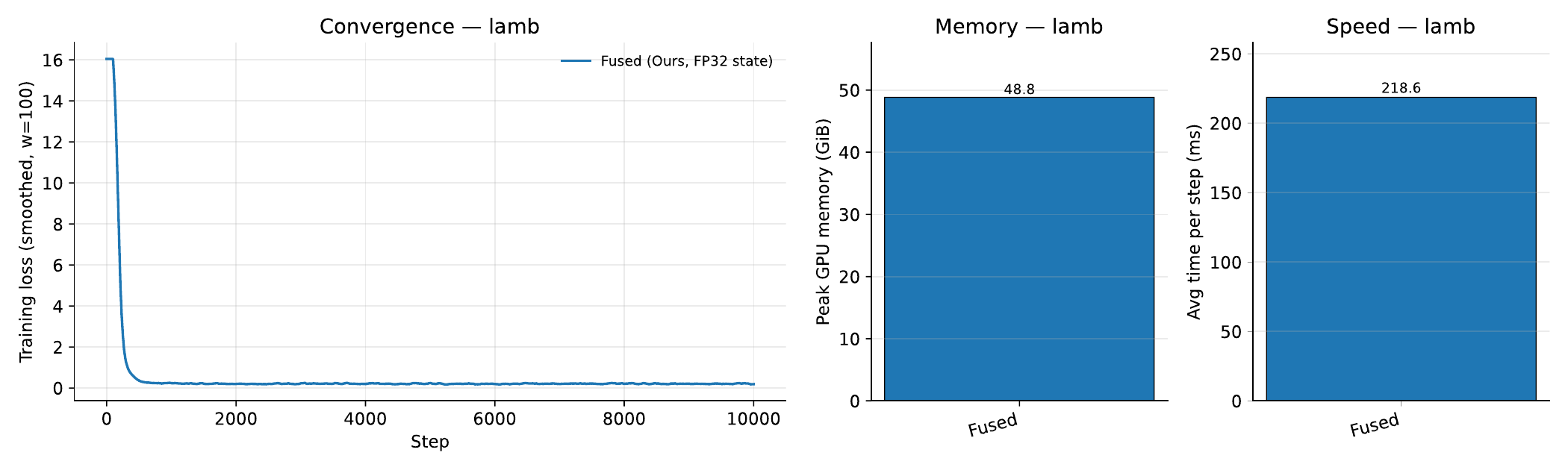}\hfill\includegraphics[width=0.495\linewidth]{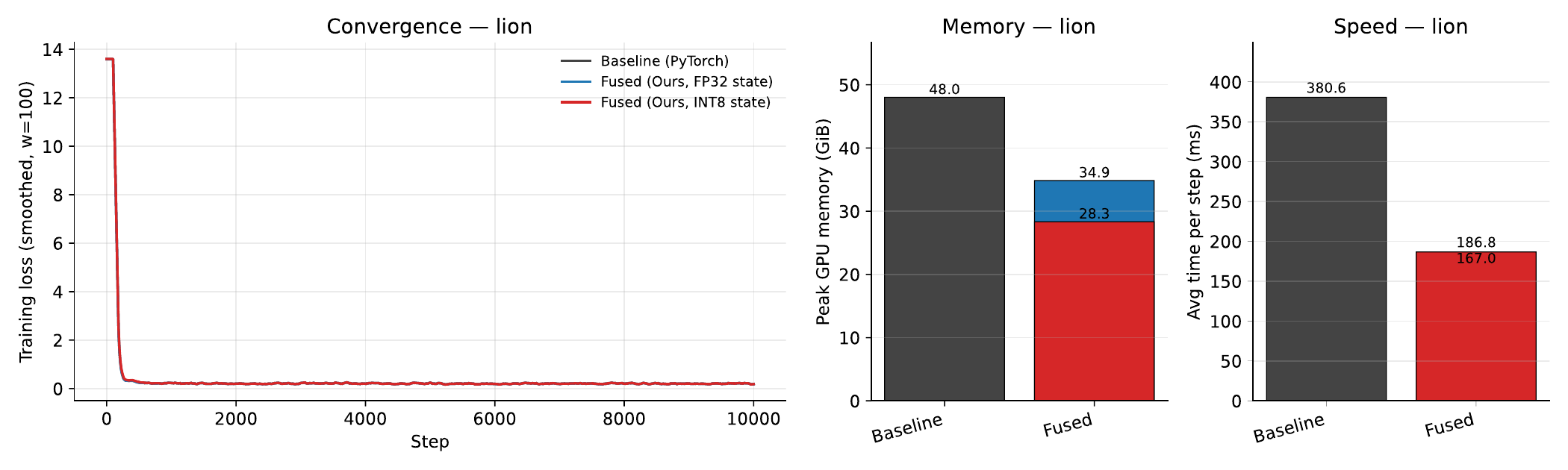}\\[1.5pt]
\includegraphics[width=0.495\linewidth]{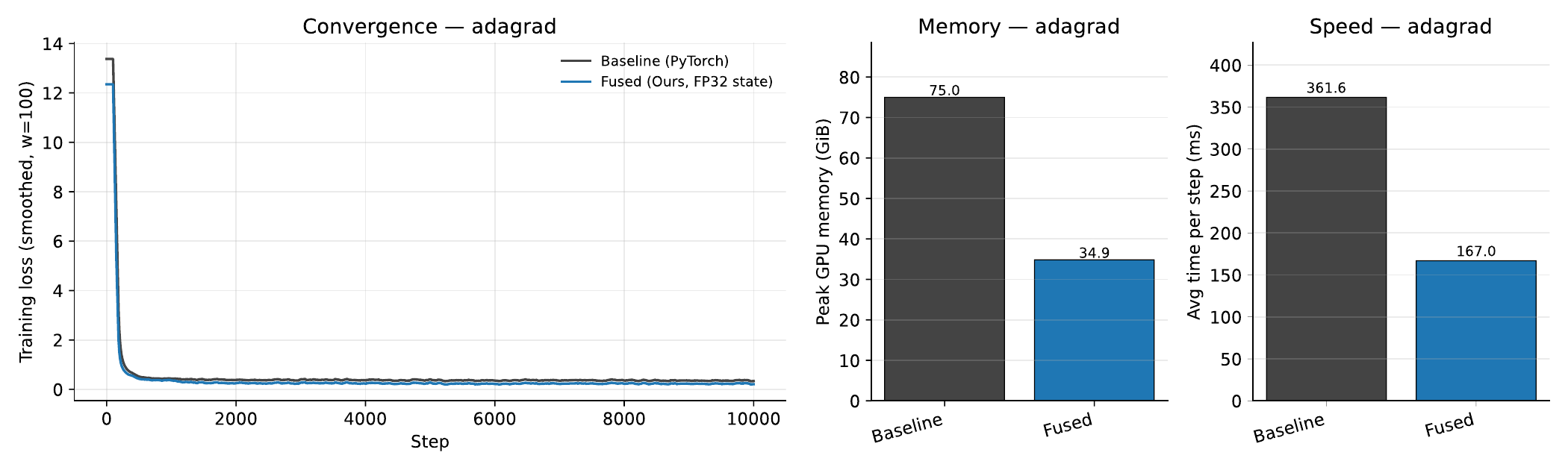}\hfill\includegraphics[width=0.495\linewidth]{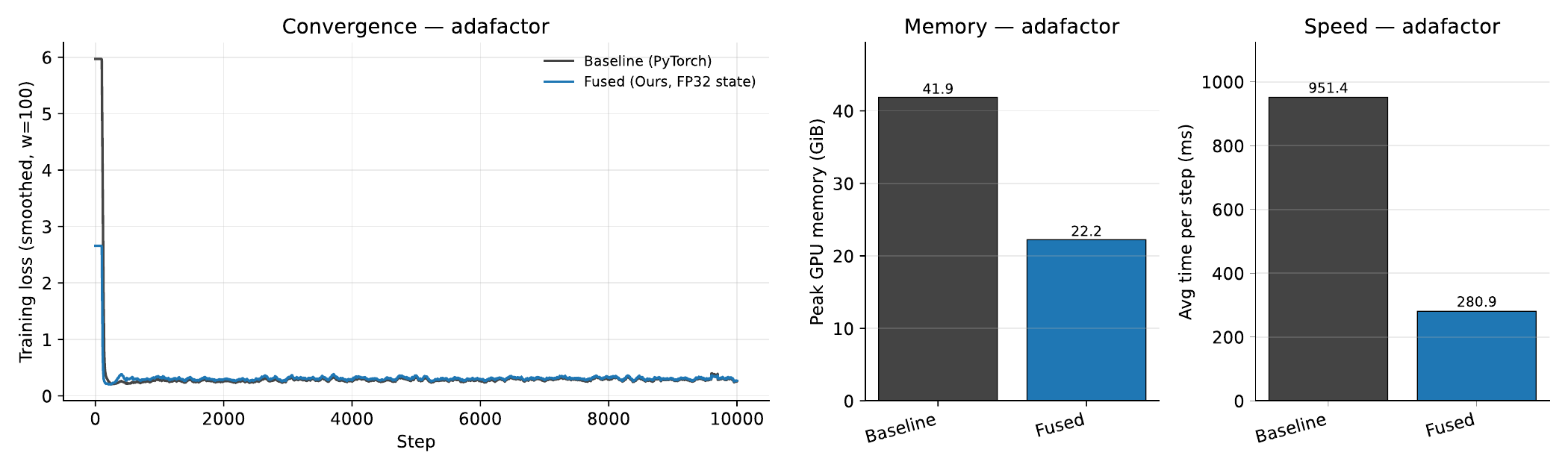}\\[1.5pt]
\includegraphics[width=0.495\linewidth]{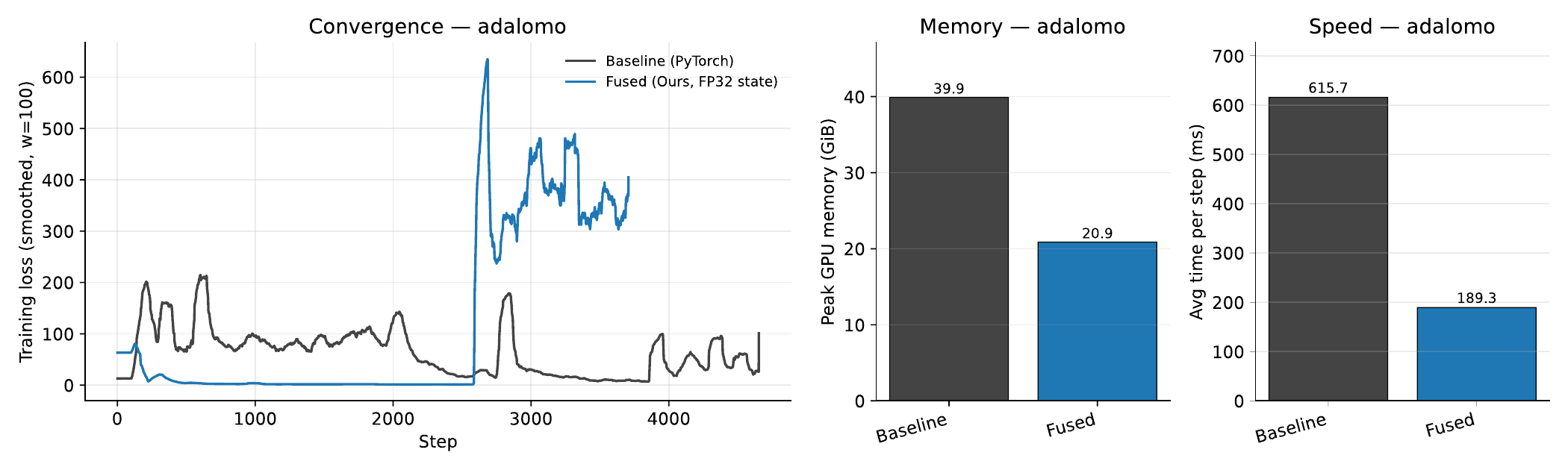}\hfill\includegraphics[width=0.495\linewidth]{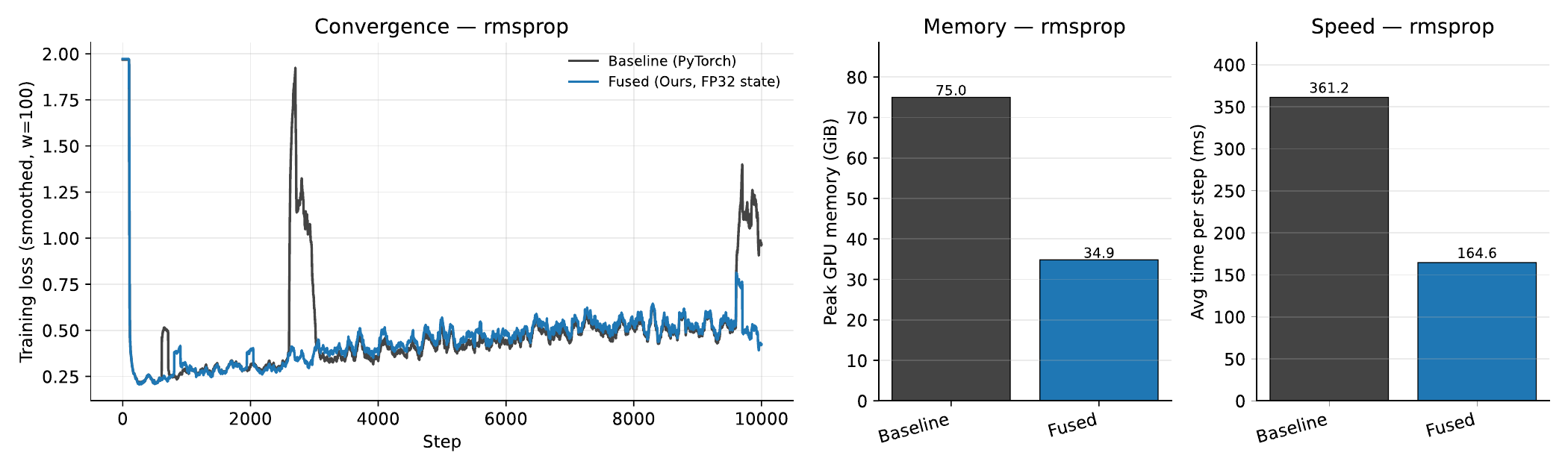}\\[1.5pt]
\includegraphics[width=0.495\linewidth]{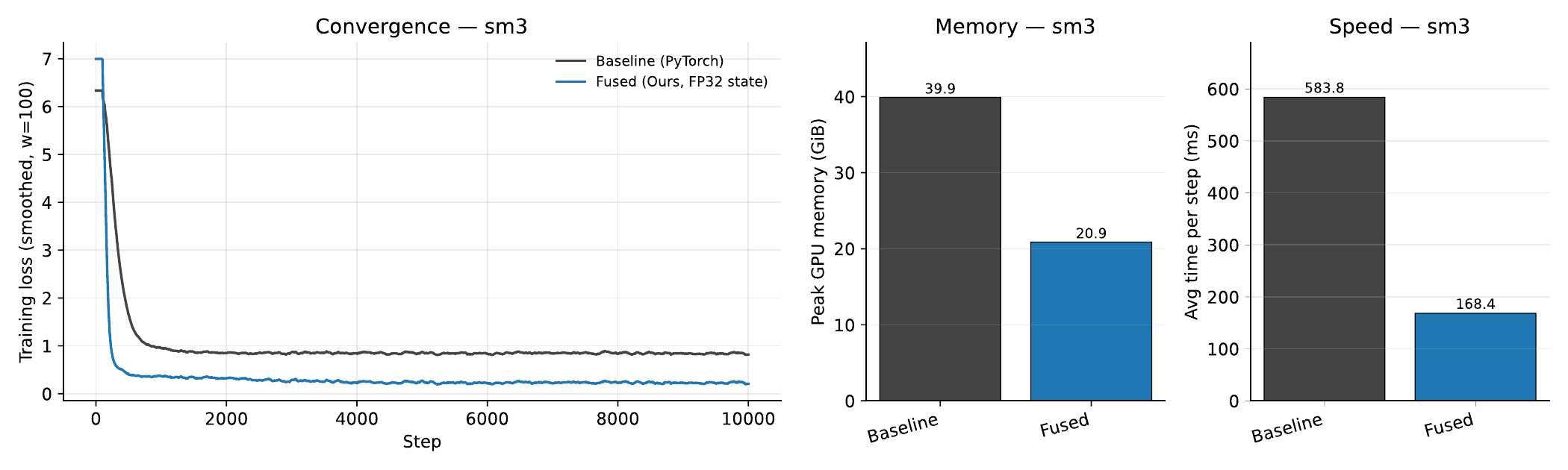}\hfill\includegraphics[width=0.495\linewidth]{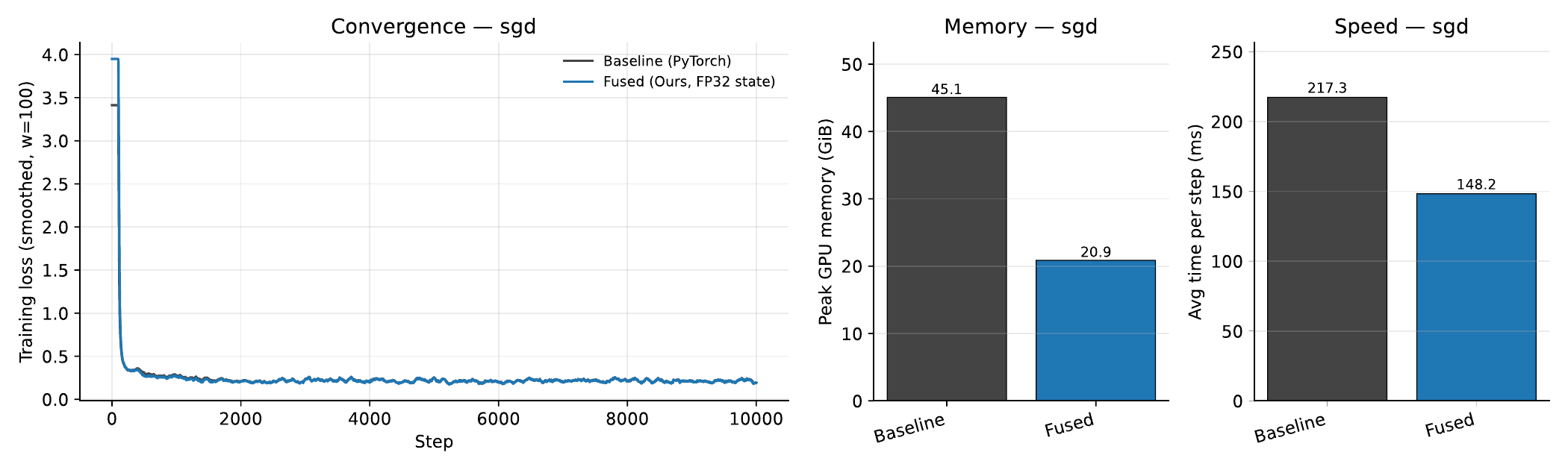}\\[1.5pt]
\includegraphics[width=0.495\linewidth]{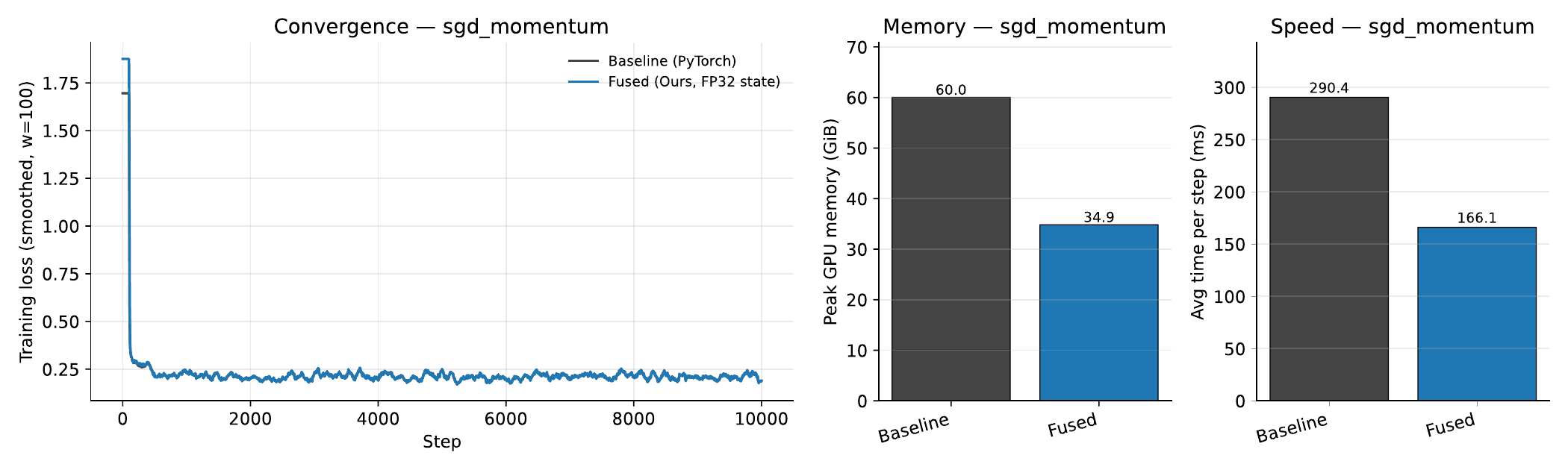}
\caption{Family sweep, all thirteen optimizer families under the shared protocol --- per family: training loss (left), peak memory (center), step time (right). Top to bottom, left to right: AdamW, NAdam, RAdam, Adam-mini, LAMB, Lion, Adagrad, Adafactor, AdaLomo, RMSprop, SM3, SGD, SGD with momentum.}
\label{fig:zoo1}
\end{figure}

\subsection{Regime 2 measured: Muon}\label{app:zoomuon}

Qwen3 dense ladder, one H200 NVL (139.8~GB usable), bf16 weights and momentum, $\mu{=}0.95$, Newton--Schulz 5 steps, batch~1 sequence~512. Both arms run the same Newton--Schulz on the same shapes, so $U$ cancels and the difference is exactly what $A$'s fusion does: \texttt{muon} writes $\mathbf{G}$ to HBM and folds it into $\mathbf{B}$ in a separate elementwise pass; \forge{}-Muon accumulates $\dY^{\!\top}\mathbf{X}$ into $\mathbf{B}$ through the GEMM's $\beta$ and $\mathbf{G}$ never exists. Activation checkpointing on throughout, the one condition in which every cell is comparable.

\begin{center}\small
\setlength{\tabcolsep}{5pt}
\begin{tabular}{@{}l rr rr r r@{}}
\toprule
& & & \multicolumn{2}{c}{saved} & removed & step \\
\cmidrule(lr){4-5}
model & Muon & \forge{}-Muon & AC on & AC off & @ bnd. & ratio \\
\midrule
Qwen3-0.6B & 4.32 & \textbf{3.50} & 18.9\% & $-0.7\%$ & 0.85 & $1.013\times$ \\
Qwen3-1.7B & 11.42 & \textbf{8.79} & 23.0\% & 17.4\% & 2.62 & $0.970\times$ \\
Qwen3-4B & 24.86 & \textbf{18.05} & 27.4\% & 20.5\% & 6.81 & $0.998\times$ \\
Qwen3-8B & 48.69 & \textbf{35.66} & 26.8\% & 22.1\% & 12.94 & $1.007\times$ \\
Qwen3-14B & 86.42 & \textbf{61.64} & 28.7\% & 24.6\% & 24.61 & $1.005\times$ \\
Qwen3-32B & OOM & \textbf{129.00} & capability & both OOM & --- & --- \\
\bottomrule
\end{tabular}
\end{center}

\noindent At Qwen3-32B standard Muon needs 61.0 (weights) $+$ 58.1 (momentum) $+$ 58.1 (gradient pool) $\approx177$~GB against the 139.8 available and does not fit at any setting; \forge{}-Muon trains at 129.00~GB, 7307~ms/step, with 10.8~GB of headroom. Checkpointing is what brings it inside the card: without it the arm peaks at 138.45~GB.

\paragraph{The removed bytes are the pool, exactly.} Sampled live at the backward-to-optimizer boundary, checkpointing off:

\begin{center}\small
\setlength{\tabcolsep}{6pt}
\begin{tabular}{@{}l rrrr@{}}
\toprule
model & Muon @ bnd. & \forge{} @ bnd. & removed & live gradient pool \\
\midrule
Qwen3-0.6B & 3.73 & 2.90 & 0.83 & 0.82 \\
Qwen3-1.7B & 10.26 & 7.63 & 2.62 & 2.62 \\
Qwen3-4B & 23.41 & 16.59 & 6.81 & 6.77 \\
Qwen3-8B & 48.16 & 35.22 & 12.94 & 12.94 \\
Qwen3-14B & 85.49 & 60.88 & 24.61 & 24.61 \\
\bottomrule
\end{tabular}
\end{center}

\noindent Removed equals the pool at every size, and is token-invariant: Qwen3-8B removes 12.94~GB at every $BT$ from 128 to 4096. Whether that reaches the \emph{peak} is the separate question, and it is the $BT$ crossover of the AdamW path reproduced --- peak saving \%:

\begin{center}\small
\setlength{\tabcolsep}{6pt}
\begin{tabular}{@{}l rrrrrr@{}}
\toprule
model & $BT{=}128$ & 256 & 512 & 1024 & 2048 & 4096 \\
\midrule
Qwen3-1.7B & 23.0 & 23.0 & 17.4 & 0.0 & 0.0 & 0.0 \\
Qwen3-4B & 27.4 & 26.9 & 20.5 & 4.9 & $-0.1$ & $-0.0$ \\
Qwen3-8B & 26.8 & 26.1 & 22.1 & 13.6 & 0.0 & 0.0 \\
\bottomrule
\end{tabular}
\end{center}

\noindent The crossover moves right with model size, and batch and sequence are interchangeable at matched $BT$ to the digit (Qwen3-1.7B at $BT{=}1024$: 11.82/11.82~GB by either route). Every constant-$BT$ anti-diagonal of the batch$\times$sequence grid is identical to the digit on both arms: at Qwen3-8B, $BT{=}2048$ gives 51.2/51.2 by all three routes and $BT{=}4096$ gives 69.6/69.5 by all three, so the saving is a function of $BT$ alone (22\% at $BT{=}512$, 14\% at 1024, zero from 2048).

\paragraph{Checkpointing and gradient elimination are complementary.} The two \emph{saved} columns above are the same models with checkpointing on and off. Checkpointing removes the token-scaled activation crest and the fusion removes the parameter-scaled pool, so only with both does the removed term sit at the peak: the bytes removed at the boundary are identical in either condition, and checkpointing decides only whether they reach the peak. That is why Qwen3-0.6B moves from $-0.7\%$ to $18.9\%$. At 32B the two are not alternatives --- without checkpointing \forge{}-Muon peaks at 138.45~GB against 139.8 usable and misses by ${\sim}1.4$~GB, so that cell is a near-miss rather than a wide margin.

\paragraph{Where the step time goes.} Checkpointing off, batch~1 sequence~512; $\Delta$mom is the separate momentum pass the fusion deletes:

\begin{center}\small
\setlength{\tabcolsep}{5pt}
\begin{tabular}{@{}l rr r r r@{}}
\toprule
model & Muon & \forge{}-Muon & ratio & NS share & $\Delta$mom \\
\midrule
Qwen3-0.6B & 111.4 & 114.4 & $0.973\times$ & 53.9\% & $-2.00$ \\
Qwen3-1.7B & 229.4 & 231.2 & $0.992\times$ & 76.4\% & $-4.28$ \\
Qwen3-4B & 594.6 & 591.8 & $1.005\times$ & 88.0\% & $-10.33$ \\
Qwen3-8B & 1454.6 & 1440.7 & $1.010\times$ & 94.2\% & $-18.15$ \\
Qwen3-14B & 3219.1 & 3190.7 & $1.009\times$ & 96.1\% & $-33.52$ \\
\bottomrule
\end{tabular}
\end{center}

\noindent $\Delta$mom scales linearly with Muon-governed parameter count ($-2.0$ to $-33.5$~ms) and is the cleanest evidence that the mechanism does what it claims. End to end the arms are at parity because Newton--Schulz is $54$--$96\%$ of the step at batch~1: its cost is fixed per optimizer step and independent of tokens, so at 512 tokens nothing else is visible. Below ${\sim}$4B a per-layer \texttt{autograd.Function} overhead that tracks matrix count rather than size exceeds the fused accumulate's win, which is what makes the two smallest arms net-negative.

\paragraph{Accumulation is where the speed result lives.} The deleted momentum pass is paid once per optimizer step while the per-micro-batch delta is paid $K$ times, so depth amplifies whichever sign the model size sets:

\begin{center}\small
\setlength{\tabcolsep}{5pt}
\begin{tabular}{@{}l r r rr r r@{}}
\toprule
model & $K$ & $BT$ & Muon & \forge{} & ratio & NS share \\
\midrule
Qwen3-1.7B & 1 & 512 & 11.42 & 9.44 & $0.996\times$ & 76.6\% \\
 & 16 & 8{,}192 & 12.64 & 10.02 & $0.918\times$ & 16.9\% \\
 & 64 & 32{,}768 & 12.64 & 10.02 & $0.929\times$ & 5.1\% \\
\addlinespace[2pt]
Qwen3-8B & 1 & 512 & 48.69 & 37.92 & $1.008\times$ & 94.2\% \\
 & 16 & 8{,}192 & 53.18 & 40.24 & $1.045\times$ & 50.9\% \\
 & 64 & 32{,}768 & 53.18 & 40.24 & $\mathbf{1.073\times}$ & 21.1\% \\
\bottomrule
\end{tabular}
\end{center}

\noindent Qwen3-8B's backward delta is almost exactly linear in $K$ ($-7.8$, $-31.4$, $-133.0$, $-516.7$~ms at $K{=}1,4,16,64$, or ${\approx}-8$~ms per micro-batch), and the memory saving is flat in $K$ at $24.3\%$ while the materialized arm's peak grows. The operating point for a regime-2 rule is therefore a large model with deep accumulation, which is also where Newton--Schulz amortizes.

\paragraph{Recipe variants (Qwen3-8B, batch~1 sequence~512).}

\begin{center}\small
\setlength{\tabcolsep}{5pt}
\begin{tabular}{@{}l rr r r r@{}}
\toprule
recipe & Muon & \forge{}-Muon & saved & ratio & NS phase \\
\midrule
fp32 momentum & 61.63 & 50.86 & 17.5\% & $\mathbf{1.028\times}$ & 1365.9 \\
bf16 momentum & 48.69 & 37.92 & \textbf{22.1\%} & $1.010\times$ & 1356.5 \\
NS steps $=3$ & 48.69 & 37.92 & 22.1\% & $1.015\times$ & \textbf{830.0} \\
NS fused-axpy & 48.50 & 37.92 & 21.8\% & $1.012\times$ & 1281.4 \\
\bottomrule
\end{tabular}
\end{center}

\noindent fp32 momentum costs 12.9~GB at 8B but gives the best ratio, because the pass the materialized arm must run moves twice the bytes. Cutting Newton--Schulz from five steps to three removes $39\%$ of the optimizer phase and is the largest single lever on step time, but it changes the optimizer and is a recipe choice rather than a free win.

\paragraph{Distributed Muon.} Under data parallelism the reduce-into-state construction of Proposition~\ref{prop:dist}(iii) applies and no gradient is formed on any rank. Per-rank peak is identical at DP1/2/4/8 --- a replicated momentum does not grow with world size --- and wire volume is byte-identical between arms, since the same $P$ elements are reduced either way and only the object reduced changes:

\begin{center}\small
\setlength{\tabcolsep}{5pt}
\begin{tabular}{@{}l rr r r@{}}
\toprule
model & Muon & \forge{}-Muon & saved & step (Muon $\to$ \forge{}) \\
\midrule
Qwen3-1.7B & 11.42 & \textbf{9.44} & 17.4\% & $368.7 \to 369.8$~ms \\
Qwen3-4B & 24.86 & \textbf{19.75} & 20.5\% & $868.3 \to 858.4$~ms \\
Qwen3-8B & 48.69 & \textbf{37.92} & 22.1\% & $1909.1 \to 1898.9$~ms \\
Qwen3-14B & 86.42 & \textbf{65.13} & 24.6\% & $4055.5 \to 4021.8$~ms \\
\bottomrule
\end{tabular}
\end{center}

\noindent The gate (world size 4) agrees to one bf16 rounding: momentum relative deviation $3.1\times10^{-3}$, weights $8.1\times10^{-4}$. It is not bit-exact, because NCCL's reduction tree does not reproduce the reference's sequential order.

\paragraph{Tensor parallelism, and the one axis where a regime-2 rule is worse.} A shard's owner sees every token, so the fused accumulate needs no collective at all; but Newton--Schulz reads the whole layer matrix, so the weight update must all-gather $\mathbf{B}$. Both effects are visible at Qwen3-8B:

\begin{center}\small
\setlength{\tabcolsep}{5pt}
\begin{tabular}{@{}l r rr r r@{}}
\toprule
model & TP & Muon & \forge{}-Muon & saved & NS gather \\
\midrule
Qwen3-4B & 2 & 15.40 & \textbf{12.54} & 18.6\% & 3.63B \\
 & 8 & 9.40 & \textbf{8.55} & 9.0\% & 3.63B \\
\addlinespace[2pt]
Qwen3-8B & 2 & 29.05 & \textbf{23.28} & 19.9\% & 6.95B \\
 & 4 & 21.46 & \textbf{18.31} & 14.7\% & 6.95B \\
 & 8 & 17.66 & \textbf{16.05} & 9.2\% & 6.95B \\
\bottomrule
\end{tabular}
\end{center}

\noindent The saving decays as $1/\mathrm{TP}$ because sharding divides weights, momentum and the pool alike while activations stay per-rank --- the $BT$ crossover on a different axis. The gather does not: it is the full parameter count at every degree, so at 8B/TP8 each rank moves 13.9~GB per step to own a 1.62~GB shard. \emph{A coordinate-wise rule needs nothing here}, because its update never crosses a shard boundary. This is the one axis on which a regime-2 rule is strictly worse than the regime-1 rule it replaces, and it worsens as sharding deepens. The sharded forward reproduces the unsharded one at relative $8.9\times10^{-3}$ and gathered Newton--Schulz is bit-identical to single-device Newton--Schulz in both orientations.

\paragraph{Context parallelism: a negative result.} The prediction was that context parallelism shrinks the activation crest and so exposes the pool saving. It does not, at either size or sequence:

\begin{center}\small
\setlength{\tabcolsep}{5pt}
\begin{tabular}{@{}l r rr r r@{}}
\toprule
model & global seq & Muon & \forge{}-Muon & saved & @ bnd.\ (M $\to$ F) \\
\midrule
Qwen3-1.7B & 4{,}096 & 12.81 & 12.81 & \textbf{0.0\%} & $10.26 \to 7.63$ \\
Qwen3-1.7B & 16{,}384 & 32.68 & 32.68 & \textbf{0.0\%} & $10.26 \to 7.63$ \\
Qwen3-4B & 4{,}096 & 26.00 & 26.00 & \textbf{0.0\%} & $23.40 \to 16.58$ \\
Qwen3-4B & 16{,}384 & 59.19 & 59.19 & \textbf{0.0\%} & $23.32 \to 16.55$ \\
\bottomrule
\end{tabular}
\end{center}

\noindent The implementation is why. Ulysses-style context parallelism all-gathers K and V, so every rank materializes them at full sequence and the per-rank peak grows with the \emph{global} sequence even though each rank owns $S/\mathrm{CP}$ queries. The crest is not sharded, so the removal stays invisible in the peak while the boundary removal is unaffected and identical to the data-parallel case. Obtaining the predicted benefit requires ring attention, which shards K and V instead of gathering them. Fully sharded training is a second gap: \texttt{fully\_shard} makes the weight a \texttt{DTensor} while $\mathbf{B}$ and the Newton--Schulz output are plain tensors, and the weight update raises on the mixed types; a fix needs a \texttt{DTensor}-aware Newton--Schulz, which would pay the same gather as tensor parallelism.

\paragraph{Mixture of experts: the second capability cell.} Qwen3-30B-A3B (30.53B parameters, 128 experts over 48 layers) on one H200 with checkpointing, batch~1 sequence~512: standard Muon is out of memory --- weights 57 $+$ momentum 56 $+$ gradient pool 56 ${\approx}$ 170~GB against 139.8 usable --- while \forge{}-Muon trains at \textbf{120.67~GB}, 6797~ms/step. The first attempt reached only $3\%$ of parameters and both arms went out of memory, for the reason Appendix~\ref{app:archgen} documents on Mamba-2: this \texttt{transformers} version stores the experts as fused 3-D parameters (\texttt{gate\_up\_proj}, \texttt{down\_proj}) rather than a list of \texttt{nn.Linear}, so 29B of expert weight never reached the module interface and fell through to AdamW. Accumulating into a \emph{slice} of a 3-D momentum takes coverage to $98\%$; the mechanism needs nothing new, because an expert's weight gradient is still $\dY_e^{\!\top}\mathbf{X}_e$ over its own routed tokens and the accumulate is still one \texttt{addmm} with $\beta$. Each expert's gradient is complete locally, so no reduction is required even under expert parallelism. Remark~\ref{rem:untouched} is what the routed case makes concrete: an expert nothing routes to on a step is never touched, so decay-on-first-touch would silently skip it, and the untouched experts are decayed explicitly at the end of the step. Batching Newton--Schulz over same-shape experts is bit-identical to the serial path and $2.7$--$4.2\times$ faster ($128\times(768,2048)$: 42.0 to 10.0~ms), which matters because 18{,}624 governed matrices at 15 small GEMMs each is ${\sim}280$k launches per step.

\paragraph{Scope of the Muon programme.} Built, gated and swept: the single-GPU ladder and batch$\times$sequence grid, checkpointing on and off, bf16 and fp32 momentum, data parallelism at 1/2/4/8, tensor parallelism at 2/4/8, context parallelism, and the mixture-of-experts cell. Not available: int8 and fp8 states and FP8 weights are not implemented on the Muon path; fully sharded training is incompatible as described above; pipelining is the accumulation case above plus a schedule that was not built. Two caveats bound the mixture-of-experts cell specifically: the arms do not govern identical parameter sets (98\% under the layer rule against 100\% under the control's \texttt{ndim}$\ge2$ rule, immaterial against a ${\sim}30$~GB margin but not matched), and the router is 2-D and therefore falls inside the Muon set by the stated rule, where it is often deliberately excluded. Losses in this programme are recorded on randomly initialized models with an untuned learning rate: these cells measure memory and time, not convergence.

\section{Method Taxonomy}\label{app:taxonomy}

Where \forge{} sits against the representative memory-efficient methods, on the properties the main paper's comparisons exercise:

\begin{center}\small
\setlength{\tabcolsep}{4pt}
\begin{tabular}{@{}l cccccc@{}}
\toprule
property & AdamW & bnb 8-bit & GaLore & APOLLO & COAT & \forge{} \\
\midrule
granularity & per-step & per-param & per-step & per-step & per-step & per-tile \\
tile-local $A$ (Def.~\ref{def:tl}) & \checkmark & \checkmark & --- & --- & \checkmark & \checkmark \\
full Adam dynamics & \checkmark & \checkmark & --- & --- & \checkmark & \checkmark \\
algebraically exact (fp32 state) & \checkmark & --- & --- & --- & --- & \checkmark \\
$\nabla_{\mathbf{W}}\mathcal{L}$ in HBM & \checkmark & \checkmark & \checkmark & \checkmark & \checkmark & --- \\
kernel-level fusion & --- & --- & --- & --- & \checkmark & \checkmark \\
FP8 compatible & --- & --- & --- & --- & \checkmark & \checkmark \\
DP / FSDP & \checkmark & \checkmark & \checkmark & \checkmark & \checkmark & \checkmark$^{a}$ \\
peak saving vs.\ fused (8B cell) & --- & 39\% & 50\% & 49\% & ref$^{b}$ & 53\% \\
\bottomrule
\end{tabular}
\end{center}

\noindent \forge{} is the only method that is simultaneously per-tile, full-Adam-preserving, and algebraically exact in fp32 state. $^{a}$Zero-materialization under tensor, sequence, and expert sharding; bucket-transient under data and context parallelism and FSDP (Appendix~\ref{app:exact}). $^{b}$COAT's own BF16-everywhere anchor is the correct reference for its rows; cross-regime comparison against COAT's fp32-master$+$SDPA recipe would be misleading, so none is drawn. Savings are the Llama-3.1-8B, BS=1, $S{=}512$ H200 cell (Appendix~\ref{app:fullh200}).

\subsection{Composability, by mechanism}\label{app:composeclass}

The table above places methods on shared properties. The composition experiments ask something narrower: \emph{what does this method do at the backward-to-optimizer boundary, and can \forge{} act at the same boundary?} Cells are Qwen3-8B, one H200, batch~1 sequence~512, bf16, no activation checkpointing; each method alone and composed with \forge{}. Target layers, non-target handling, seed, data, warmup and peak metric are identical across arms.

\begin{center}\footnotesize
\setlength{\tabcolsep}{4pt}
\begin{tabular}{@{}l l rr r rr r@{}}
\toprule
& & \multicolumn{2}{c}{peak GB} & & \multicolumn{2}{c}{ms/step} & \\
\cmidrule(lr){3-4}\cmidrule(lr){6-7}
method & verdict & alone & $+$\forge{} & $\Delta$ & alone & $+$\forge{} & speedup \\
\midrule
vanilla AdamW & subsumed & 76.35 & 50.86 & $\mathbf{-33.4\%}$ & 190.8 & 143.0 & $\mathbf{1.33\times}$ \\
fused AdamW & subsumed & 61.10 & 50.86 & $\mathbf{-16.8\%}$ & 122.9 & 143.0 & $0.86\times$ \\
bitsandbytes 8-bit & compatible & 46.08 & 38.86 & $\mathbf{-15.7\%}$ & 323.2 & 169.9 & $\mathbf{1.90\times}$ \\
fp8 states & compatible & \multicolumn{1}{c}{n/a} & 37.93 & --- & \multicolumn{1}{c}{n/a} & 292.5 & --- \\
GaLore $r{=}128$ & outside & 38.83 & 27.15 & $\mathbf{-30.1\%}$ & 136.8 & 172.0 & $0.80\times$ \\
APOLLO $r{=}256$ & outside & 39.85 & 29.18 & $\mathbf{-26.8\%}$ & 176.0 & 224.4 & $0.78\times$ \\
AdaLomo & compatible, reg.~3 & 26.00 & 25.60 & $-1.5\%$ & 620.1 & 253.3 & $\mathbf{2.45\times}$ \\
optimi grad-release & mechanism-exclusive & 67.13 & \multicolumn{2}{c}{does not exist} & 202.9 & \multicolumn{2}{c}{---} \\
FlashOptim & mechanism-exclusive & \multicolumn{6}{c}{no installable package} \\
\bottomrule
\end{tabular}
\end{center}

\noindent The two AdamW arms share one composed cell: the epilogue replaces the multi-tensor update kernel and the pool it reads together. The fused-AdamW arm reproduces Appendix~\ref{app:h200qwen} across harnesses, 61.10 against 61.31~GB; the vanilla arm sits one parameter-sized temporary higher, 76.35 against 63.63, this harness leaving the \texttt{foreach} intermediates in place. GaLore's and APOLLO's stock classes read \texttt{p.grad}, so their composed arms are a low-rank Adam with the projection folded into the weight-gradient GEMM, matched on rank, scale, betas, epsilon, weight decay and learning rate: equivalent in form and not bit-equivalent, which is what licenses the memory and speed comparison and stops short of a convergence claim.

\paragraph{The memory result is unconditional; the speed result is the bound.} Memory falls against every method that composes, $16$--$33\%$, and furthest against those that shrink \emph{state} and leave the gradient alone: GaLore $-30.1\%$ and APOLLO $-26.8\%$ are \texttt{torch.optim} subclasses reading \texttt{p.grad}, so the $O(P)$ pool is fully resident at the boundary --- exactly the term \forge{} deletes. Step time follows the bound the mechanism sets rather than an independent trend. Fusion can remove only the separate optimizer phase, so the gain is that phase's share of the step and no more: it is largest where the phase dominates (bitsandbytes' int8 unpacking $1.90\times$, AdaLomo's per-parameter hooks $2.45\times$) and vanishes where the phase is already a fused multi-tensor kernel costing 122.9~ms of a 61~GB step ($0.86$--$0.78\times$ for fused AdamW, GaLore and APOLLO on Qwen3's shapes). The column measures the operating condition of the main paper's Section~3.4 on the composition axis, and the memory column is what does not depend on it.

\paragraph{The verdicts.} \emph{Subsumed}: both AdamW arms are regime~1 by Definition~\ref{def:tl}, so this is not composition --- \forge{} \emph{is} that update, moved into the epilogue. \emph{Mechanism-exclusive}: optimi and FlashOptim occupy the same scheduling slot, being coarser points on the granularity ladder of the main paper's Section~1, and a weight cannot be co-owned by two schedules without being stepped twice; scheduling decides this one, not the criterion of Appendix~\ref{app:exact}. \emph{Compatible}: quantized state changes what the state costs without changing what $A$ reads, so the savings add. \emph{Compatible through the deferred path}: AdaLomo's factored moment is a row-and-column reduction of $\mathbf{G}^{\odot2}$, hence regime~3, and composes at layer granularity --- a small memory delta, its factored state already near the floor, and a large time delta, the deferred path replacing its own per-parameter hooks. \emph{Outside the contract}: GaLore and APOLLO project the gradient before the update reads it, so a projected coordinate depends on a whole row or column of $\mathbf{G}$; $A$ is not tile-local and the dynamics are approximate. Whether the projection can nonetheless be pushed into the weight-gradient GEMM is a mechanical question with its own answer, reported with the composition measurements.

\section{Distributed Protocol and Recipes}\label{app:protocol2}

\paragraph{Node.} 8$\times$NVIDIA H200 NVL (139.8~GB each) with NVLink, idle during the sweep. The companion program in Appendix~\ref{app:rtx} ran on 8$\times$RTX PRO 6000 (96~GB, PCIe, no NVLink); the two boxes are never merged in any table.

\paragraph{Protocol.} One process tree per cell; 5 warmup $+$ 20 measured steps (3$+$10 under context parallelism), median reported; peak is the per-rank allocator maximum, reset after warmup; gradients averaged, no accumulation, no clipping; learning rate, betas, epsilon and weight decay identical across arms. Losses are recorded every measured step and checked per table to land in a common band, so no method buys memory with divergence; that check is why result tables carry no loss column. Out-of-memory cells are capability boundaries.

\paragraph{Recipes and fairness.} recipe classes are named descriptively --- \emph{bf16-state}: bf16 weights and bf16 moments, no master copy; \emph{fp32-master}: bf16 compute weights with fp32 master and fp32 moments; \emph{SR-store}: bf16 store committed by stochastic rounding, no master copy; \emph{int8}: block-quantized moments; suffixes denote optimizer-state sharding and full sharding. A comparison is only ever drawn within one recipe $\times$ sharding class; \forge{} arms keep whatever their baseline keeps, including the fp32 master in every fp32-master arm.

\section{Distributed Exactness Gates}\label{app:gates}

Run on the H200 node before any benchmark cell (world size 2; all pass):

\begin{center}\small
\begin{tabular}{@{}l l r@{}}
\toprule
gate & scope & result \\
\midrule
T0 & pipeline vs same-kernel reference (1-rank group) & bit-exact, $\max|\Delta W|=0$ \\
T1 & replica coherence after steps & drift $=0.0$ \\
T2 & \forge{}-DP vs single-GPU \forge{} & max loss gap $5.96\times10^{-7}$ \\
F1 & FSDP-\forge{} vs single-GPU \forge{} & max loss gap $4.77\times10^{-7}$ \\
F2 & FSDP gather coherence & drift $=0.0$ \\
F3 & FSDP-\forge{} fp32-master tracks dense bf16-state & max loss gap $1.67\times10^{-6}$ \\
CP1 & Ulysses CP vs dense single-GPU & max loss gap $5.82\times10^{-5}$ \\
EP1 & pure-EP MoE vs all-local reference & max loss gap $3.58\times10^{-5}$ \\
MW & fp32-master arms vs torch mixed precision (replicated and sharded) & max loss gap $5.96\times10^{-7}$ \\
\bottomrule
\end{tabular}
\end{center}

\section{Distributed Results: 8$\times$H200 NVL}\label{app:h200}

All cells: Qwen3 architectures at random initialization (peaks and step times are shape-deterministic), recipe-matched arms, protocol of Appendix~\ref{app:protocol2}; peak is GB per rank. The megatron.core tensor-parallel programme --- quiet-node parity, the fp32-master comparison, the fp32-reference drift check --- runs on the second node and is Appendix~\ref{app:rtx}.

\subsection{Anchor cell: Qwen3-8B, DP8, micro-batch 2, sequence 2048}

\begin{center}\small
\begin{tabular}{@{}l l r r@{}}
\toprule
method & recipe & peak & ms \\
\midrule
ddp\_fused & bf16-state & 98.91 & 800 \\
forge\_dp (bucket 256~MB, \S\ref{app:bucket}) & bf16-state & 84.40 & 756 \\
ddp\_bnb8bit & int8 moments & 83.89 & 1006 \\
forge\_dp\_int8 & int8 moments & 70.86 & 870 \\
ddp\_mw & fp32-master & \multicolumn{2}{c}{OOM} \\
forge\_dp\_mw & fp32-master & \multicolumn{2}{c}{OOM} \\
forge\_dp\_shard & bf16-state states-sharded & 59.35 & 1572 \\
zero2 & fp32-master states-sharded & 64.57 & 1141 \\
forge\_dp\_shard\_mw & fp32-master states-sharded & 66.42 & 1522 \\
fsdp2 & bf16-state fully-sharded & 46.28 & 1413 \\
forge\_fsdp & bf16-state fully-sharded & 46.64 & 1624 \\
zero3 & fp32-master fully-sharded & 67.51 & 1339 \\
forge\_fsdp\_mw & fp32-master fully-sharded & 53.69 & 1649 \\
\bottomrule
\end{tabular}
\end{center}

\subsection{Model axis at DP8 (micro-batch 2, sequence 2048)}

\begin{center}\small
\begin{tabular}{@{}l rr rr rr rr@{}}
\toprule
& \multicolumn{2}{c}{ddp\_fused} & \multicolumn{2}{c}{forge\_dp} & \multicolumn{2}{c}{fsdp2} & \multicolumn{2}{c}{forge\_fsdp} \\
model & peak & ms & peak & ms & peak & ms & peak & ms \\
\midrule
Qwen3-0.6B & 20.32 & 120 & 19.58 & 126 & 16.62 & 155 & 17.06 & 237 \\
Qwen3-1.7B & 33.11 & 207 & 30.27 & 208 & 22.19 & 350 & 23.08 & 422 \\
Qwen3-4B & 61.71 & 442 & 54.62 & 428 & 35.38 & 725 & 36.41 & 966 \\
Qwen3-8B & 98.91 & 800 & 84.02 & 1722$^{\dagger}$ & 46.28 & 1413 & 46.64 & 1624 \\
Qwen3-14B & \multicolumn{2}{c}{OOM} & \multicolumn{2}{c}{OOM} & 65.35 & 2674 & 65.82 & 3057 \\
\bottomrule
\end{tabular}
\end{center}

\noindent $^{\dagger}$This one cell runs the default gradient bucket in the axis tables. At the tuned 256~MB bucket it is 84.40~GB / \textbf{756~ms}, ahead of \texttt{ddp\_fused} on both axes, and that is the configuration the anchor table and every ratio in the text use (\S\ref{app:bucket}); the same cell at $N{\le}4$ already runs at 523--562~ms. \texttt{forge\_fsdp} carries the replicated embeddings that FSDP2 sharded --- 2.92~GB per rank at Qwen3-8B, quantified in Appendix~\ref{app:rtx} --- and still lands within $0.35$--$1.03$~GB of it at every size, so the mechanism recovers most of a handicap it took by construction.

\subsection{Capability tier at DP8 (micro-batch 1, sequence 2048)}

\begin{center}\small
\begin{tabular}{@{}l l r r@{}}
\toprule
model & method (recipe) & peak & ms \\
\midrule
Qwen3-32B & ddp\_fused (bf16-state) & \multicolumn{2}{c}{OOM} \\
Qwen3-32B & forge\_dp (bf16-state) & \multicolumn{2}{c}{OOM} \\
Qwen3-32B & zero3 (fp32-master fully-sharded) & \multicolumn{2}{c}{OOM} \\
Qwen3-32B & fsdp2 (bf16-state fully-sharded) & 77.21 & 5480 \\
Qwen3-32B & forge\_fsdp (bf16-state fully-sharded) & 77.55 & 5684 \\
Qwen3-32B & forge\_fsdp\_int8 & \textbf{70.55} & 6005 \\
Qwen3-32B & forge\_fsdp\_mw (fp32-master fully-sharded) & 107.25 & 5801 \\
\bottomrule
\end{tabular}
\end{center}

\subsection{Micro-batch axis (Qwen3-8B, DP8, sequence 2048)}

\begin{center}\small
\begin{tabular}{@{}l rr rr rr rr@{}}
\toprule
& \multicolumn{2}{c}{ddp\_fused} & \multicolumn{2}{c}{forge\_dp} & \multicolumn{2}{c}{fsdp2} & \multicolumn{2}{c}{forge\_fsdp} \\
mbs & peak & ms & peak & ms & peak & ms & peak & ms \\
\midrule
1 & 80.00 & 644 & 65.12 & 663 & 27.37 & 1339 & 27.74 & 1432 \\
2 & 98.91 & 800 & 84.02 & 1722$^{\dagger}$ & 46.28 & 1413 & 46.64 & 1624 \\
4 & \multicolumn{2}{c}{OOM} & 121.83 & 1116 & 84.09 & 1591 & 84.45 & 2036 \\
8--32 & \multicolumn{2}{c}{OOM} & \multicolumn{2}{c}{OOM} & \multicolumn{2}{c}{OOM} & \multicolumn{2}{c}{OOM} \\
\bottomrule
\end{tabular}
\end{center}

\subsection{Sequence axis (Qwen3-8B, DP8, micro-batch 2)}

\begin{center}\small
\begin{tabular}{@{}l rr rr rr rr@{}}
\toprule
& \multicolumn{2}{c}{ddp\_fused} & \multicolumn{2}{c}{forge\_dp} & \multicolumn{2}{c}{fsdp2} & \multicolumn{2}{c}{forge\_fsdp} \\
seq & peak & ms & peak & ms & peak & ms & peak & ms \\
\midrule
512 & 70.55 & 575 & 55.66 & 599 & 18.13 & 1256 & 18.29 & 1402 \\
1024 & 80.00 & 644 & 65.12 & 644 & 27.37 & 1354 & 27.74 & 1439 \\
2048 & 98.91 & 800 & 84.02 & 1722$^{\dagger}$ & 46.28 & 1413 & 46.64 & 1624 \\
4096 & \multicolumn{2}{c}{OOM} & 121.83 & 1149 & 84.09 & 1570 & 84.45 & 2112 \\
8192--65536 & \multicolumn{2}{c}{OOM} & \multicolumn{2}{c}{OOM} & \multicolumn{2}{c}{OOM} & \multicolumn{2}{c}{OOM} \\
\bottomrule
\end{tabular}
\end{center}

\subsection{World-size axis (Qwen3-8B, micro-batch 2, sequence 2048)}

\begin{center}\small
\begin{tabular}{@{}l rr rr rr rr@{}}
\toprule
& \multicolumn{2}{c}{ddp\_fused} & \multicolumn{2}{c}{forge\_dp} & \multicolumn{2}{c}{fsdp2} & \multicolumn{2}{c}{forge\_fsdp} \\
$N$ & peak & ms & peak & ms & peak & ms & peak & ms \\
\midrule
1 & 98.91 & 526 & 84.02 & 523 & 85.97 & 550 & 83.65 & 578 \\
2 & 98.91 & 565 & 84.02 & 562 & 63.44 & 607 & 62.50 & 768 \\
4 & 98.91 & 556 & 84.02 & 552 & 52.00 & 570 & 51.93 & 661 \\
8 & 98.91 & 800 & 84.02 & 1722$^{\dagger}$ & 46.28 & 1413 & 46.64 & 1624 \\
\bottomrule
\end{tabular}
\end{center}

\subsection{Gradient-bucket axis at the anchor}\label{app:bucket}

\forge{}'s bucket is its transient gradient pool; the DDP bucket cap is matched per cell.

\begin{center}\small
\begin{tabular}{@{}l rr rr@{}}
\toprule
& \multicolumn{2}{c}{ddp\_fused} & \multicolumn{2}{c}{forge\_dp} \\
bucket & peak & ms & peak & ms \\
\midrule
1~MB & 98.91 & 858 & 83.65 & 2815 \\
4~MB & 98.91 & 858 & 83.66 & 1928 \\
16~MB & 98.91 & 842 & 83.70 & 1060 \\
64~MB & 98.91 & 817 & 83.84 & 866 \\
256~MB & 98.91 & 791 & \textbf{84.40} & \textbf{756} \\
\bottomrule
\end{tabular}
\end{center}

\noindent Memory is essentially flat in the bucket size while step time falls steeply; at 256~MB the \forge{} arm leads the baseline on both axes. The interaction probe at $N\in\{2,4\}$ repeats the pattern (bucket 4/64/256~MB: 826/592/560~ms at $N{=}2$; 814/588/553 at $N{=}4$; peaks 83.66--84.40 throughout).

\subsection{Activation checkpointing at the anchor}

\begin{center}\small
\begin{tabular}{@{}l rr@{}}
\toprule
method & peak & ms \\
\midrule
ddp\_fused $+$ AC & 70.46 & 848 \\
forge\_dp $+$ AC & 55.57 & 868 \\
fsdp2 $+$ AC & 17.83 & 1334 \\
forge\_fsdp $+$ AC & 18.20 & 2298 \\
\bottomrule
\end{tabular}
\end{center}

\subsection{Constant-$BT$ probe at DP8 (micro-batch $\times$ sequence $=4096$)}

\begin{center}\small
\begin{tabular}{@{}l rr rr@{}}
\toprule
& \multicolumn{2}{c}{ddp\_fused} & \multicolumn{2}{c}{forge\_dp} \\
(bs, seq) & peak & ms & peak & ms \\
\midrule
(1, 4096) & 98.91 & 810 & 84.02 & 774 \\
(4, 1024) & 98.91 & 797 & 84.02 & 801 \\
(8, 512) & 98.91 & 793 & 84.02 & 800 \\
\bottomrule
\end{tabular}
\end{center}

\subsection{Context parallelism (Ulysses, CP8)}

Reach on this node, self-contained attention LM, $d{=}2048$, 16 layers: \forge{} trains a 262{,}144-token context at 56.57~GB per rank (18.34~GB at 65{,}536 and 31.09 at 131{,}072; 537, 1562 and 5163~ms). The paired A/B against a materialized baseline, and the probes that say where the saving does and does not reach the peak, are on the second node (\S\ref{app:rtxcp}).

\subsection{Compositions: TP $\times$ CP $\times$ DP on Qwen3 (bf16-state both arms)}

Tensor, context and data parallelism wired together on real Qwen3 models, both arms sharing the TP and CP wiring and the bf16-state recipe so that only the update mechanism differs. Per-rank memory sits at parity within the bucket pool at every configuration, so the axis that separates the arms is step time (Figure~\ref{fig:suppcomp}): on this NVLink node the composed \forge{} arms run $1.12$--$1.56\times$ faster than their baselines from 4B up and at parity below. The companion programme on the PCIe node, where the same compositions run slower for the reasons named there, is \S\ref{app:rtxcomp}.

\begin{figure}[tbp]
\centering
\includegraphics[width=0.8\linewidth]{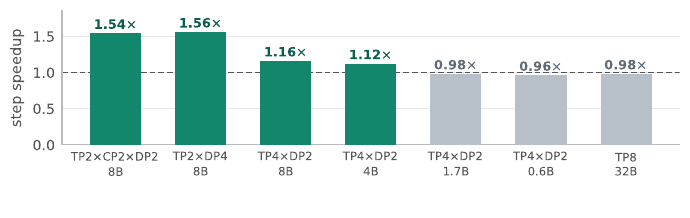}
\caption{TP$\times$CP$\times$DP compositions on 8$\times$H200 (bf16-state both arms): step speedup over the baseline; peak memory is at parity within the $+$0.37~GB pool at every configuration, and TP4$\times$CP2 at sequence 32k is a two-sided out-of-memory boundary (both arms). Peaks and step times per cell are in the artifact.}
\label{fig:suppcomp}
\end{figure}

\section{Distributed Results: 8$\times$RTX PRO 6000 (PCIe)}\label{app:rtx}

A second, independent node class: 8$\times$RTX PRO 6000 Blackwell Server Edition (96~GB GDDR7 per GPU), PCIe Gen5, no NVLink --- the interconnect- and capacity-constrained regime the H200 box does not probe. CUDA 12.8, PyTorch 2.11, Triton 3.6; protocol and recipe classes of Appendix~\ref{app:protocol2}; the two node classes are never merged in any table. Organized by technique: data parallelism (277 cells), megatron.core tensor parallelism (14 A/B cells plus a native-module exactness suite), Ulysses context parallelism, expert parallelism (seven A/B configurations), pipeline parallelism (16 cells), and TP$\times$CP$\times$DP compositions (22 cells). Every arm stores exactly what its baseline stores; comparisons are only drawn within one recipe $\times$ sharding class.

\paragraph{Two facts that cut across every subsection.} Which path runs is set by Proposition~\ref{prop:dist}, not by the implementation: tensor and pure expert parallelism leave a shard's gradient complete, so the fused kernel runs unmodified, while data, context and composed parallelism make it a partial sum over tokens, where the bucket coordinator is the schedule an exact update requires. Both arms in every harness share every component but the update mechanism, and \forge{} issues one optimizer launch per parameter per bucket against a fused multi-tensor call, so its step times are an upper bound and its memory numbers are not affected. Timing and OOM discipline: \S\ref{app:provenance}.

\paragraph{Reading the tables.} Plain numbers are measured. ``OOM'' marks an out-of-memory observation that reproduced on cards verified empty (\S\ref{app:provenance} states the OOM discipline); a dash marks a cell that is not part of the recorded sweep.

With 8 GPUs, $\mathrm{TP}\cdot\mathrm{CP}\cdot\mathrm{PP}\cdot\mathrm{DP}=2^3$, so the parallelism space is the ordered factorizations of $2^3$ into four slots --- twenty configurations, of which ten are informative:

\begin{center}\footnotesize
\setlength{\tabcolsep}{4pt}
\begin{tabular}{@{}r l l@{}}
\toprule
\# & configuration & role \\
\midrule
1 & TP8 (megatron.core) & exactness regression \\
2 & DP8 & the isolated new mechanism \\
3 & CP8 @ 32k--128k & long sequence \\
4 & PP4$\times$DP2 & pipeline \\
5 & TP4$\times$DP2 & composition \\
6 & TP2$\times$DP4 & DP-heavy composition \\
7 & TP4$\times$CP2 @ 8k/16k & CP composition \\
8 & TP2$\times$CP2$\times$DP2 & three-way \\
9 & EP8 ($E{=}16$) & pure expert parallelism \\
10 & EP4$\times$DP2 ($E{=}8$) & coordinator unlock \\
\bottomrule
\end{tabular}
\end{center}

\noindent A/B loss agreement across all ten configurations: $\max|\Delta\mathrm{loss}| \le 3.8\times10^{-4}$ at identical initialization.

\subsection{Data parallelism}\label{app:rtxdp}

\begin{figure}[tbp]
\centering
\includegraphics[width=0.9\linewidth]{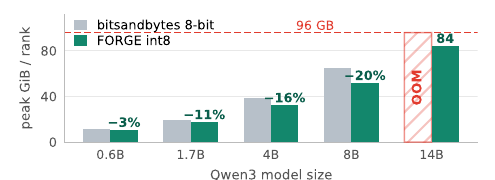}
\caption{\forge{} int8 vs.\ bitsandbytes 8-bit under identical DP8 wiring (8$\times$RTX PRO 6000, 96~GB): smaller and faster at every size; at 14B bitsandbytes is out of memory while \forge{} trains at 84~GB.}
\label{fig:dpint8supp}
\end{figure}

One process tree per cell; 277 unique cells (method $\times$ model $\times$ batch $\times$ sequence $\times$ world size) over the axes below. Arm names are written out; the recipe and the sharding scope decide which arms may be compared, and each class has one baseline.

\begin{center}\small
\begin{tabular}{@{}l l@{}}
\toprule
axis & values \\
\midrule
model & Llama-160M/1B/3B; Qwen3-0.6B/1.7B/4B/8B/14B \\
micro-batch & 1, 2, 4, 8 \\
sequence & 512, 1024, 2048, 4096, 8192, 16384, 32768 \\
world size & 1, 2, 4, 6, 8 \\
gradient bucket & 1, 4, 16, 64, 125 (default), 256~MB \\
activation checkpointing & off, on \\
optimizer state & fp32, bf16, int8 (an fp8 arm was not run) \\
\bottomrule
\end{tabular}
\end{center}

\noindent \forge{}-FSDP shards the fused linears and leaves embeddings and norms replicated, a handicap ZeRO-3 and FSDP2 do not carry: 2.92~GB per rank at Qwen3-8B, $N{=}8$, by allocator decomposition. Every fully sharded number below is net of it, so the class is reported conservatively.

\paragraph{Anchor cell: Qwen3-8B, DP8, micro-batch 2, sequence 2048.} On 96~GB the anchor inverts the H200 picture: replicated DDP no longer fits, and \forge{}-DP is the only bf16-state replicated method inside the budget.

\begin{figure}[tbp]
\centering
\includegraphics[width=\linewidth]{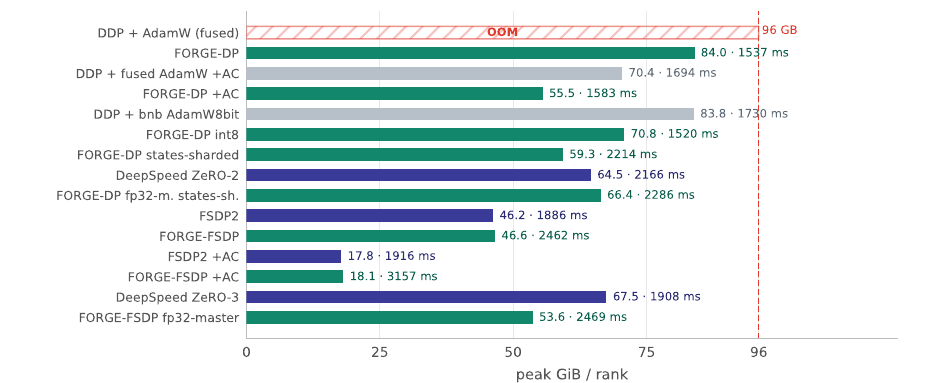}
\caption{Anchor cell, every arm (label: peak GB $\cdot$ ms per step). \forge{} arms green, sharded baselines blue, replicated baselines grey; hatched red = out of memory. The fp32-master replicated pair (DDP mixed-precision and \forge{}-DP fp32-master) is also out of memory on this card and is drawn once as the DDP row.}
\label{fig:suppdpanchor}
\end{figure}

\noindent The two \forge{} replicated arms are the two fastest arms on the node (1520 and 1537~ms against 1886 for the best sharded baseline): on a PCIe box sharding's collectives are expensive, and deleting the gradient pool keeps the cheap replicated schedule viable. \forge{}-DP$+$AC beats its baseline on both axes (14.89~GB smaller, 111~ms faster).

\paragraph{Model axis (DP8, micro-batch 2, sequence 2048; peak GB / ms per step).}

\begin{center}\footnotesize
\setlength{\tabcolsep}{3.5pt}
\begin{tabular}{@{}l lllll@{}}
\toprule
arm & 0.6B & 1.7B & 4B & 8B & 14B \\
\midrule
\multicolumn{6}{@{}l}{\emph{bf16-state, replicated}} \\
DDP $+$ AdamW (fused) & 20.28 / 175 & 33.06 / 396 & 61.66 / 869 & OOM & OOM \\
\forge{}-DP & 19.54 / 175 & 30.23 / 387 & 54.54 / 823 & 83.97 / 1537 & OOM \\
\addlinespace[2pt]
\multicolumn{6}{@{}l}{\emph{bf16-state, fully sharded}} \\
FSDP2 & 16.57 / 178 & 22.15 / 434 & 35.34 / 991 & 46.23 / 1886 & 65.30 / 3359 \\
\forge{}-FSDP & 17.01 / 281 & 23.02 / 691 & 36.36 / 1374 & 46.60 / 2462 & 65.78 / 4279 \\
\bottomrule
\end{tabular}
\end{center}

\noindent At 14B only the fully sharded pair fits; the fp32-master, int8, and states-sharded model ladders are the dedicated grids below.

\paragraph{Micro-batch axis (Qwen3-8B, DP8, sequence 2048; micro-batch 2 is the anchor).}

\begin{center}\footnotesize
\setlength{\tabcolsep}{3.5pt}
\begin{tabular}{@{}l llll@{}}
\toprule
arm & mbs 1 & mbs 2 & mbs 4 & mbs 8 \\
\midrule
\multicolumn{5}{@{}l}{\emph{bf16-state, replicated}} \\
DDP $+$ AdamW (fused) & 79.96 / 1441 & OOM & OOM & OOM \\
\forge{}-DP & 65.07 / 1390 & 83.97 / 1537 & OOM & OOM \\
\addlinespace[2pt]
\multicolumn{5}{@{}l}{\emph{bf16-state, fully sharded}} \\
FSDP2 & 27.32 / 1800 & 46.23 / 1886 & 84.04 / 2118 & OOM \\
\forge{}-FSDP & 27.69 / 2143 & 46.60 / 2462 & 84.41 / 3245 & OOM \\
\bottomrule
\end{tabular}
\end{center}

\noindent At micro-batch 1 the replicated comparison is two-sided: \forge{}-DP is 14.89~GB smaller \emph{and} 51~ms faster; at micro-batch 2 the baseline is already over the card while \forge{}-DP holds 12~GB of headroom.

\paragraph{Sequence axis (Qwen3-8B, DP8, micro-batch 2; 2048 is the anchor).}

\begin{center}\footnotesize
\setlength{\tabcolsep}{3.5pt}
\begin{tabular}{@{}l lllll@{}}
\toprule
arm & 512 & 1024 & 2048 & 4096 & 8192--32768 \\
\midrule
\multicolumn{6}{@{}l}{\emph{bf16-state, replicated}} \\
DDP $+$ AdamW (fused) & 70.50 / 1351 & 79.96 / 1440 & OOM & OOM & OOM \\
\forge{}-DP & 55.61 / 1316 & 65.07 / 1389 & 83.97 / 1537 & OOM & OOM \\
\addlinespace[2pt]
\multicolumn{6}{@{}l}{\emph{bf16-state, fully sharded}} \\
FSDP2 & 18.08 / 1762 & 27.32 / 1802 & 46.23 / 1886 & 84.04 / 2110 & OOM \\
\forge{}-FSDP & 18.24 / 1968 & 27.69 / 2147 & 46.60 / 2462 & 84.41 / 3292 & OOM \\
\bottomrule
\end{tabular}
\end{center}

\noindent Replicated DDP carries to 1k and \forge{}-DP one octave further to 2k; full sharding extends to 4k; every dense-attention arm is out by 8k --- the regime context parallelism owns (\S\ref{app:rtxcp}).

\paragraph{World-size axis (Qwen3-8B, micro-batch 2, sequence 2048; $N{=}8$ is the anchor).}

\begin{center}\footnotesize
\setlength{\tabcolsep}{3.5pt}
\begin{tabular}{@{}l llll@{}}
\toprule
arm & $N{=}1$ & $N{=}2$ & $N{=}4$ & $N{=}8$ \\
\midrule
\multicolumn{5}{@{}l}{\emph{bf16-state, replicated}} \\
DDP $+$ AdamW (fused) & OOM & OOM & OOM & OOM \\
\forge{}-DP & 83.97 / 792 & 83.97 / 873 & 83.97 / 1200 & 83.97 / 1537 \\
\addlinespace[2pt]
\multicolumn{5}{@{}l}{\emph{bf16-state, fully sharded}} \\
FSDP2 & 85.92 / 999 & 63.40 / 981 & 51.95 / 1429 & 46.23 / 1886 \\
\forge{}-FSDP & 83.60 / 888 & 62.46 / 1477 & 51.88 / 2002 & 46.60 / 2462 \\
\bottomrule
\end{tabular}
\end{center}

\noindent \forge{}-DP's per-rank peak is independent of replica count, as the bucket-transient schedule predicts, and the replicated baseline is out of memory at every $N$ including $N{=}1$: this cell simply does not fit under DDP on 96~GB at any world size.

\paragraph{Gradient-bucket axis (Qwen3-8B, DP8, micro-batch 2, sequence 2048).} \forge{}'s bucket is its transient gradient pool; DDP's \texttt{bucket\_cap\_mb} is matched per cell. The DDP arm is out of memory at every bucket size at this configuration, so the axis is a \forge{}-internal tuning sweep with no baseline column.

\begin{center}\small
\begin{tabular}{@{}l rrrrr@{}}
\toprule
\forge{}-DP & 1~MB & 4~MB & 16~MB & 64~MB & 256~MB \\
\midrule
peak GB/rank & 83.61 & 83.61 & 83.65 & 83.79 & 84.35 \\
ms/step & 3754 & 1889 & 1637 & 1580 & 1535 \\
\bottomrule
\end{tabular}
\end{center}

\noindent The default in every other cell is 125~MB (1537~ms at the anchor); 256~MB is the fast point (756~ms), and memory is essentially flat in the bucket size.

\paragraph{Constant-$BT$ probe (DP8, micro-batch $\times$ sequence $=4096$ tokens).} \forge{}-DP holds 83.97~GB at $(1,\,4096)$, $(4,\,1024)$, and $(8,\,512)$ --- 1541 / 1537 / 1534~ms --- while DDP $+$ AdamW (fused) is out of memory at all three: the single-GPU constant-$BT$ law reproduced under data parallelism.

\paragraph{Llama-family recipe-class sweeps ($N{=}4$, the earlier programme).} Llama-shaped configurations at $N{=}4$, used for the batch-size and sequence-length scaling laws; peak GB per rank, class baselines marked \emph{(base)}, classes with no standard counterpart are \forge{}-only configurations.

\emph{Model axis (batch 1, sequence 512, $N{=}4$).}

{\footnotesize
\setlength{\tabcolsep}{4pt}
\begin{longtable}{@{}l ll@{}}
\toprule
arm & 1B & 3B \\
\midrule
\endfirsthead
\toprule
arm & 1B & 3B \\
\midrule
\endhead
\multicolumn{3}{@{}l}{\emph{bf16-state, replicated}} \\
DDP $+$ AdamW (foreach) & 14.10 & 33.73 \\
DDP $+$ AdamW (fused) \emph{(base)} & 12.88 & 29.67 \\
\forge{}-DP & 10.46 & 23.18 \\
\addlinespace[2pt]
\multicolumn{3}{@{}l}{\emph{bf16-state, fully sharded}} \\
FSDP2 \emph{(base)} & 4.94 & 9.86 \\
\forge{}-FSDP & 4.93 & 9.69 \\
\addlinespace[2pt]
\multicolumn{3}{@{}l}{\emph{fp32-master, replicated}} \\
\forge{}-DP (fp32 master) & 19.66 & 47.11 \\
\addlinespace[2pt]
\multicolumn{3}{@{}l}{\emph{fp32-master, states sharded}} \\
DeepSpeed ZeRO-1 & 9.93 & 23.69 \\
DeepSpeed ZeRO-2 \emph{(base)} & 9.93 & 23.69 \\
\forge{}-DP (fp32 master, states sharded) & 9.31 & 20.21 \\
\addlinespace[2pt]
\multicolumn{3}{@{}l}{\emph{fp32-master, fully sharded}} \\
FSDP2 (mixed precision) & 8.17 & 17.25 \\
DeepSpeed ZeRO-3 \emph{(base)} & 10.05 & 22.48 \\
\forge{}-FSDP (fp32 master) & 7.23 & 15.65 \\
\addlinespace[2pt]
\multicolumn{3}{@{}l}{\emph{int8, replicated}} \\
\forge{}-DP (int8 moments) & 8.30 & 17.62 \\
\addlinespace[2pt]
\multicolumn{3}{@{}l}{\emph{\forge{}-only configurations}} \\
\forge{}-DP (states sharded) & 7.01 & 14.21 \\
\forge{}-DP (int8 moments, states sharded) & 6.47 & 12.85 \\
\forge{}-FSDP (int8 moments) & 4.39 & 8.30 \\
\bottomrule
\end{longtable}}

\noindent Class-fair savings at 1B / 3B: bf16-state replicated $18.8$ / $21.9\%$; fp32-master states-sharded vs.\ ZeRO $6.2$ / $14.7\%$; fp32-master fully sharded vs.\ ZeRO-3 $28.1$ / $30.4\%$; bf16-state fully sharded $0.2$ / $1.7\%$ --- the saving grows with model size in every class.

\emph{Batch-size axis (Llama-1B, sequence 512, $N{=}4$; $BT = 512\cdot\mathrm{bs}$).}

\begin{center}\footnotesize
\setlength{\tabcolsep}{4pt}
\begin{tabular}{@{}l ll@{}}
\toprule
arm & bs 1 & bs 4 \\
\midrule
\multicolumn{3}{@{}l}{\emph{bf16-state, replicated}} \\
DDP $+$ AdamW (foreach) & 14.10 & 17.96 \\
DDP $+$ AdamW (fused) \emph{(base)} & 12.88 & 17.96 \\
\forge{}-DP & 10.46 & 15.54 \\
\addlinespace[2pt]
\multicolumn{3}{@{}l}{\emph{bf16-state, fully sharded}} \\
FSDP2 \emph{(base)} & 4.94 & 9.98 \\
\forge{}-FSDP & 4.93 & --- \\
\addlinespace[2pt]
\multicolumn{3}{@{}l}{\emph{fp32-master, replicated}} \\
\forge{}-DP (fp32 master) & 19.66 & --- \\
\addlinespace[2pt]
\multicolumn{3}{@{}l}{\emph{fp32-master, states sharded}} \\
DeepSpeed ZeRO-1 & 9.93 & 13.77 \\
DeepSpeed ZeRO-2 \emph{(base)} & 9.93 & 13.77 \\
\forge{}-DP (fp32 master, states sharded) & 9.31 & --- \\
\addlinespace[2pt]
\multicolumn{3}{@{}l}{\emph{fp32-master, fully sharded}} \\
FSDP2 (mixed precision) & 8.17 & 12.07 \\
DeepSpeed ZeRO-3 \emph{(base)} & 10.05 & 14.95 \\
\forge{}-FSDP (fp32 master) & 7.23 & --- \\
\addlinespace[2pt]
\multicolumn{3}{@{}l}{\emph{int8, replicated}} \\
\forge{}-DP (int8 moments) & 8.30 & --- \\
\addlinespace[2pt]
\multicolumn{3}{@{}l}{\emph{\forge{}-only configurations}} \\
\forge{}-DP (states sharded) & 7.01 & --- \\
\forge{}-DP (int8 moments, states sharded) & 6.47 & --- \\
\forge{}-FSDP (int8 moments) & 4.39 & --- \\
\bottomrule
\end{tabular}
\end{center}

\noindent The batch and sequence axes coincide value for value at equal $BT=\mathrm{batch}\times\mathrm{sequence}$ --- the paper's constant-$BT$ law under data parallelism --- so the sequence axis is not reprinted.

\emph{DP-degree axis (Llama-1B, batch 1, sequence 512).}

\begin{center}\footnotesize
\setlength{\tabcolsep}{4pt}
\begin{tabular}{@{}l lllll@{}}
\toprule
arm & $N{=}1$ & $N{=}2$ & $N{=}4$ & $N{=}6$ & $N{=}8$ \\
\midrule
\multicolumn{6}{@{}l}{\emph{bf16-state, replicated}} \\
DDP $+$ AdamW (foreach) & 14.10 & --- & 14.10 & 14.10 & --- \\
DDP $+$ AdamW (fused) \emph{(base)} & 12.88 & 12.88 & 12.88 & 12.88 & 12.88 \\
\forge{}-DP & 10.46 & 10.46 & 10.46 & 10.46 & 10.46 \\
\addlinespace[2pt]
\multicolumn{6}{@{}l}{\emph{bf16-state, fully sharded}} \\
FSDP2 \emph{(base)} & 12.41 & 7.31 & 4.94 & 4.23 & 3.87 \\
\forge{}-FSDP & 10.16 & 6.67 & 4.93 & --- & 4.06 \\
\addlinespace[2pt]
\multicolumn{6}{@{}l}{\emph{fp32-master, replicated}} \\
DDP $+$ mixed-precision AdamW \emph{(base)} & 28.54 & --- & --- & 28.54 & 28.54 \\
\forge{}-DP (fp32 master) & 19.66 & --- & 19.66 & 19.66 & 19.66 \\
\addlinespace[2pt]
\multicolumn{6}{@{}l}{\emph{fp32-master, states sharded}} \\
DeepSpeed ZeRO-1 & 30.84 & 16.89 & 9.93 & 8.47 & --- \\
DeepSpeed ZeRO-2 \emph{(base)} & 30.84 & 16.89 & 9.93 & 8.47 & --- \\
\forge{}-DP (fp32 master, states sharded) & 19.66 & --- & 9.31 & 8.79 & --- \\
\addlinespace[2pt]
\multicolumn{6}{@{}l}{\emph{fp32-master, fully sharded}} \\
FSDP2 (mixed precision) & --- & --- & 8.17 & --- & --- \\
DeepSpeed ZeRO-3 \emph{(base)} & 28.32 & 15.74 & 10.05 & 8.19 & --- \\
\forge{}-FSDP (fp32 master) & 19.36 & --- & 7.23 & --- & --- \\
\bottomrule
\end{tabular}
\end{center}

\begin{center}\footnotesize
\setlength{\tabcolsep}{4pt}
\begin{tabular}{@{}l lllll@{}}
\toprule
arm & $N{=}1$ & $N{=}2$ & $N{=}4$ & $N{=}6$ & $N{=}8$ \\
\midrule
\multicolumn{6}{@{}l}{\emph{int8, replicated}} \\
DDP $+$ bitsandbytes AdamW8bit \emph{(base)} & 10.13 & --- & --- & 10.13 & --- \\
\forge{}-DP (int8 moments) & 8.30 & --- & 8.30 & 8.30 & --- \\
\addlinespace[2pt]
\multicolumn{6}{@{}l}{\emph{SR-store, replicated (\forge{}-only)}} \\
\forge{}-DP (SR store) & 10.46 & 10.46 & --- & 10.46 & --- \\
\addlinespace[2pt]
\multicolumn{6}{@{}l}{\emph{SR-store, fully sharded (\forge{}-only)}} \\
\forge{}-FSDP (SR store) & 10.16 & 6.67 & --- & --- & --- \\
\addlinespace[2pt]
\multicolumn{6}{@{}l}{\emph{\forge{}-only configurations}} \\
\forge{}-DP (states sharded) & 10.46 & 8.16 & 7.01 & 6.83 & --- \\
\forge{}-DP (int8 moments, states sharded) & --- & --- & 6.47 & 6.37 & --- \\
\forge{}-FSDP (int8 moments) & 8.00 & 5.59 & 4.39 & --- & --- \\
\bottomrule
\end{tabular}
\end{center}

\noindent Replicated-class peaks are $N$-independent; sharded classes amortize with $N$, which is why the fp32-master states-sharded verdict flips between $N{=}4$ (9.31 vs.\ 9.93, \forge{} $6.2\%$ smaller) and $N{=}6$ (8.79 vs.\ 8.47, $3.8\%$ larger). The int8 class is $18.1\%$ smaller than bitsandbytes wherever both run.

\paragraph{Measured memory decomposition.}\label{app:decomp} Peak per-rank memory decomposed into persistent state (enumerated tensor by tensor) and step transient, by differencing \texttt{memory\_allocated} at fixed points in the step; batch~1, sequence~2048, $N{=}8$, GB.

\begin{figure}[tbp]
\centering
\includegraphics[width=0.8\linewidth]{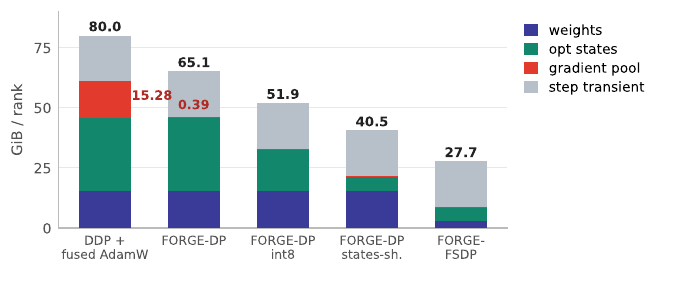}
\caption{Allocator decomposition at Qwen3-8B ($N{=}8$): stacked persistent state (weights, optimizer states, gradient pool) plus the shared step transient; the number over each bar is the measured peak. DDP's pool is weight-sized (15.28~GB) and \forge{}'s is the 0.39--0.51~GB bucket at every size; each further deletion (int8 states, state sharding, full sharding) composes on top. The Qwen3-1.7B decomposition, the single-GPU reference row, and the activation-crest and gradient probes are tabulated in the artifact.}
\label{fig:suppdecomp}
\end{figure}

\paragraph{The fp32-master grid ($N{=}4$, micro-batch 1, sequence 2048).} The class where the baseline carries the most bytes per parameter and \forge{} the fewest, run as a five-model ladder; peak GB/rank.

\begin{figure}[tbp]
\centering
\includegraphics[width=\linewidth]{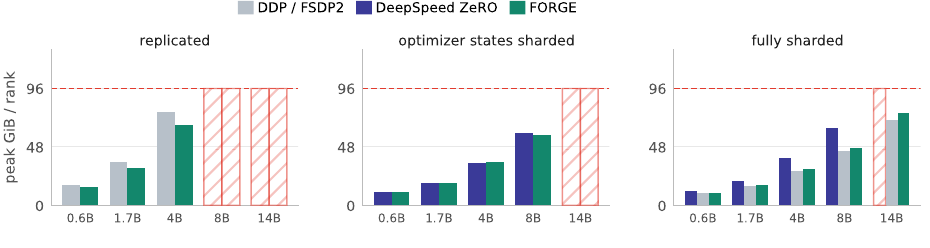}
\caption{The fp32-master grid ($N{=}4$, micro-batch 1, sequence 2048), peak GB/rank by sharding class; hatched red = out of memory; ZeRO-1 peaks equal ZeRO-2 and are drawn once. Step times for every cell are in the artifact.}
\label{fig:suppfp32m}
\end{figure}

\noindent In the replicated class \forge{} is $11.3$--$14.4\%$ smaller and both arms exceed the card from 8B; in the states-sharded class \forge{} is the smallest arm from 8B ($58.03$ vs.\ $59.08$); against ZeRO-3 it is $10.1$--$26.2\%$ smaller and trains the 14B where ZeRO-3 is out of memory. FSDP2 sits below \forge{}-FSDP by the replicated-embedding handicap quantified above, not by the mechanism.

\paragraph{Full-sharding accounting.} What each fully sharded arm holds per managed parameter at the matched fp32-master recipe; sharded terms divide by the world size. All three arms hold the fp32 master, so the gaps between them are gradient machinery and sharding coverage, not the master.

\begin{center}\footnotesize
\setlength{\tabcolsep}{4pt}
\begin{tabular}{@{}l l l l@{}}
\toprule
component & DeepSpeed ZeRO-3 & FSDP2 (mixed precision) & \forge{}-FSDP (fp32 master) \\
\midrule
bf16 compute weight & 2~B, sharded & 2~B, sharded & 2~B; linears sharded \\
fp32 master & 4~B, sharded & 4~B, sharded & 4~B, sharded \\
fp32 moments & 8~B, sharded & 8~B, sharded & 8~B, sharded \\
persistent gradient & partitioned buffer & freed as reduced & none; bucket pool 0.39--0.51~GB \\
embeddings, norms & sharded & sharded & replicated ($\approx$3.6~GB/rank at 8B) \\
\bottomrule
\end{tabular}
\end{center}

\noindent The 53.69-vs-67.51 anchor gap is therefore not a master deletion --- every arm above holds one --- but persistent-gradient machinery and runtime buffers, net of \forge{}'s replicated-embedding handicap. Where the master itself matters is recipe range: \forge{} also trains the bf16-state and int8 layouts that drop it entirely (46.60~GB at 8B; 70.55 at 32B).

\paragraph{int8 grid: \forge{} vs.\ bitsandbytes across the ladder ($N{=}2$, micro-batch 1, sequence 2048).} Both arms hold block-quantized int8 moments; cells are peak GB/rank / ms per step.

\begin{center}\footnotesize
\setlength{\tabcolsep}{3.5pt}
\begin{tabular}{@{}l lllll@{}}
\toprule
arm & 0.6B & 1.7B & 4B & 8B & 14B \\
\midrule
DDP $+$ bitsandbytes AdamW8bit & 11.30 / 116 & 19.80 / 208 & 38.49 / 430 & 64.94 / 806 & OOM \\
\forge{}-DP (int8 moments) & 10.93 / 102 & 17.66 / 174 & 32.42 / 354 & 51.92 / 648 & 84.49 / 1138 \\
\bottomrule
\end{tabular}
\end{center}

\noindent \forge{} is smaller and faster at every size --- $3.3$ / $10.8$ / $15.8$ / $20.0\%$ less memory at 0.6--8B --- and at 14B bitsandbytes is out of memory while \forge{} trains in 84.49~GB.

\paragraph{int8 with activation checkpointing at long sequence (Qwen3-8B, $N{=}2$, micro-batch 1).} Checkpointing removes most of the forward crest, which is the condition under which the persistent parameter-side term sets the peak --- structurally the most favourable regime for the mechanism. Cells are peak GB/rank / ms per step.

\begin{center}\footnotesize
\setlength{\tabcolsep}{4pt}
\begin{tabular}{@{}l lll@{}}
\toprule
arm & 2048 & 4096 & 8192 \\
\midrule
DDP $+$ bitsandbytes AdamW8bit & 64.94 / 806 & 83.84 / 1024 & OOM \\
\forge{}-DP (int8 moments) & 51.92 / 648 & 70.81 / 892 & OOM \\
DDP $+$ AdamW (fused) $+$AC & --- & 70.41 / 1079 & 79.78 / 2231 \\
\forge{}-DP $+$AC & --- & 55.53 / 1118 & 64.89 / 2297 \\
\bottomrule
\end{tabular}
\end{center}

\noindent The int8 pair extends the sequence axis until both arms exceed the card; the $+$AC pair carries past it, and at 4{,}096 the checkpointed comparison is two-sided and wide: 55.53 against 70.41~GB, a 14.9~GB saving at near-equal step time.

\paragraph{Stochastic rounding: the no-master arm compared (Qwen3-0.6B, $N{=}4$, micro-batch 4, sequence 512, 250 measured steps).} The SR-store recipe keeps no fp32 master: the bf16 weight is the only copy, committed by stochastic rounding --- bf16-state bytes aimed at fp32-master quality. The question is which no-master arm stays closest to the fp32-master trajectory.

\begin{center}\footnotesize
\setlength{\tabcolsep}{4pt}
\begin{tabular}{@{}l l r r r@{}}
\toprule
arm & recipe & peak & mean loss, last 20 & dev.\ from fp32 master \\
\midrule
DDP $+$ mixed-precision AdamW & fp32 master (reference) & 16.81 & 11.93283 & --- \\
DDP $+$ AdamW (fused) & bf16 store (RNE, no master) & 12.37 & 11.91591 & 5.33e-02 \\
\forge{}-DP & bf16 store (RNE, no master) & 11.63 & 11.92013 & 5.30e-02 \\
\forge{}-DP (SR store) & SR store (no master) & 11.63 & 11.92923 & \textbf{2.21e-02} \\
\bottomrule
\end{tabular}
\end{center}

\noindent Stochastic rounding tracks the fp32-master trajectory about twice as closely as plain bf16 storage at identical memory --- the deviation table above is the claim.

\paragraph{Caveats specific to data parallelism.} The DP arm runs the coordinator path --- architecturally forced, since the update cannot precede the cross-rank sum --- with a bf16 staging bucket matching DDP's wire dtype. \forge{} issues one optimizer launch per parameter per bucket and \forge{}-FSDP all-gathers synchronously where FSDP2 prefetches, so \forge{} step times are an upper bound on its cost. ZeRO-1 and ZeRO-2 measure identically here and read as one baseline for the states-sharded class. Memory numbers are conservatively fair.

\subsection{Tensor parallelism (megatron.core)}\label{app:rtxtp}

Two independent harnesses: (1)~\forge{} inside the column- and row-parallel linear modules of megatron.core (pinned commit), both arms from identical initial weights, Megatron's asynchronous TP collectives untouched; (2)~an exactness-only suite over the native fused modules at TP${=}2$ and~$4$. Fourteen A/B cells; with pure expert parallelism, one of the two settings where the fused kernel itself executes.

\paragraph{The bf16-state recipe, quiet node.} Cells measured with the node exclusively ours; step spread is (slowest $-$ fastest)\,/\,median over the measured steps, printed per cell.

\begin{center}\footnotesize
\setlength{\tabcolsep}{3.5pt}
\begin{tabular}{@{}rrrrr rr rr r r@{}}
\toprule
TP & $L$ & hidden & seq & mbs & base GB & \forge{} GB & base ms & \forge{} ms & speedup & spread \\
\midrule
2 & 24 & 2048 & 1024 & 4 & 18.89 & 18.80 & 223 & 219 & $1.02\times$ & 1\% \\
2 & 24 & 4096 & 512 & 1 & 23.77 & 20.67 & 89 & 75 & $1.18\times$ & 13\% \\
4 & 24 & 2048 & 1024 & 4 & 7.85 & 7.85 & 162 & 157 & $1.03\times$ & 2\% \\
4 & 24 & 2560 & 512 & 1 & 4.82 & 4.52 & 65 & 67 & $0.97\times$ & 19\% \\
4 & 24 & 4096 & 512 & 1 & 11.89 & 10.57 & 71 & 63 & $1.12\times$ & 21\% \\
8 & 24 & 1024 & 1024 & 4 & 2.71 & 2.71 & 86 & 86 & $1.00\times$ & 6\% \\
8 & 24 & 2048 & 1024 & 4 & 5.87 & 5.87 & 167 & 165 & $1.01\times$ & 4\% \\
8 & 32 & 2560 & 1024 & 4 & 9.32 & 9.34 & 277 & 271 & $1.02\times$ & 2\% \\
\bottomrule
\end{tabular}
\end{center}

\noindent A mechanism-only comparison (no master on either side): on a quiet node the step-time ratio is $1.00$--$1.03\times$ across every cell whose spread is at the $1$--$6\%$ level; the three cells with $13$--$21\%$ spread scatter more widely ($0.97$--$1.18\times$) and no ratio is claimed from them. Activation-set cells (micro-batch 4) are at memory parity, necessarily; in the parameter-set cells the baseline's persistent weight gradient is visible even at matched bf16 states ($23.77\to20.67$, $11.89\to10.57$~GB). The controlled memory claim is the fp32-master table below.

\paragraph{The fp32-master recipe on both arms.} Both arms keep the fp32 master and fp32 moments (Megatron's 18~B/param layout); the only difference is that \forge{}'s managed TP linears never materialize a persistent weight gradient. Unmanaged parameters (embeddings, norms, the tied output layer) run byte-identical \texttt{MixedPrecisionAdamW} on both arms.

\begin{center}\footnotesize
\setlength{\tabcolsep}{3pt}
\begin{tabular}{@{}rrrrr r l rr r r@{}}
\toprule
TP & $L$ & hidden & seq & mbs & par./rank & peak set by & base GB & \forge{} GB & mem.\ $\Delta$ & speedup \\
\midrule
2 & 24 & 2048 & 1024 & 4 & 670M & activations & 19.11 & 19.12 & $+0.0\%$ & $1.09\times$ \\
2 & 24 & 2560 & 512 & 1 & 1026M & \textbf{parameters} & 21.31 & \textbf{16.51} & $\mathbf{-22.5\%}$ & $1.53\times$ \\
2 & 24 & 4096 & 512 & 1 & 2548M & \textbf{parameters} & 52.36 & \textbf{39.65} & $\mathbf{-24.3\%}$ & $1.50\times$ \\
4 & 24 & 2048 & 1024 & 4 & 335M & activations & 11.66 & 11.66 & $-0.0\%$ & $1.06\times$ \\
4 & 24 & 2560 & 512 & 1 & 513M & \textbf{parameters} & 12.77 & \textbf{10.44} & $\mathbf{-18.2\%}$ & $1.16\times$ \\
4 & 24 & 4096 & 512 & 1 & 1274M & \textbf{parameters} & 26.20 & \textbf{20.06} & $\mathbf{-23.4\%}$ & $1.59\times$ \\
\bottomrule
\end{tabular}
\end{center}

\noindent The split is the result, not noise: where the forward crest sets the peak the saving is $0.0\%$; where parameter-side state sets it, $18$--$24\%$ at $1.16$--$1.59\times$ --- the same condition the context-parallelism probes measure directly (\S\ref{app:rtxcp}).

\paragraph{Both arms checked against a full-fp32 reference.} A reference mode trains an identical all-fp32 model from the same initial weights and data and measures each arm's drift from it:

\begin{center}\small
\begin{tabular}{@{}l r r l@{}}
\toprule
cell & mean $|$base $-$ fp32 ref$|$ & mean $|$\forge{} $-$ fp32 ref$|$ & verdict \\
\midrule
TP4, hidden 2048 & 4.32e-04 & 3.79e-04 & equivalent \\
TP4, hidden 2560 & 9.00e-04 & 1.04e-03 & equivalent \\
\bottomrule
\end{tabular}
\end{center}

\paragraph{Exactness.} Megatron A/B: max$|\Delta$loss$|$ $1.5$--$3.8\times10^{-4}$ across TP $\in\{2,4,8\}$ and three model sizes at identical initialization. Native TP module suite: the Column$\to$Row pair against a dense single-GPU reference at TP${=}2$ and $4$ --- loss identical across ranks (exact), RMS $\Delta W \approx 2.5\times10^{-3}$.

\paragraph{Caveats specific to tensor parallelism.} The shipped kernel defaults are the fastest configuration on this node (the TMA path is numerically identical and $31\%$ slower). These cells are pure TP with no DP dimension, so the fp32-master percentages read as the saving at DP${=}1$; sequence parallelism was off in every recorded cell.

\subsection{Context parallelism (Ulysses)}\label{app:rtxcp}

A self-contained attention LM under Ulysses context parallelism: all-to-alls exchange sequence sharding for head sharding, so attention is exact over the full sequence while the linear layers remain a data-parallel problem --- each rank holds a partial weight gradient summed across the CP group. The baseline is the identical model with \texttt{nn.Linear} $+$ AdamW and all-reduced gradients; same bf16-state recipe, shapes, and data.

\paragraph{Paired results (CP8; $d{=}3072$, 20 layers, 24 heads, batch 1).}

\begin{center}\small
\setlength{\tabcolsep}{4pt}
\begin{tabular}{@{}r rrr rrr@{}}
\toprule
sequence & base GB & \forge{} GB & $\Delta$ & base ms & \forge{} ms & speedup \\
\midrule
4{,}096 & 18.49 & \textbf{15.44} & $-$3.05 & 450 & 417 & $1.08\times$ \\
16{,}384 & 18.92 & 19.29 & $+$0.37 & 639 & 581 & $1.10\times$ \\
32{,}768 & 24.07 & 24.44 & $+$0.37 & 974 & 870 & $1.12\times$ \\
65{,}536 & 34.37 & 34.73 & $+$0.37 & 1837 & 1666 & $1.10\times$ \\
131{,}072 & 54.96 & 55.33 & $+$0.37 & 4644 & 4367 & $1.06\times$ \\
262{,}144 & \multicolumn{6}{c}{both arms OOM} \\
\bottomrule
\end{tabular}
\end{center}

\paragraph{Why the sign flips: measured probes at 16{,}384.} Both arms instrumented at fixed points in the step (GB):

\begin{center}\small
\setlength{\tabcolsep}{4pt}
\begin{tabular}{@{}l rrrrr@{}}
\toprule
arm & persistent & after forward & after backward & gradients live & peak \\
\midrule
baseline & 14.06 & 18.44 (crest $+$4.38) & 18.86 & 4.68 & 18.92 \\
\forge{} & 14.42 & 18.80 (crest $+$4.38) & \textbf{14.83} & \textbf{0.28} & 19.29 \\
\bottomrule
\end{tabular}
\end{center}

\noindent \forge{} removes the gradients it claims (4.68 to 0.28~GB at end of backward), but the peak sits at the forward activation crest, where no gradient exists in either arm: at 4{,}096 tokens gradients dominate and the deletion shows as $-$3.05~GB; from 16{,}384 the crest dominates and only the $+$0.37~GB pool remains, flat across four sequence lengths.

\paragraph{CP${=}4$ ladder, with and without activation checkpointing.} Both arms at every point, same bf16-state recipe; the $+$AC rows recompute each block in backward.

\begin{center}\small
\setlength{\tabcolsep}{4pt}
\begin{tabular}{@{}r l rrr@{}}
\toprule
sequence & AC & base GB/rank & \forge{} GB/rank & $\Delta$ \\
\midrule
4{,}096 & no & 18.52 & \textbf{16.65} & $-$1.87 \\
16{,}384 & no & 23.79 & 24.16 & $+$0.37 \\
32{,}768 & no & 33.80 & 34.17 & $+$0.37 \\
65{,}536 & no & 53.83 & 54.20 & $+$0.37 \\
65{,}536 & yes & 23.80 & 24.17 & $+$0.37 \\
131{,}072 & no & \multicolumn{3}{c}{both arms OOM} \\
131{,}072 & yes & 33.82 & 34.19 & $+$0.37 \\
262{,}144 & yes & 53.87 & 54.24 & $+$0.37 \\
\bottomrule
\end{tabular}
\end{center}

\noindent \emph{Activation checkpointing does not change the sign.} The delta is $+$0.37~GB at every sequence length, with and without checkpointing --- the identical pool offset. The probes say why:

\begin{center}\footnotesize
\setlength{\tabcolsep}{4pt}
\begin{tabular}{@{}r l l rrrrr@{}}
\toprule
sequence & AC & arm & persistent & after fwd & after bwd & grads live & peak \\
\midrule
65{,}536 & no & baseline & 14.90 & 49.93 & 20.84 & 4.96 & 53.83 \\
65{,}536 & yes & baseline & 14.90 & 19.89 & 20.84 & 4.96 & 23.80 \\
131{,}072 & yes & baseline & 16.03 & 26.01 & 23.31 & 5.34 & 33.82 \\
65{,}536 & no & \forge{} & 15.27 & 50.29 & 16.80 & \textbf{0.56} & 54.20 \\
65{,}536 & yes & \forge{} & 15.27 & 20.26 & 16.80 & \textbf{0.56} & 24.17 \\
131{,}072 & yes & \forge{} & 16.39 & 26.38 & 19.28 & \textbf{0.93} & 34.19 \\
\bottomrule
\end{tabular}
\end{center}

\noindent Gradients live fall $4.96\to0.56$~GB at 65k and $5.34\to0.93$ at 131k, but the peak never occurs at that moment: without checkpointing it is the forward crest, and with checkpointing it moves to a block-recompute transient both arms pay identically. The operative condition: \emph{the persistent gradient has to be larger than the largest shared transient.} Here the model is ${\sim}$2.5B, gradients are ${\sim}$5~GB, and \forge{} pays the 0.37~GB pool; at Qwen3-8B under data parallelism the same comparison goes the other way by 14.9~GB (\S\ref{app:rtxdp}) --- the two results bracket the condition.

\paragraph{Exactness.} CP1 gate: Ulysses CP${=}2$ and $4$ vs.\ dense single-GPU, loss gap $\leq 1.2\times10^{-4}$; the paired arms agree to $10^{-4}$ at every measured sequence.

\paragraph{Caveats specific to context parallelism.} The baseline reduces gradients with a per-parameter synchronous \texttt{all\_reduce}; the \forge{} arm runs the coordinator path (\S\ref{app:rtxdp}). The pool was not tuned for this regime: the $+$0.37~GB penalty is the pool ($3\times125$~MB buckets), and shrinking it would take the penalty toward zero.

\subsection{Expert parallelism (MoE)}\label{app:rtxep}

A mixture-of-experts LM ($d{=}2048$, 12 layers) with expert MLPs sharded across the EP group; routing is deterministic round-robin, which balances the all-to-alls and does not affect the mechanism under test. \emph{Pure EP (DP${=}1$):} each expert's gradient is rank-local and complete, so the real fused kernel executes. \emph{EP$\times$DP:} expert replicas hold partial gradients and the coordinator supplies the sum. Fourteen cells, seven configurations measured on both arms; the baseline is the identical model with \texttt{nn.Linear} experts $+$ AdamW.

\begin{center}\small
\setlength{\tabcolsep}{4pt}
\begin{tabular}{@{}l r rrr rrr@{}}
\toprule
config & $E$ & base GB & \forge{} GB & $\Delta$ & base ms & \forge{} ms & ratio \\
\midrule
EP8$\times$DP1 & 8 & 8.87 & 8.87 & $-$0.00 & 110 & 112 & $0.98\times$ \\
EP8$\times$DP1 & 16 & 11.12 & 11.12 & $-$0.00 & 118 & 117 & $1.01\times$ \\
EP8$\times$DP1 & 32 & 15.62 & 15.62 & $-$0.00 & 130 & 124 & $1.05\times$ \\
EP4$\times$DP1 & 16 & 15.62 & 15.62 & $-$0.00 & 118 & 113 & $1.04\times$ \\
EP4$\times$DP1 & 32 & 26.80 & \textbf{24.62} & $\mathbf{-2.18}$ & 140 & 125 & $1.12\times$ \\
EP4$\times$DP1 & 64 & 50.80 & \textbf{42.62} & $\mathbf{-8.18}$ & 185 & 146 & $\mathbf{1.26\times}$ \\
EP4$\times$DP2 & 8 & 11.12 & 11.49 & $+$0.37 & 272 & 563 & $0.48\times$ \\
\bottomrule
\end{tabular}
\end{center}

\noindent Where activations set the peak the fused update costs nothing (EP8, parity to 0.01~GB); as experts pack the rank the deleted gradient emerges, $-$0.00 $\to$ $-$2.18 $\to$ $-$8.18~GB ($-$16.1\%) from $E{=}16$ to $E{=}64$ at $1.04$--$1.26\times$ the baseline's speed. Both arms run the fused kernel with no gradient reduction at all, which makes these the cleanest timed comparisons in the programme. Adding a data-parallel dimension moves the granule down one rung by Proposition~\ref{prop:dist}(ii): the coordinator costs $+$0.37~GB and $0.48\times$ against a baseline whose per-parameter synchronous \texttt{all\_reduce} is not a production schedule either.

\paragraph{Exactness.} EP1 gate: pure-EP MoE at EP${=}2$ and $4$ vs.\ an all-local reference, loss gap $\leq 4.9\times10^{-5}$.

\paragraph{Caveats specific to expert parallelism.} The EP$\times$DP cell reduces gradients with a per-parameter synchronous \texttt{all\_reduce}. The router is round-robin rather than learned top-$k$, which changes traffic but not which gradients are complete on which rank --- the property under test; the published Qwen3-MoE architectures were not run.

\subsection{Pipeline parallelism}\label{app:pp}

A GPipe pipeline: all micro-batch forwards (stashing activations), then all backwards, with stage boundaries exchanging hidden states and stages split by layer count. Pipelining pays only when the pipeline stays full, which requires micro-batch gradient accumulation; under PP \forge{} therefore runs its exact-accumulation mode --- an fp32 gradient-shaped accumulator, stepped once at the final micro-batch, at 4~B per parameter where the baseline's bf16 buffer costs 2. The baseline is plain \texttt{nn.Linear} $+$ AdamW with gradients accumulating across micro-batches.

Across the sixteen cells, at $d{=}2048$ the penalty is a constant $+$0.55~GB, amortizing with micro-batch count ($+20\%$ at $m{=}2$ to $+6\%$ at $m{=}16$); against an fp32-master baseline the accounting favours \forge{} (16~B/param vs.\ 10). Per-cell peaks and step times are in the artifact.

\paragraph{The accumulator dtype is the whole difference.} The exact-accumulation mode holds the micro-batch accumulator in fp32 against the baseline's bf16 \texttt{.grad} --- strictly more accurate, and charged for it: $+$2.63~GB ($+11.0\%$) with fp32 accumulation, cut to $+$0.63~GB ($+2.6\%$, the residual bucket cost) at matched bf16 accumulation (PP4, $d{=}4096$, 8 micro-batches; losses 8.7228 / 8.7207 / 8.7222). Pipeline parallelism offers a choice rather than a penalty: pay 2~B/param for better-than-baseline accumulation, or match it exactly for $2.6\%$.

\paragraph{Why exact accumulation is the only correct configuration.} Under a stashing schedule the earlier micro-batches' autograd graphs still reference the pre-update weight, so an in-place update on any micro-batch before the last is structurally invalid; exact accumulation is not one option of two.

\paragraph{Caveats specific to pipeline parallelism.} \emph{Schedule:} both arms share a hand-written GPipe schedule and a per-parameter synchronous \texttt{all\_reduce} over the data-parallel group, so the comparison is like for like on both axes. \emph{No 1F1B schedule:} GPipe stashes every micro-batch's activations, which inflates activation memory for both arms equally.

\subsection{Compositions: TP $\times$ CP $\times$ DP}\label{app:rtxcomp}

Real Qwen3 models with three mechanisms wired together: megatron-style TP, Ulysses CP as an attention interface, and the \forge{} coordinator over the CP$\times$DP slice --- exactly the group holding partial gradients. Both arms share the TP and CP wiring and the bf16-state recipe; only the update mechanism differs. Twenty-two cells.

\paragraph{Results (Qwen3-8B unless noted).} Composition works as Proposition~\ref{prop:dist} splits it: the TP shards step locally on complete gradients while the coordinator completes the sum over exactly the CP$\times$DP slice that holds partials, and losses stay in family at every configuration including the three-way TP2$\times$CP2$\times$DP2. Across the twenty-two cells the arms sit at memory parity within the bucket pool, $+$0.10--0.42~GB, which is the expected result at these shapes because the peak is the activation crest. The constant-$BT$ law appears here as two independent measurements that agree: TP4$\times$CP2 at batch~1 / sequence 16k and at batch~2 / sequence 8k both process 16{,}384 tokens per rank and both land at 62.01 / 62.38~GB; at batch~2 / sequence 16k both arms exceed the card.

\paragraph{Caveats specific to compositions.} \emph{Step time} on this PCIe node is up to $2.6\times$ the baseline's, from two implementation causes rather than the mechanism: the tensor-parallel-sharded output head exceeds the gradient bucket and is row-chunked, and per-bucket optimizer kernels serialize against 2--4-rank reduction groups. Neither is present in the data-parallel path here, and the same compositions on NVLink run $1.12$--$1.56\times$ \emph{faster} than their baselines (Appendix~\ref{app:h200}), which is what identifies the cause as the interconnect and the bucket rather than the schedule. The baseline reduces gradients with a per-parameter synchronous \texttt{all\_reduce}; PP is a separate harness (\S\ref{app:pp}).

\subsection{RTX gate suite}\label{app:rtxgates}

Run on this node before any benchmark cell (all pass):

\begin{center}\footnotesize
\setlength{\tabcolsep}{4pt}
\begin{tabular}{@{}l l l@{}}
\toprule
gate & scope & result \\
\midrule
T0 & pipeline vs.\ same-kernel reference & bit-exact, $\max|\Delta W|=0$ \\
T1 & replica coherence, $N{=}2/4$, replicated and sharded states & drift $=0.0$ \\
T2 & \forge{}-DP over shards vs.\ single-GPU full batch & loss gap $\le 6\times10^{-7}$ \\
MW & fp32-master arms vs.\ torch mixed precision & loss gap $6\times10^{-7}$ \\
F1--F3 & FSDP-\forge{} vs.\ dense / gather coherence / fp32 master & $4.8\times10^{-7}$ / $0.0$ / $1.3\times10^{-6}$ \\
SR & stochastic-rounding microcheck (dead zone, unbiasedness) & RNE drift 0.000; SR $-$8.195e-4 vs.\ $-$8.0e-4 target \\
CP1 & Ulysses CP${=}2/4$ vs.\ dense single-GPU & loss gap $\le 1.2\times10^{-4}$ \\
EP1 & pure-EP MoE at EP${=}2/4$ vs.\ all-local reference & loss gap $\le 4.9\times10^{-5}$ \\
MG & megatron TP $\in\{2,4,8\}$ A/B, identical init & max$|\Delta$loss$|$ $1.5$--$3.8\times10^{-4}$ \\
TP & native TP module suite (Column$\to$Row vs.\ dense, TP${=}2/4$) & loss identical across ranks; RMS $\Delta W \approx$ 2.5e-3 \\
-- & single-GPU suite & 17/17 pass \\
\bottomrule
\end{tabular}
\end{center}

\subsection{Timing and OOM discipline}\label{app:provenance}

\paragraph{Step times.} Quoted only from idle-node cells (per-step spread at the 1\% level); shared-node cells contribute memory and exactness columns only, since peak memory is shape-determined.

\paragraph{OOM cells.} Every cell is gated on free device memory before launch and an OOM is retried once before it is recorded. The one cell that failed at setup rather than warmup, DDP $+$ mixed-precision AdamW at Qwen3-14B, $N{=}4$, is a boundary by arithmetic: 18~B/param $\times$ 14.77~B params is 266~GB per rank on a 96~GB card.

\section{Cross-Platform Capability}\label{app:platforms}

This appendix is arithmetic on measured footprints, not a measurement campaign. Peak training footprint is shape-deterministic and platform-invariant --- Appendix~\ref{app:grids1} measures the same grids on H200, H100 and B200 and peak agrees to within rounding wherever two platforms both run --- so the footprints transfer and can be set against each platform's budget at eight GPUs. The optimizer-side term is an exact byte count; the activation term is calibrated on those sweeps to $\pm$2--3 percentage points at sequence 4096, sequence parallelism on, recomputation off. Step time does not transfer, so the only timing statement carried across platforms is a measured one: the megatron-core A/B ratio is $1.00$--$1.03\times$ across the low-variance tensor-parallel cells on a quiet node (Appendix~\ref{app:rtx}).

\paragraph{Recipe basis, and what each part of it buys.} The baseline column is the standard mixed-precision layout at 16~B per managed parameter --- bf16 weight 2, fp32 master 4, fp32 moments 8, bf16 gradient buffer 2 --- and the \forge{} column is the bf16-state recipe at 6~B: bf16 weight 2, bf16 moments 4, no master, no gradient buffer. Of the 10~B between them, \emph{2 are the gradient the mechanism deletes and 8 are the state recipe its fidelity results license} (Appendices~\ref{app:precision}, \ref{app:fidelity}, \ref{app:convergence}), and the two should be read separately. At matched bf16 state on both sides the mechanism alone is 2~B of 8, the $25\%$ ceiling the single-GPU and cross-architecture ladders converge on (Appendices~\ref{app:grids1}, \ref{app:archgen}); what the tables below price is the end-to-end consequence of adopting both, which is the choice a practitioner actually makes. Recipe-matched, same-platform, measured comparisons are Appendices~\ref{app:h200}--\ref{app:rtx}.

\paragraph{Platforms.} Vendor specifications; RTX PRO 6000 Server Edition has no NVLink, so tensor parallelism above~2 is bandwidth-bound there.

\begin{center}\small
\setlength{\tabcolsep}{4pt}
\begin{tabular}{@{}l rrl r@{}}
\toprule
GPU & memory & bandwidth & interconnect & BF16 \\
\midrule
A100-40GB SXM & 40~GB & 1.56~TB/s & NVLink 3, 600~GB/s & 312~TF \\
A100-80GB SXM & 80~GB & 2.04~TB/s & NVLink 3, 600~GB/s & 312~TF \\
H200 SXM & 141~GB & 4.8~TB/s & NVLink 4, 900~GB/s & 990~TF \\
RTX PRO 6000 SE & 96~GB & 1.60~TB/s & PCIe Gen5 only & 500~TF \\
B200 SXM & 180~GB & 8.0~TB/s & NVLink 5, 1.8~TB/s & 2250~TF \\
B300 SXM & 288~GB & 8.0~TB/s & NVLink 5, 1.8~TB/s & 2250~TF \\
\bottomrule
\end{tabular}
\end{center}

\paragraph{Per-rank footprint at eight GPUs (platform-independent).} Micro-batch~1, sequence 4096; dense models under TP8, mixture-of-experts under their KV ceiling as TP1$\times$EP8. These per-rank bytes are the platform-invariant quantity; each platform then either holds them or does not:

\begin{center}\small
\begin{tabular}{@{}l l rrrr@{}}
\toprule
model & split & baseline & \forge{} & saved & \% \\
\midrule
Qwen3-0.6B & TP8 & 4.0 & 3.5 & 0.5 & 13 \\
Qwen3-1.7B & TP8 & 7.0 & 5.4 & 1.6 & 23 \\
Qwen3-4B & TP8 & 13.0 & 8.8 & 4.2 & 33 \\
Qwen3-8B & TP8 & 22.0 & 13.2 & 8.8 & 40 \\
Qwen3-14B & TP8 & 36.7 & 20.5 & 16.3 & 44 \\
Qwen3-32B & TP8 & 78.9 & 41.7 & 37.2 & 47 \\
Qwen3-30B-A3B & TP1$\times$EP8 & 107.4 & 62.4 & 45.1 & 42 \\
Qwen3-235B-A22B & TP1$\times$EP8 & 647.0 & 314.4 & 332.6 & 51 \\
\bottomrule
\end{tabular}
\end{center}

\noindent The MoE per-rank numbers are large because expert parallelism concentrates weights, not because anything grew: at TP1$\times$EP8 each rank holds the full attention/shared stack plus $16$ of the $128$ experts --- about 40B parameters per rank at Qwen3-235B-A22B, hence 647~GB baseline and 314~GB under \forge{}. No eight-GPU machine of any generation holds that, which is exactly what the dash row in the capability table below says; at that scale the operative statement is the smallest-cluster table at the end of this appendix, where the same arithmetic is what cuts the required cluster by $4\times$.

\paragraph{Largest micro-batch that fits.} Eight GPUs, sequence 4096, baseline / \textbf{\forge{}}; a dash means the configuration does not fit at any batch size:

\begin{center}\footnotesize
\setlength{\tabcolsep}{3.5pt}
\begin{tabular}{@{}l cccccc@{}}
\toprule
model & A100-40 & A100-80 & H200 & RTX PRO & B200 & B300 \\
\midrule
Qwen3-0.6B & 25 / \textbf{26} & 53 / \textbf{54} & 64 / \textbf{64} & 64 / \textbf{64} & 64 / \textbf{64} & 64 / \textbf{64} \\
Qwen3-1.7B & 14 / \textbf{15} & 31 / \textbf{31} & 56 / \textbf{57} & 37 / \textbf{38} & 64 / \textbf{64} & 64 / \textbf{64} \\
Qwen3-4B & 7 / \textbf{8} & 16 / \textbf{17} & 31 / \textbf{32} & 20 / \textbf{21} & 43 / \textbf{44} & 64 / \textbf{64} \\
Qwen3-8B & 4 / \textbf{5} & 11 / \textbf{13} & 22 / \textbf{24} & 14 / \textbf{15} & 31 / \textbf{33} & 49 / \textbf{50} \\
Qwen3-14B & 1 / \textbf{3} & 6 / \textbf{8} & 13 / \textbf{15} & 8 / \textbf{10} & 19 / \textbf{21} & 31 / \textbf{33} \\
Qwen3-32B & --- / --- & --- / \textbf{3} & 4 / \textbf{6} & 1 / \textbf{4} & 7 / \textbf{9} & 12 / \textbf{15} \\
Qwen3-30B-A3B & --- / --- & --- / \textbf{1} & 1 / \textbf{3} & --- / \textbf{1} & 3 / \textbf{5} & 6 / \textbf{8} \\
Qwen3-235B-A22B & --- / --- & --- / --- & --- / --- & --- / --- & --- / --- & --- / --- \\
\bottomrule
\end{tabular}
\end{center}

\paragraph{Smallest cluster for the mixture-of-experts models.} At MoE scale the meaningful quantity is how much hardware a run requires: the smallest GPU count that fits at micro-batch~1, sequence 4096, searching TP$\times$EP splits within each model's KV-head TP ceiling (4 for both MoE models). The MoE rows assume pure expert parallelism, where each expert lives on exactly one rank and its local gradient is complete --- the configuration measured at EP8 with no overhead (Appendix~\ref{app:rtx}); under expert replication those weights need the bucket coordinator instead.

\begin{center}\small
\setlength{\tabcolsep}{4pt}
\begin{tabular}{@{}l l rrl@{}}
\toprule
model & platform & baseline & \forge{} & ratio \\
\midrule
Qwen3-30B-A3B & A100-40GB & 128 (TP2$\times$EP64) & \textbf{32} (TP2$\times$EP16) & $4\times$ fewer \\
Qwen3-30B-A3B & A100-80GB & 32 (TP1$\times$EP32) & \textbf{8} (TP1$\times$EP8) & $4\times$ fewer \\
Qwen3-30B-A3B & H200 & 8 & \textbf{8} & --- \\
Qwen3-30B-A3B & RTX PRO 6000 & 16 (TP1$\times$EP16) & \textbf{8} (TP1$\times$EP8) & $2\times$ fewer \\
Qwen3-30B-A3B & B200 / B300 & 8 & \textbf{8} & --- \\
Qwen3-235B-A22B & A100-40GB & $>$512 & $>$512 & --- \\
Qwen3-235B-A22B & A100-80GB & $>$512 & \textbf{256} (TP4$\times$EP64) & --- \\
Qwen3-235B-A22B & H200 & 256 (TP4$\times$EP64) & \textbf{64} (TP2$\times$EP32) & $4\times$ fewer \\
Qwen3-235B-A22B & RTX PRO 6000 & 512 (TP4$\times$EP128) & \textbf{128} (TP4$\times$EP32) & $4\times$ fewer \\
Qwen3-235B-A22B & B200 & 128 (TP2$\times$EP64) & \textbf{32} (TP2$\times$EP16) & $4\times$ fewer \\
Qwen3-235B-A22B & B300 & 64 (TP2$\times$EP32) & \textbf{16} (TP1$\times$EP16) & $4\times$ fewer \\
\bottomrule
\end{tabular}
\end{center}

\noindent Activation recomputation raises every relative figure. The optimizer-side bytes removed by mechanism and recipe together translate directly into a smaller cluster, a factor of four at Qwen3-235B-A22B on four of the six platforms.

\section{Cross-Architecture Generalization: ViT, Mamba-2, and MLP-Mixer}\label{app:archgen}

This appendix tests the architecture-independence claim --- that the fused step reads only the identity $\nabla_{\mathbf{W}}\mathcal{L}=\dY^{\!\top}\mathbf{X}$ --- on the three families furthest from the dense decoder LLM: a pure vision transformer (ViT), a selective-scan state-space model (Mamba-2, no attention), and an all-MLP Mixer (no attention and almost nothing but linear layers; \S\ref{app:pmixer}). All cells run on one H200 (141~GB) under the bf16-state recipe on both arms (bf16 weights and states, no fp32 master; $8$~B/param, of which \forge{} removes the $2$~B gradient on a fused layer and int8 halves the two state tensors), at random initialization (peaks and step times are shape-deterministic; losses checked finite). Step times are medians of five \emph{round-robin} rounds with the within-arm spread recorded, so noise-dominated cells are discarded rather than quoted. The same kernel runs unmodified across all three families; only the module-matching front-end differs (below).

\subsection{Coverage Across Architectures}\label{app:pcov}

\forge{} manages a weight when it is a single-consumer linear layer \emph{and} the model reaches it through \texttt{Module.forward}. ViT satisfies both for essentially every parameter; Mamba-2's optimized \texttt{mamba\_ssm} path hands \texttt{out\_proj} directly to a fused scan kernel, so the achieved fraction plateaus near $0.66$ --- below, not above, a dense transformer's $0.92$.

\begin{center}\small
\begin{tabular}{@{}l l r r@{}}
\toprule
family & model & params & achieved fraction \\
\midrule
Mamba-2 & 130m & 168M & 0.599 \\
Mamba-2 & 780m & 857M & 0.645 \\
Mamba-2 & 1.3b & 1.45B & 0.650 \\
Mamba-2 & 2.7b & 2.83B & 0.657 \\
Mamba-2 & 7b & 6.96B & 0.661 \\
Mamba-2 & 13b & 13.3B & 0.664 \\
Mamba-2 & 20b & 20.8B & 0.665 \\
\addlinespace[2pt]
ViT & base--32b & 87M--33B & 0.990--1.000 \\
\bottomrule
\end{tabular}
\end{center}

\subsection{Peak Memory and Capability}\label{app:pmem}

\paragraph{Peak memory vs.\ model scale (batch 1; GB).} The saving grows with model size toward the $25\%$/$50\%$ ceilings as the fixed activation floor becomes negligible, and int8 trains sizes both other arms cannot fit. The ladder is a family of width- and depth-scaled ViT configurations anchored at ViT-G/14 (1.8B) and extending past the 21.7B of ViT-22B (Dehghani et al.\ 2023), the largest published ViT; the Mamba-2 ladder extends the released 130M--2.7B family the same way. Peak and step time are shape-deterministic, so no pretrained checkpoint is involved or needed.

\begin{figure}[tbp]
\centering
\includegraphics[width=\linewidth]{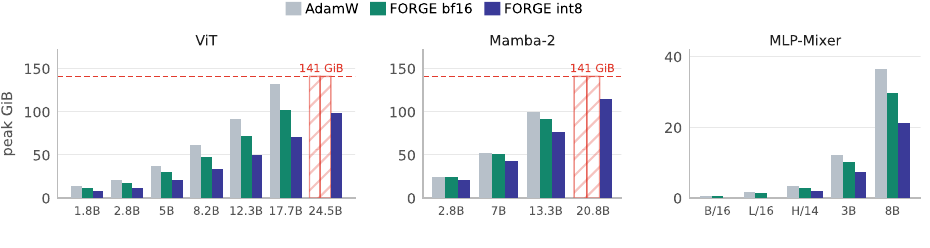}
\caption{Peak memory vs.\ model scale (batch 1, H200): AdamW vs.\ \forge{} bf16 vs.\ \forge{} int8 across the ViT, Mamba-2, and MLP-Mixer ladders; hatched red = out of memory. int8 trains ViT-25b (98.22~GB) and Mamba2-20b (114.64~GB), sizes both other arms cannot fit.}
\label{fig:supparch}
\end{figure}

\paragraph{Peak memory vs.\ batch (GB).} The bf16 saving is the transient gradient and is swamped once activations dominate (gone by batch~8); the int8 saving is persistent optimizer state and decays far more slowly. That difference is the operative one for choosing an arm: the transient saving prices the gradient, the int8 saving prices the states.

\begin{figure}[tbp]
\centering
\includegraphics[width=0.72\linewidth]{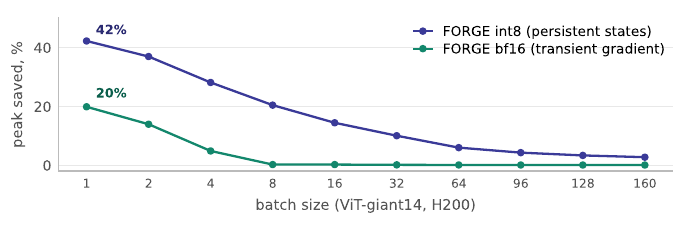}
\caption{Peak saved vs.\ batch on ViT-giant14: the bf16 saving is the transient gradient and is gone once activations dominate (by batch~8), while the int8 saving is persistent optimizer state and decays far more slowly. ViT-huge14 and the Mamba-2 ladders reproduce the shape and are tabulated in the artifact.}
\label{fig:suppbatchdecay}
\end{figure}

\subsection{Resident Optimizer State}\label{app:pstate}

Measured after warmup with gradients released, isolating state (GB; activation-independent). Under the bf16-state recipe both arms hold bf16 moments (4~B/param), so bf16-arm parity here is expected, not a null result --- \forge{}'s bf16-state saving is the transient gradient, which this metric deliberately excludes; int8 is where resident state moves, by $29$--$46\%$.

\begin{center}\small
\setlength{\tabcolsep}{5pt}
\begin{tabular}{@{}l rrr r@{}}
\toprule
model & AdamW & \forge{} bf16 & \forge{} int8 & int8 saving \\
\midrule
Mamba2-370m & 1.627 & 1.627 & 1.162 & 28.6\% \\
Mamba2-780m & 3.257 & 3.257 & 2.292 & 29.6\% \\
Mamba2-1.3b & 5.463 & 5.454 & 3.819 & 30.1\% \\
Mamba2-2.7b & 10.644 & 10.617 & 7.365 & 30.8\% \\
ViT-base16 & 0.393 & 0.407 & 0.238 & 39.6\% \\
ViT-large16 & 1.198 & 1.198 & 0.668 & 44.2\% \\
ViT-huge14 & 2.429 & 2.428 & 1.363 & 43.9\% \\
ViT-giant14 & 6.936 & 6.936 & 3.719 & 46.4\% \\
\bottomrule
\end{tabular}
\end{center}

\subsection{Step Time and Phase Split}\label{app:ptime}

Medians of five interleaved rounds (ms). The transient-buffer arm is at parity ($0.98$--$1.03\times$); the fully-fused arm is slower ($0.73$--$0.87\times$) because its Triton weight-gradient mainloop lands in the backward pass and outweighs the optimizer saving. The optimizer phase alone falls $1.8$--$136\times$ with parameter count ($7.5\to0.07$~ms at ViT-giant14), which is mechanism evidence rather than an end-to-end claim. The end-to-end result at these cells is memory, and the reason is the operating condition of the main paper's Section~3.4 rather than the architecture: these are batch 4--128 vision and state-space cells, where the backward sets the step, against the small-$BT$ language-model cells where the optimizer phase does.

\begin{center}\small
\begin{tabular}{@{}l rrrr@{}}
\toprule
cell & baseline & \forge{}-2k & \forge{}-bf16 & \forge{}-int8 \\
\midrule
mamba2-1.3b bs4 & 256.2 & 256.4 ($0.999\times$) & 306.9 ($0.835\times$) & 272.5 ($0.940\times$) \\
vit-giant14 bs32 & 226.7 & 231.3 ($0.980\times$) & 309.5 ($0.732\times$) & 265.9 ($0.853\times$) \\
vit-huge14 bs64 & 208.6 & 202.2 ($1.032\times$) & 260.0 ($0.802\times$) & 224.8 ($0.928\times$) \\
vit-large16 bs128 & 218.5 & 218.2 ($1.001\times$) & 263.4 ($0.830\times$) & 231.1 ($0.945\times$) \\
\bottomrule
\end{tabular}
\end{center}

\noindent Two further interleaved cells were measured and are reported for completeness: mamba2-780m bs4 (baseline 184.6; \forge{}-2k 184.7, $0.999\times$; \forge{}-bf16 213.3, $0.865\times$; \forge{}-int8 194.2, $0.951\times$) and vit-base16 bs128, whose within-arm spread (32.5\%) exceeds the 15\% gate and is therefore noise-dominated and not quoted.

\paragraph{Phase split.} Where the end-to-end result comes from: \forge{}'s optimizer phase nearly vanishes, and the fully-fused arm's cost reappears in the backward as the Triton weight-gradient mainloop --- on Mixer-8B, $21.68\to58.61$~ms of backward against $19.199\to0.078$~ms of optimizer (\S\ref{app:pmixer}). Phase-split cells are separately instrumented runs (per-phase event timing), so their totals differ from the interleaved medians above by a few percent and only within-run ratios are meaningful; the per-phase forward/backward/optimizer breakdown for every cell is in the artifact.

\paragraph{Optimizer phase vs.\ model size (ms).} The contraction scales with parameter count --- the baseline optimizer is launch- and bandwidth-bound over many tensors, while \forge{} folds the update into the weight-gradient epilogue. It is the expected win, and it does not carry to end to end; it is reported as mechanism evidence.

\begin{figure}[tbp]
\centering
\includegraphics[width=\linewidth]{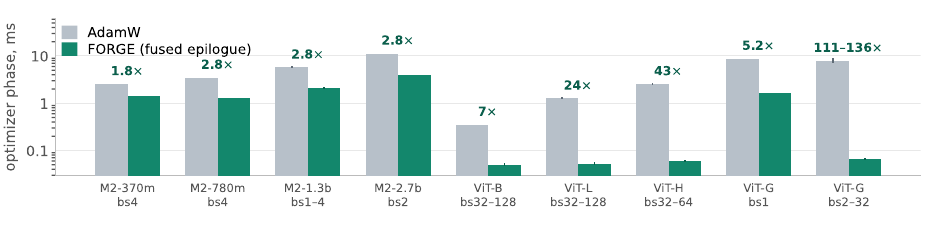}
\caption{Optimizer phase, AdamW vs.\ \forge{} (ms, log scale); whiskers span the measured range across batch sizes; labels give the contraction.}
\label{fig:suppoptph}
\end{figure}

\noindent The vit-giant14 bs1 outlier ($5.2\times$ rather than ${\sim}136\times$) is the fully-fused arm paying its Triton launch overhead where the step is launch-bound; from bs2 the epilogue amortizes and the contraction saturates.

\paragraph{Capability at batch 1.} ViT-25b and Mamba2-20b do not just allocate under \forge{}-int8, they step --- 334.4 and 981.3~ms respectively; at batch~1 the step is launch-bound, so the interleaved table above is the timing statement.

\subsection{Correctness and Bypassed-Module Handling}\label{app:pcorr}

Against a same-dtype (bf16) reference --- an fp32 reference is uninformative at this horizon, since the weights move less than the initial bf16 representation error --- the per-step deviation is consistent with the LLM path: on ViT $4.0\times10^{-2}$ of the update magnitude over 20 steps, on Mamba-2 $5.5\times10^{-3}$ at one step, matching a Qwen3-0.6B control ($5.2\times10^{-3}$). Multi-step Mamba-2 parity is not measurable on \emph{either} arm: \texttt{mamba\_ssm}'s backward uses atomics and is nondeterministic, and the recurrence amplifies that to six orders of magnitude of run-to-run spread with identical code, so the one-step deviation is the available statement.

\paragraph{Bypassed-module handling.} \forge{}'s contract is that a fused weight's only consumer is the fused epilogue, so the manager removes fused weights from the standard optimizer. Where a library bypasses \texttt{Module.forward} --- as \texttt{mamba\_ssm} does for 24 of 49 candidate modules --- autograd still produces a gradient that no path would apply, so the manager reverts any module the model never calls (\texttt{revert\_bypassed\_fusions}) back to the standard optimizer. This probe is empirical rather than architecture-specific, so it covers libraries not examined here; with it in place the Mamba-2 deviation sits at $0.0055\times$ the update magnitude, and Qwen3 and Llama have zero bypassed modules in both matching modes.

\subsection{MLP-Mixer at the Coverage Extreme}\label{app:pmixer}

\forge{}'s reach is bounded by the fraction of parameters in \texttt{nn.Linear} layers entered through \texttt{Module.forward}, and MLP-Mixer (Tolstikhin et al.\ 2021) is the extreme point of that axis: it replaces self-attention with a token-mixing MLP over the patch dimension and a channel-mixing MLP over features, both plain matrix multiplications, leaving only the patch-embedding convolution and the LayerNorms outside. It is therefore the clean test that \forge{} is a \emph{Linear-layer} method rather than an attention method --- there is no attention to exploit. One H200, PyTorch 2.12.1, CUDA 13.1, Triton 3.6, 2 warmup $+$ 5 measured steps per cell. \texttt{mixer-b16}/\texttt{-l16}/\texttt{-h14} follow the published specifications, landing within $+0.3$--$1.5\%$ of 59M/207M/431M on the 1000-way head and LayerNorms; \texttt{mixer-3b} and \texttt{mixer-8b} are width- and depth-scaled beyond the published sizes (hidden 2048/3072, 48/64 blocks) and labelled \emph{(scaled)}, as on the ViT and Mamba-2 ladders.

\begin{center}\small
\begin{tabular}{@{}l r r r r@{}}
\toprule
model & params & in \texttt{nn.Linear} & fraction & bypassed modules \\
\midrule
Mixer-B/16 & 59.88M & 59.20M & 0.9886 & 0 \\
Mixer-L/16 & 208.20M & 207.17M & 0.9951 & 0 \\
Mixer-H/14 & 432.35M & 431.20M & 0.9973 & 0 \\
Mixer-3B \emph{(scaled)} & 1.64B & 1.64B & 0.9987 & 0 \\
Mixer-8B \emph{(scaled)} & 4.89B & 4.89B & \textbf{0.9992} & 0 \\
\bottomrule
\end{tabular}
\end{center}

\noindent The highest reachable fraction in the study, against $0.990$--$1.000$ for ViT, $0.92$ for dense Qwen3 and $0.66$ \emph{achieved} for Mamba-2, with zero bypassed modules at every size, so nothing is lost between architectural and achieved reach. The low end of that range is an interface property and not an architectural one: a HuggingFace mixture-of-experts checkpoint as shipped reaches $0.04$, because its experts are stored as fused 3-D parameters rather than a list of \texttt{nn.Linear} and so never pass through \texttt{Module.forward} --- the same seam \texttt{mamba\_ssm} opens with its scan, at a different size.

\paragraph{Peak memory (batch 1).} The bf16-state recipe on both sides; the token count is fixed by patch geometry (196 for /16, 256 for /14), and the int8 arm is unavailable at /16 (blocked by the shape constraint below). The ladder is plotted in the right panel of Figure~\ref{fig:supparch}: 0.519/1.622/3.293/12.302/36.534~GB AdamW against 0.460/1.405/2.827/10.170/29.689 \forge{} bf16 ($-$11.4 to $-$18.7\%), with int8 at 2.078/7.310/21.189 from H/14 up ($-$36.9 to $-$42.0\%).

\noindent The saving grows monotonically with model size --- $11.4\to18.7\%$ for bf16 and $36.9\to42.0\%$ for int8 --- converging on the $25\%/50\%$ ceilings implied by the per-parameter budget as the fixed activation and workspace floor becomes negligible: the ViT ladder's trend reproduced on an architecture with no attention at all.

\paragraph{Resident optimizer state (GB; activation-independent).} Measured after warmup with gradients released. As in the ViT/Mamba table above, baseline--bf16 parity is by construction, not accident: under bf16-state both hold bf16 moments, and \forge{}'s bf16 saving is the transient gradient, which this metric deliberately excludes.

\begin{center}\small
\begin{tabular}{@{}l r r r r@{}}
\toprule
model & AdamW & \forge{} bf16 & \forge{} int8 & int8 saving \\
\midrule
Mixer-B/16 & 0.288 & 0.287 & --- & --- \\
Mixer-L/16 & 0.838 & 0.838 & --- & --- \\
Mixer-H/14 & 1.674 & 1.673 & 0.924 & $-$44.8\% \\
Mixer-3B \emph{(scaled)} & 6.172 & 6.172 & 3.312 & $-$46.3\% \\
Mixer-8B \emph{(scaled)} & 18.276 & 18.275 & 9.798 & $-$46.4\% \\
\bottomrule
\end{tabular}
\end{center}

\noindent int8 approaches the theoretical $-50\%$ (halving both state tensors), reaching $-46.4\%$ at 8B. The two \forge{} arms differ in kind, and the distinction is practical: bf16 removes a \emph{transient} buffer, visible only while the peak is parameter-dominated and washed out as activations grow; int8 removes \emph{persistent} state, so its saving is far more robust across operating points. For a memory-constrained deployment, int8 is the arm that pays.

\paragraph{Step time.} The optimizer phase nearly vanishes --- $19.199\to0.078$~ms at 8B ($246\times$) and $7.592\to0.062$~ms at 3B ($122\times$) --- the mechanism working exactly as designed: the update is folded into the weight-gradient epilogue and no longer exists as a separate pass over parameter memory. The contraction is mechanism evidence rather than an end-to-end speedup --- the epilogue's cost moves into the backward ($21.68\to58.61$~ms at 8B) and the net step is $0.70$--$1.30\times$ the baseline --- so on H200 the Mixer win is memory. \emph{Timing caveat:} cells are single runs with ${\sim}10\%$ within-arm spread in comparable interleaved measurements; the $122\times$/$246\times$ contractions sit far outside any noise band, and ratios below ${\sim}2\times$ read as within measurement resolution.

\paragraph{Correctness.} All 13 successful cells produced finite losses at every step, and baseline and \forge{} bf16 track as an algebraically equivalent update should: $7.566$ / $7.567$ at B/16, $7.462$ / $7.464$ at L/16, bitwise-equal reported values at H/14, 3B and 8B. The int8 arm separates by 2--3 nats, as quantized state does over this horizon. These are randomly initialized models on synthetic labels over a handful of steps, so the numbers establish finiteness and arm-to-arm agreement; the convergence evidence is Appendix~\ref{app:convergence}.

\paragraph{Scope condition on the int8 arm: $H \equiv 0 \pmod{64}$.} Mixer is also the architecture that exposes the quantization block size. The token-mixing MLP's \texttt{in\_features} is the token count, so patch-16 at $224{\times}224$ gives $196$, not a multiple of the 64-element block, and the int8 arm is unavailable at B/16 and L/16; patch-14 gives $256$ and quantizes without issue, which is why H/14, 3B and 8B all run. Padding the blocks or adding a tail path removes it. The bf16 arm is unaffected at every size.

\end{document}